\documentclass{article}

    \usepackage[final]{neurips_2023}

\usepackage[utf8]{inputenc} 
\usepackage[T1]{fontenc}    
\usepackage{hyperref}       
\usepackage{url}            
\usepackage{booktabs}       
\usepackage{amsfonts}       
\usepackage{nicefrac}       
\usepackage{microtype}      

\usepackage{macros}

\title{Multi-task Representation Learning for Pure Exploration in Bilinear Bandits}

\author{%
  Subhojyoti Mukherjee \\
  ECE Department\\
  UW-Madison\\
  Wisconsin, Madison \\
  \texttt{smukherjee27@wisc.edu} 
  \and
  Qiaomin Xie \\
  ISyE Department\\
  UW-Madison\\
  Wisconsin, Madison 
  \and
  Josiah P. Hanna \\
  CS Department\\
  UW-Madison\\
  Wisconsin, Madison \\
  \and
  Robert Nowak \\
  ECE Department\\
  UW-Madison\\
  Wisconsin, Madison \\
}

\begin{document}

\maketitle

\begin{abstract}
We study  multi-task representation learning for the problem of pure exploration in bilinear bandits.
In bilinear bandits, an action takes the
form of a pair of arms from two different entity
types and the reward is a bilinear function
of the known feature vectors of the arms. 
In the \textit{multi-task bilinear bandit problem}, we aim to find optimal actions for multiple tasks that share a common low-dimensional linear representation. The objective is to leverage this characteristic to expedite the process of identifying the best pair of arms for all tasks. 
We propose the algorithm \gob\ that uses an experimental design approach to optimize sample allocations for learning the global representation as well as minimize the number of samples needed to identify the optimal pair of arms in individual tasks. 
To the best of our knowledge, this is the first study to give sample complexity analysis for pure exploration in bilinear bandits with shared representation.
Our results demonstrate that by learning the shared representation across tasks, we achieve significantly improved sample complexity compared to the traditional approach of solving tasks independently. 
%
\end{abstract}

\section{Introduction}
\label{sec:intro}
Bilinear bandits \citep{jun2019bilinear,lu2021low,kang2022efficient} are an important class of sequential decision-making problems. 
In bilinear bandits (as opposed to the standard linear bandit setting) we are given a pair of arms $\bx_t\in\R^{d_1}$ and $\bz_t\in\R^{d_2}$ at every round $t$ and the interaction of this pair of arms with a low-rank hidden parameter, $\bTheta_*\in\R^{d_1\times d_2}$ generates the noisy feedback (reward) $r_t = \bx_t^\top \bTheta_* \bz_t + \eta_t$. The $\eta_t$ is  random $1$-subGaussian noise.

A lot of real-world applications exhibit the above bilinear feedback structure, particularly applications that involve selecting pairs of items and evaluating their compatibility. 
For example, in a drug discovery application, scientists may want to determine whether a particular (drug, protein) pair interacts in the desired way \citep{luo2017network}. Likewise, an online dating service might match a pair of people and gather feedback about their compatibility \citep{shen2023user}. A clothing website's recommendation system may suggest a pair of items (top, bottom) for a customer based on their likelihood of matching \citep{reyes2021adaptable}. 
In all of these scenarios, the two items are considered as a single unit, and the system must utilize available feature vectors ($\bx_t, \bz_t$) to learn which features of the pairs are most indicative of positive feedback in order to make effective recommendations. 
All the previous works in this setting \citep{jun2019bilinear,lu2021low,kang2022efficient} exclusively focused on maximizing the number of pairs with desired interactions discovered over time (regret minimization).
However, in many real-world applications where obtaining a sample is expensive and time-consuming, e.g., clinical trials \citep{zhao2009reinforcement,zhang2012multi}, it is often desirable to identify the optimal option using as few samples as possible, i.e., we face the pure exploration scenario \citep{fiez2019sequential, katz2020empirical} rather than regret minimization. 

%
Moreover, in various decision-making scenarios, we may encounter multiple interrelated tasks such as treatment planning for different diseases \citep{bragman2018uncertainty} and content optimization for multiple websites \citep{agarwal2009explore}. Often, there exists a shared representation among these tasks, such as the features of drugs or the representations of website items. Therefore, we can leverage this shared representation to accelerate learning. 
This area of research is called multi-task  representation learning and has recently generated a lot of attention in machine learning \citep{bengio2013representation, li2014joint, maurer2016benefit, du2020few, tripuraneni2021provable}. 
There are many applications of this multi-task representation learning in real-world settings.
For instance, in clinical treatment planning, we seek to determine the optimal treatments for multiple diseases, and there may exist a  low-dimensional representation common to multiple diseases. To avoid the time-consuming process of conducting clinical trials for individual tasks and collecting samples, we utilize the shared representation and decrease the number of required samples.

The above multi-task representation learning naturally shows up in bilinear bandit setting as follows: Let there be $M$ tasks indexed as $m=1,2,\ldots, M$ with each task having its own hidden parameter $\bTheta_{m,*}\in \R^{d_1\times d_2}$. 
Let each $\bTheta_{m,*}$ has a decomposition of $\bTheta_{m,*} = \bB_1 \bS_{m,*}\bB_2^\top$, where $\bB_1\in\R^{d_1\times k_1}$ and $\bB_2\in\R^{d_2\times k_2}$ are shared across tasks, but $\bS_{m,*}\in \R^{k_1\times k_2}$ is specific for task $m$.
We assume that $k_1, k_2 \ll d_1, d_2$ and $M\gg d_1,d_2$. Thus, $\bB_1$ and $\bB_2$ provide a means of dimensionality reduction.  Furthermore, we assume that each $\bS_{m,*}$ has rank $r\ll \min\{ k_1, k_2\}$. 
In the terminology of multi-task representation learning $\bB_1,\bB_2$ are called \emph{feature extractors} and $\bx_{m,t},\bz_{m,t}$ are called \emph{rich observations} \citep{yang2020impact, yang2022nearly, du2023multi}. The reward for the task $m\in\{1,2,\ldots,M\}$ at round $t$ is
\begin{align}
    r_{m,t} = \bx_{m,t}^\top \bTheta_{m,*} \bz_{m,t} + \eta_{m,t} = \underbrace{\bx_{m,t}^\top\bB_1}_{\bg_{m,t}^\top}\bS_{m,*}\underbrace{\bB_2^\top\bz_{m,t}}_{\bv_{m,t}} + \eta_{m,t} = \bg_{m,t}^\top\bS_{m,*}\bv_{m,t} + \eta_{m,t}. \label{eq:reward-model-multi}
    \vspace*{-1em}
\end{align}
Observe that similar to the learning procedure in \citet{yang2020impact,yang2022nearly}, at each round $t=1, 2, \cdots$, for each task $m \in[M]$, the learner selects a left and right action $\bx_{m,t} \in \X$ and $\bz_{m,t}\in\Z$. After the player commits the batch of actions for each task $\left\{\bx_{m, t}, \bz_{m, t}: m \in[M]\right\}$, it receives the batch of rewards $\left\{r_{m, t}: m \in[M]\right\}$.
Also note that in \eqref{eq:reward-model-multi} we define the $\tg_{m,t}\in\R^{k_1}, \tv_{m,t}\in\R^{k_2}$ as the latent features, and both $\tg_{m,t}, \tv_{m,t}$ are unknown to the learner and needs to be learned for each task $m$ (hence the name multi-task representation learning).  

In this paper, we focus on pure exploration for multi-task representation learning in bilinear bandits where the goal is to find the optimal left arm $\bx_{m,*}$ and right arm $\bz_{m,*}$ for each task $m$ with a minimum number of samples (fixed confidence setting). 
First, consider a single-task setting and let $\bTheta_*$ have low rank $r$. Let the SVD of the  $\bTheta_* = \bU\bD\bV^\top$. Prima-facie, if $\bU$ and $\bV$ are known then one might want to project all the left and right arms in the $r\times r$ subspace of $\bU$ and $\bV$ and reduce the bilinear bandit problem into a $r^2$ dimension linear bandit setting. 
Then one can apply one of the algorithms from \citet{soare2014best,fiez2019sequential,katz2020empirical} to solve this $r^2$ dimensional linear bandit pure exploration problem. Following the analysis of this line of work (in linear bandits) \citep{mason2021nearly,mukherjee2022chernoff,mukherjee2023speed} one might conjecture that a sample complexity bound of $\tO(r^2/\Delta^2)$ is possible where $\Delta$ is the minimum reward gap and $\tO(\cdot)$ hides log factors.
Similarly, for the multi-task setting one might be tempted to use the linear bandit analysis of \citet{du2023multi} to convert this problem into $M$ concurrent $r^2$ dimensional linear bandit problems with shared representation and achieve a sample complexity bound of $\tO(M r^2/\Delta^2)$.
However, these matrices (subspaces) are unknown and so there is a model mismatch as noted in the regret analysis of bilinear bandits \citep{jun2019bilinear,lu2021low,kang2022efficient}. Thus it is difficult to apply the $r^2$ dimensional linear bandit sample complexity analysis. 
Following the regret analysis of bilinear bandit setting by \citet{jun2019bilinear,lu2021low,kang2022efficient} we know that the effective dimension is actually $(d_1+d_2)r$. Similarly for the multi-task representation learning the effective dimension should scale with the learned latent features $(k_1+k_2)r$. 
Hence the natural questions to ask are these:
\begin{quote}
    \small\textbf{1) Can we design a single-task pure exploration bilinear bandit algorithm whose sample complexity scales as $\tO((d_1+d_2)r/\Delta^2)$?} \\\\
    \small\textbf{2) Can we design an algorithm for multi-task pure exploration bilinear bandit problem that can learn the latent features and has sample complexity that scales as $\tO(M(k_1+k_2)r/\Delta^2)$?}
\end{quote}

In this paper, we answer both these questions affirmatively. In doing so, 
we make the following novel contributions to the growing literature of multi-task representation learning in online  settings:
\par
\textbf{1)} We formulate the multi-task bilinear representation learning problem. To our knowledge, this is the first work that explores pure exploration in a multi-task bilinear representation learning setting.
\par
\textbf{2)} We proposed the algorithm \gob\ for a single-task pure exploration bilinear bandit setting whose sample complexity scales as $\tO((d_1+d_2)r/\Delta^2)$. This improves over RAGE \citep{fiez2019sequential} whose sample complexity scales as $\tO((d_1d_2)/\Delta^2)$.
\par
\textbf{3)} Our algorithm \gob\ for multi-task pure exploration bilinear bandit problem learns the latent features and has sample complexity that scales as $\tO(M(k_1+k_2)r/\Delta^2)$. This improves over DouExpDes \citep{du2023multi} whose samples complexity scales as $\tO(M(k_1k_2)/\Delta^2)$.

\label{sec:prelim}
\textbf{Preliminaries:} We assume that $\|\bx\|_2\leq 1$, $\|\bz\|_2\leq 1$, $\|\bTheta_*\|_F\leq S_0$ and the $r$-th largest singular value of $\bTheta_*\in\R^{d_1\times d_2}$ is $S_r$. Let $p\coloneqq d_1d_2$ denote the ambient dimension, and $k=(d_1+d_2)r$ denote the effective dimension. Let $[n] \coloneqq \{1,2,\ldots,n\}$. 
Let $\bx_*,\bz_*  \!\coloneqq\! \argmax_{\bx,\bz}\bx^\top\bTheta_*\bz$.
For any $\bx, \bz$ define the gap $\Delta(\bx,\bz) \!\coloneqq\! \bx_*^\top\bTheta_*\bz_* - \bx^\top\bTheta_*\bz$ and furthermore $\Delta = \min_{\bx \neq \bx_*,\bz\neq \bz_*} \Delta(\bx,\bz)$. 
%
%
Similarly, for any arbitrary vector $\bw \in \W$ define the gap of $\bw\in\R^p$ as $\Delta(\bw)\coloneqq\left(\bw_*-\bw\right)^{\top} \btheta_*$, for some $\btheta_*\in\R^p$ and furthermore, $\Delta_{}=\min _{\bw \neq \bw_*} \Delta(\bw)$. 
If $\bA \in \R_{\geq 0}^{d \times d}$ is a positive semidefinite matrix, and $\bw \in \R^p$ is a vector, let $\|\bw\|_{\bA}^2:=\bw^{\top} \bA \bw$ denote the induced semi-norm. 
Given any vector $\bb\in\R^{|\W|}$ we denote the $\bw$-th component as $\bb_{\bw}$.
Let $\Delta_{\W}:=\left\{\bb \in \R^{|\W|}: \bb_{\bw} \geq 0, \sum_{\bw \in \W} \bb_{\bw}=1\right\}$ denote the set of probability distributions on $\W$. 
We define $\Y(\W)=\left\{\bw-\bw^{\prime}: \forall \bw, \bw^{\prime} \in \W, \bw \neq \bw^{\prime}\right\}$ as the directions obtained from the differences between each pair of arms and $\Y^*(\W)=\left\{\bw_*-\bw: \forall \bw \in \W \backslash \bw_*\right\}$ as the directions obtained from the differences between the optimal arm and each suboptimal arm.



\section{Pure Exploration in Single-Task Bilinear Bandits}
\label{sec:pure_exp_bi}
In this section, we consider pure exploration in a single-task bilinear bandit setting as a warm-up to the main goal of learning representations for the multi-task bilinear bandit. 
To our knowledge, this is the first study of pure exploration in single-task bilinear bandits.
%
We first recall the single-task bilinear bandit setting as follows: At every round $t=1,2,\ldots$ the learner observes the reward $r_t = \bx_t^\top\bTheta_*\bz_t + \eta_t$ where the low rank hidden parameter $\bTheta_*\in\R^{d_1\times d_2}$ is unknown to the learner, $\bx_t\in\R^{d_1}$, $\bz_t\in\R^{d_2}$ are visible to the learner, and $\eta_t$ is a $1$-sub-Gaussian noise. We assume that the matrix $\bTheta_*$ has a low rank $r$ which is known to the learner and $d_1, d_2 \gg r$.
Finally recall that the goal is to identify the optimal left and right arms $\bx_*, \bz_*$ with a minimum number of samples. 


We propose a phase-based, two-stage arm elimination algorithm called \textbf{G}-\textbf{O}ptimal Design for \textbf{B}i\textbf{lin}ear Bandits (abbreviated as \gob). 
\gob\ proceeds in phases indexed by $\ell=1,2,\ldots$ As this is a pure-exploration problem, the total number of samples is controlled by the total phases which depends on the intrinsic problem complexity. 
Each phase $\ell$ of \gob\ consists of two stages; the estimation of $\bTheta_*$ stage, which runs for $\tau^E_\ell$ rounds, and pure exploration in rotated arms stage that runs for $\tau^G_\ell$ rounds. We will define $\tau^E_\ell$ in \Cref{sec:low-dim-single}, while rotated arms and $\tau^G_\ell$ are defined in \Cref{sec:optimal-design-single}.
At the end of every phase, \gob\ eliminates sub-optimal arms to build the active set for the next phase and stops when only the optimal left and right arms are remaining.
Now we discuss the individual stages that occur at every phase $\ell$ of \gob.

\subsection{Estimating Subspaces of $\bTheta_*$ (Stage 1 of the $\ell$-th phase)}
\label{sec:low-dim-single}
In the first stage of phase $\ell$, \gob\ estimates the row and column sub-spaces $\bTheta_*$. Then \gob\ uses these estimates to reduce the bilinear bandit problem in the original ambient dimension $p\coloneqq d_1d_2$ to a lower effective dimension $k\coloneqq (d_1 + d_2)r$. 
%
%
%
%
To do this, \gob\ first vectorizes the $\bx\in\R^{d_1},\bz\in\R^{d_2}$ into a new vector $\ow\in\R^{p}$ and then solves the $E$-optimal design in Step $3$ of \Cref{alg:bandit-pure}  \citep{pukelsheim2006optimal,jun2019bilinear,du2023multi}. 
%
Let the solution to the $E$-optimal design problem at the stage $1$ of $\ell$-th phase be denoted by $\bb^E_\ell$. Then \gob\ samples each $\ow$ for $\lceil \tau^E_\ell\bb^E_{\ell,\ow} \rceil$ times, where $\tau^E_\ell = \tO(\sqrt{d_1 d_2 r}/ S_r)$ (step $7$ of \Cref{alg:bandit-pure}).
In this paper, we sample an arm $\lceil \tau^E_\ell\bb^E_{\ell,\ow} \rceil$ number of times. However, this may lead to over-sampling of an arm than what the design ($G$ or $E$-optimal) is actually suggesting.
However, we can match the number of allocations of an arm to the design using an \textit{efficcient Rounding Procedures} (see \citet{pukelsheim2006optimal, fiez2019sequential}).
%
%
%
%
%
%
Let $\wTheta_{\ell}$ be estimate of $\bTheta_*$ in stage $1$ of phase $\ell$. \gob\ estimates this by solving the following well-defined regularized minimization problem with nuclear norm penalty:
\begin{align}
\hspace*{-1.5em}\wTheta_{\ell}=\argmin_{\bTheta \in \mathbb{R}^{d_1 \times d_2}} L_{\ell}(\bTheta)+\gamma_\ell\|\bTheta\|_{\mathrm{nuc}},  \quad L_{\ell}(\bTheta)=\langle\bTheta, \bTheta\rangle-\tfrac{2}{\tau^E_\ell} \sum_{s=1}^{\tau^E_\ell}\langle\widetilde{\psi}_\nu(r_s \cdot Q(\bx_s\bz_s^\top)), \bTheta\rangle \label{eq:convex-prog}
\end{align}
where $Q(\cdot)$, $\widetilde{\psi}_\nu(\cdot)$, are appropriate functions stated in  \Cref{def:score-func}, \ref{def:psi-tilde} respectively in \Cref{app:app-stein-lemma}. The $Q(\cdot)$ function takes as input the rank-one matrix $\bx_s\bz_s^\top$ which is obtained after reshaping $\ow_s$. Note that $\bx_s$, and $\bz_s$ are the observed vectors in $d_1$ and $d_2$ dimension and $\wTheta_\ell \in \R^{d_1\times d_2}$
Finally, set the regularization parameter  $\gamma_\ell \coloneqq 4 \sqrt{\tfrac{2\left(4 + S^2_0\right) C d_1 d_2 \log \left(2\left(d_1+d_2\right) / \delta\right)}{\tau^E_\ell}}$. This is in step $8$ of \Cref{alg:bandit-pure}.


\subsection{Optimal Design for Rotated Arms (Stage 2 of $\ell$-th phase)}
\label{sec:optimal-design-single}
%
In stage $2$ of phase $\ell$, \gob\ leverages the information about the learned sub-space of $\bTheta_*$ to rotate the arm set and then run the optimal design on the rotated arm set.
Once we recover $\wTheta_\ell$, one might be tempted to run a pure exploration algorithm \citep{soare2014best, fiez2019sequential, katz2020empirical, zhu2021pure} to identify $\bx_*$ and $\bz_*$. However, then the sample complexity will scale with $d_1d_2$. 
In contrast \gob\ uses the information about the learned sub-space of $\bTheta_*$ to reduce the problem from ambient dimension $d_1d_2$ to effective dimension $(d_1 + d_2)r$.
%
%
This reduction is done as follows:
Let $\wTheta_{\ell} = \wU_\ell\wD_{\ell}\wV_{\ell}^\top$ be the SVD of $\wTheta_{\ell}$ in the $\ell$-th phase. Let $\wU^\ell_{\perp}$ and $\wV^\ell_{\perp}$ be orthonormal bases of the complementary subspaces of $\wU_\ell$ and $\wV_\ell$ respectively. Let $\X_\ell$ and $\Z_\ell$ be the active set of arms in the stage $2$ of phase $\ell$. Then rotate the arm sets such that new rotated arm sets are as follows:
\begin{align}
   \uX_\ell=\{\ux=[\wU_\ell \wU_\ell^\perp]^{\top} \bx \mid \bx \in \X_\ell\}, 
\uZ_\ell = \{\uz=[\wV_\ell \wV_\ell^{\perp}]^{\top} \mathbf{z} \mid \bz \in \Z_\ell\}. \label{eq:rotated-arm-set}
\end{align}
Let $\wH_\ell=[\wU_\ell \wU^\perp_\ell]^{\top} \wTheta_{\ell}[\wV_\ell \wV_\ell^{\perp}]$. Then define vectorized arm set so that the last $\left(d_1-r\right) \cdot\left(d_2-r\right)$ components are from the complementary subspaces as follows:
\begin{align}
\ucW_\ell &= \left\{\left[\vec\left(\ux_{1: r} \uz_{1: r}^{\top}\right) ; \vec\left(\ux_{r+1: d_1} \uz_{1: r}^{\top}\right) ; \vec\left(\ux_{1: r} \uz_{r+1: d_2}^{\top}\right) ; \right.\right.\nonumber\\
&\qquad\qquad\qquad\left.\left.\vec\left(\ux_{r+1: d_1} \uz_{r+1: d_2}^{\top}\right)\right] \in \mathbb{R}^{d_1 d_2}: \ux \in \X_{\ell}, \uz \in \Z_\ell\right\}\nonumber\\
\wtheta_{\ell, 1: k}&=[\vec(\wH_{\ell,1: r, 1: r}) ; \vec(\wH_{\ell,r+1: d_1, 1: r}) ; \vec(\wH_{\ell,1: r, r+1: d_2})], \nonumber\\
 \wtheta_{\ell, k+1: p} &=\vec(\wH_{\ell,r+1: d_1, r+1: d_2}).
\label{eq:rotated-arm-set-phase}
\end{align}
which implies $\|\wtheta_{k+1: p}\|_2=O\left(d_1d_2 r / \tau^E_\ell\right)$ by \Cref{theorem:kang-low-rank} in \Cref{app:prob-tools}. So the last $p-k$ components of $\wtheta_\ell$ are very small compared to the first $k$ components.
Hence, \gob\ has now reduced the $d_1d_2$ dimensional linear bandit to $(d_1 + d_2)r$ dimensional linear bandit using \eqref{eq:rotated-arm-set}, \eqref{eq:rotated-arm-set-phase}. This is shown in step $10$ of \Cref{alg:bandit-pure}.

Now in stage $2$ of phase $\ell$, \gob\ implements $G$-optimal design \citep{pukelsheim2006optimal,fiez2019sequential} in the rotated arm set $\uX_\ell, \uZ_\ell$ defined in  \eqref{eq:rotated-arm-set}.
%
%
%
%
%
To do this, first \gob\ defines the rotated vector $\uw = [\ux_{1:d_1};\uz_{1:d_2}] \in\R^p$ that belong to the set $\ucW_\ell$. 
Then \gob\ solves the $G$-optimal design \citep{pukelsheim2006optimal} as follows:
\begin{align}
    \wb^G_\ell = \argmin_{\bb_{\uw}}\max_{\uw,\uw'\in\ucW_\ell} \|\uw-\uw'\|^2_{(\sum_{\uw\in\ucW} \bb_{\uw}\uw\ \uw^\top + \bLambda_\ell/n )^{-1}}.  \label{eq:bi-level-opt-single}
\end{align}
This is shown in step $11$ of \Cref{alg:bandit-pure} and $\bLambda_\ell$ is defined in \eqref{eq:bLambda}. It can be shown that sampling according to $\wb^G_\ell$ leads to the optimal sample complexity. This is discussed in \Cref{remark:g-optimal} in \Cref{app:G-optimal-remark}.
The key point to note from \eqref{eq:bi-level-opt-single} is that due to the estimation in the rotated arm space $\ucW_\ell$ we are guaranteed that the support of $\mathbf{supp}(\wb^G_\ell) \leq \tO(k(k+1)/2)$ \citep{pukelsheim2006optimal}. On the other hand, if the G-optimal design of \citet{fiez2019sequential,katz2020empirical} are run in $d_1d_2$ dimension then the support of $\wb^G_\ell$ will scale with $d_1d_2$ which will lead to higher sample complexity. 
Then \gob\ samples each $\uw\in\ucW_\ell$ for $\lceil \tau^G_\ell\bb^G_{\ell,\uw} \rceil$ times, where $\tau^G_\ell \coloneqq \lceil \frac{8 B^\ell_*\rho^G(\Y(\W_\ell))\log (4\ell^2|\W| / \delta)}{\epsilon_\ell^2}\rceil$. Note that the total length of phase $\ell$, combining stages $1$ and $2$ is $(\tau^E_\ell + \tau^G_\ell)$ rounds. Observe that the stage $1$ design is on the whole arm set $\ocW$ whereas stage $2$ design is on the refined active set $\ucW_\ell$.

Let the observed features in stage $2$ of phase $\ell$ be denoted by $\uW_\ell \in \R^{\tau^G_\ell\times p}$, and $\br_\ell \in \R^{\tau^G_\ell}$ be the observed rewards.
Define the diagonal matrix $\bLambda_\ell$ as
\begin{align}
    \bLambda_\ell = \mathbf{diag}[\underbrace{\lambda,\ldots,\lambda}_{k},\underbrace{\lambda^\perp_\ell,\ldots,\lambda^\perp_\ell}_{p-k}] \label{eq:bLambda}
\end{align}
where, $\lambda^\perp_\ell \coloneqq \tau^G_{\ell-1}/8k\log(1+\tau^G_{\ell-1}/\lambda) \gg \lambda$. 
Deviating from \citet{soare2014best, fiez2019sequential} \gob\ constructs a regularized least square estimator 
at phase $\ell$ as follows
\begin{align}
    \wtheta_\ell=\argmin_{\btheta\in\R^p} \frac{1}{2}\|\uW_\ell \btheta-\br_\ell\|_2^2+\frac{1}{2}\|\btheta\|_{\bLambda_\ell}^2. \label{eq:least-sq-reg}
\end{align}
This regularized least square estimator in \eqref{eq:least-sq-reg} forces the last $p-k$ components of $\wtheta_\ell$ to be very small compared to the first $k$ components. 
Then \gob\ builds the estimate $\wtheta_\ell$ from \eqref{eq:least-sq-reg} only from the observations from this phase (step $13$ in \Cref{alg:bandit-pure}) and eliminates sub-optimal actions in step $14$ in \Cref{alg:bandit-pure} using the estimator $\wtheta_\ell$.
%
%
%
%
%
Finally \gob\ eliminates sub-optimal arms to build the next phase active set $\ucW_\ell$ and stops when $|\ucW_\ell| = 1$. 
\gob\ outputs the arm in $\ucW_{\ell}$ and reshapes it to get the $\widehat{\bx}_{*}$ and $\widehat{\bz}_{*}$. The full pseudocode is presented in \Cref{alg:bandit-pure}.

\begin{algorithm}
\caption{G-Optimal Design for Bilinear Bandits (\gob) for single-task setting}
\label{alg:bandit-pure}
\begin{algorithmic}[1]
\State Input: arm set $\X,\Z$, confidence $\delta$, rank $r$ of $\bTheta_*$, spectral bound $S_r$ of $\bTheta_*$,  $S, S_\ell^{\perp} \coloneqq  \frac{8d_1 d_2 r}{\tau^E_{\ell} S^2_{r}} \log \left(\frac{d_1+d_2}{\delta_\ell}\right), \lambda, \lambda_\ell^{\perp} \coloneqq \tau^G_{\ell-1}/8(d_1+d_2)r\log(1+\frac{\tau^G_{\ell-1}}{\lambda})$. Let $p\coloneqq d_1d_2$, $k\coloneqq (d_1+d_2)r$.
\State Let $\!\ucW_1 \!\leftarrow\! \ucW, \ell \!\leftarrow 1$, $\tau^G_0 \coloneqq \log (4\ell^2|\X| / \delta)$. 
Define $\bLambda_\ell$ as in \eqref{eq:bLambda},
$B^\ell_* \coloneqq (8\sqrt{\lambda} S +\sqrt{\lambda_\ell^{\perp}} S^{\perp}_{\ell})$.
\State Define a vectorized arm $\ow \coloneqq \left[\bx_{1: d_1}; \bz_{1: d_2}\right]$ and $\ow\in\ocW$. Let $\tau^E_{\ell} \coloneqq \frac{\sqrt{8d_1 d_2 r\log (4\ell^2|\W| / \delta_\ell)}}{S_{r}}$. Let the $E$-optimal design be $\small\bb^E_\ell \coloneqq \argmin_{\bb \in \triangle_{\ocW}}\big\|\big(\sum_{\ow\in\ocW} \bb_{\ow}\ow \ \ow^\top\big)^{-1}\big\|$.
\While{$\left|\ucW_{\ell}\right|>1$}
\State $\epsilon_{\ell}=2^{-\ell}$, $\delta_\ell = \delta/\ell^2$.
\State \textbf{(Stage 1:) Explore the Low-Rank Subspace}
\State Pull arm $\ow \in \ocW$ exactly $\left\lceil\wb^E_{\ell, \ow} \tau^E_{\ell}\right\rceil$ times
and observe rewards $r_t$, for $t=1, \ldots, \tau^E_{\ell}$.
\State Compute $\wTheta_{\ell}$ using \eqref{eq:convex-prog}.
\State \textbf{(Stage 2:) Reduction to low dimensional linear bandits}  
\State Let the SVD of $\wTheta_{\ell} = \wU_\ell\wD_{\ell}\wV_{\ell}^\top$. Rotate arms in active set $\ucW_{\ell-1}$  to build $\ucW_\ell$ following \eqref{eq:rotated-arm-set-phase}.
%
%
\State Let $\wb^G_{\ell} \coloneqq \argmin_{\bb_{\uw}}\max_{\uw,\uw'\in\ucW_\ell} \|\uw-\uw'\|^2_{(\sum_{\uw\in\ucW} \bb_{\uw}\uw\ \uw^\top + \bLambda_\ell/n )^{-1}}$. 
\State Define $\rho^G(\Y(\ucW_{\ell})) \coloneqq \min_{\bb_{\uw}}\max_{\uw,\uw'\in\ucW_\ell} \|\uw-\uw'\|^2_{(\sum_{\uw\in\ucW} \bb_{\uw}\uw\ \uw^\top + \bLambda_\ell/n )^{-1}}$. 
\State Set 
$\tau^G_{\ell} \!\!\coloneqq\!\! \lceil\frac{64 B^\ell_*\rho^G(\Y(\W_\ell))\log (4\ell^2|\W| / \delta_\ell)}{\epsilon_\ell^2}\rceil$. 
Then pull arm $\uw \in \ucW$ exactly $\left\lceil\wb^G_{\ell, \uw} \tau^G_{\ell}\right\rceil$ times  \indent and construct the least squares estimator $\wtheta_{\ell}$ using only the observations of this phase where \indent $\wtheta_\ell$ is defined in \eqref{eq:least-sq-reg}. Note that $\wtheta_\ell$ is also rotated following \eqref{eq:rotated-arm-set-phase}.
\State Eliminate arms such that 
$\ucW_{\ell+1} \leftarrow \ucW_{\ell} \backslash\{\uw \in \ucW_{\ell}: \max_{\uw^{\prime} \in \ucW_{\ell}}\langle \uw^{\prime}-\uw, \wtheta_{\ell}\rangle>2 \epsilon_{\ell}\}$
\State $\ell \leftarrow \ell+1$
\EndWhile
\State Output the arm in $\ucW_{\ell}$ and reshape to get the $\widehat{\bx}_*$ and $\widehat{\bz}_*$
\end{algorithmic}
\end{algorithm}

\subsection{Sample Complexity Analysis of Single-Task \gob}
We now analyze the sample complexity of \gob\ in the single-task setting through the following theorem.

\begin{customtheorem}{1}\textbf{(informal)}
\label{thm:single-task}
With probability at least $1 - \delta$, \gob\ returns the best arms $\bx_*$, $\bz_*$, and the number of samples used is bounded by
    $\tO\left(\tfrac{(d_1+d_2)r}{\Delta_{}^2} + \tfrac{\sqrt{d_1 d_2 r}}{S_r}\right)$.
\end{customtheorem}

\begin{discussion}
In \Cref{thm:single-task} the first quantity is the number of samples needed to identify the best arms $\bx_*$, $\bz_*$ while the second quantity is the number of samples to learn $\bTheta_*$ (which is required to find the best arms).
Note that the magnitude of $S_{r}$ would be free of $d_1,d_2$ since $\bTheta_*$ contains only $r$ nonzero singular values and $\|\bTheta_*\|\leq 1$, and hence we assume that $S_{r}=\Theta(1 / \sqrt{r})$ \citep{kang2022efficient}. So 
the sample complexity of single-task \gob\ scales as $\tO(\tfrac{(d_1+d_2)r}{\Delta_{}^2})$. However, if one runs RAGE \citep{fiez2019sequential} on the arms in $\X,\Z$ then the sample complexity will scale as $\tO(\tfrac{d_1d_2}{\Delta_{}^2})$.
\end{discussion}

\textbf{Proof (Overview) of \Cref{thm:single-task}:} \textbf{Step 1 (Subspace estimation in high dimension):} We denote the vectorized arms in high dimension as $\ow\in\ocW$. We run the $E$-optimal design to sample the arms in $\ocW$. Note that this $E$-optimal design satisfies the distribution assumption of \citet{kang2022efficient} which enables us to apply the \Cref{theorem:kang-low-rank} in \Cref{app:prob-tools}. This leads to $\|\wTheta_\ell - \bTheta_*\|_F^2 \leq \tfrac{C_1 d_1 d_2 r \log (2\left(d_1+d_2\right)/\delta)}{\tau^E_\ell}$ for some $C_1>0$. 
Also, note that in the first stage of the $\ell$-th phase by setting $\tau^E_{\ell} = \tfrac{\sqrt{8d_1 d_2 r\log (4\ell^2|\W| / \delta_\ell)}}{S_{r}}$ and sampling each arm $\ow \in \ocW$ exactly $\lceil\wb^E_{\ell, \ow} \tau^E_{\ell}\rceil$ times we are guaranteed that $\|\btheta_{k+1: p}^*\|_2=O(d_1d_2 r / \tau^E_\ell)$.
Summing up over $\ell=1$ to $\left\lceil\log _2\left(4 \Delta^{-1}\right)\right\rceil$ we get that the total sample complexity of the first stage is bounded by $\tO({\sqrt{d_1 d_2 r}}/{S_r})$.

\textbf{Step 2 (Effective dimension for rotated arms):} We  rotate the arms $\ow\in\ocW$ in high dimension to get the rotated arms $\uw\in\ucW_\ell$ in step $10$ of \Cref{alg:bandit-pure}. Then we show that the effective dimension of $\uw$ scales $8 k \log \left(1+{\tau^G_{\ell-1}}/{\lambda}\right)$ 
when $\lambda_\ell^{\perp}=\tfrac{\tau^G_{\ell-1}}{8k \log \left(1+\tau^G_{\ell-1} / \lambda\right)}$ in \Cref{lemma:valko-lemma-effective-dim-exp} of \Cref{app:single-bi}. Note that this requires a different proof technique than  \citet{valko2014spectral} where the budget $n$ is given apriori and effective dimension scales with $\log(n)$. 
This step also diverges from the pure exploration proof technique of \citet{fiez2019sequential,katz2020empirical} as there is no parameter $\lambda_\ell^\perp$ to control during phase $\ell$, and the effective dimensions in those papers do not depend on phase length.

\textbf{Step 3 (Bounded Support):} For any phase $\ell$, 
we can show that $ 1 \leq \rho^G(\Y(\ucW_{\ell})) \leq p / \gamma_{\Y}^2$ where, $\gamma_{\Y}=\max \{c>0: c \Y \subset \operatorname{conv}(\ucW \cup-\ucW)\}$ is the gauge norm of $\Y$ \citep{rockafellar2015convex}. Note that this is a worst-case dependence when $\rho^G(\Y(\ucW_{\ell}))$ scales with $p$. Substituting this value of $\rho^G(\Y(\ucW_{\ell}))$ in the definition of $\lambda_\ell^\perp$ we can show that $\bLambda_\ell$ does not depend on $\uw$ or $\by = \uw-\uw'$. Then following Theorem 21.1 in \citet{lattimore2020bandit}  we can show that the $G$-optimal design $\wb^G_\ell$ is equivalent to $D$-optimal design $\wb^D_\ell = \argmax_{\bb} \log \tfrac{\left|\sum_{\uw\in\ucW_\ell}\bb_{\uw}\uw\ \uw^\top + \bLambda_\ell\right|}{|\bLambda_\ell|}$. Then using Frank-Wolfe algorithm \citep{jamieson2022interactive} 
we can show the support $\wb^G_\ell$ or equivalently $\wb^D_\ell$ is bounded by at most $\frac{8k\log(1+{\tau^G_{\ell-1}}/{\lambda})(8k\log(1+{\tau^G_{\ell-1}}/{\lambda}) + 1)}{2}$. This is shown in \Cref{lemma:equivalence} (\Cref{app:single-bi}).

\textbf{Step 4 (Phase length and Elimination):} Using the \Cref{lemma:equivalence}, concentration  \Cref{eq:conc-lemma-bilin-exp}, and using the log determinant inequality in \Cref{lemma:valko-lemma-effective-dim-exp} and \Cref{prop:conc} (\Cref{app:single-bi}) we show that the phase length in the second stage is given by $\tau^G_\ell = \lceil\tfrac{8B^\ell_* \rho(\Y(\ucW_\ell))  \log \left(2|\W| / \delta\right)}{(\bx^\top(\wtheta_\ell-\btheta^*))^2}\rceil$. This is discussed in \Cref{dis:phase-length} (\Cref{app:single-bi}). We show in \Cref{lemma:elim} (\Cref{app:single-bi}) that setting this phase length and sampling each active arm in $\ucW_\ell$ exactly $\lceil\wb_{\ell, \uw} \tau^G_{\ell}\rceil$ times results in the elimination of sub-optimal actions with high probability.

\textbf{Step 5 (Total Samples):} We first show that the total samples in the second phase are bounded by $ O(\frac{k}{\gamma_{\Y}^2} \log (\frac{k\log_2( \Delta^{-1})|\ucW|}{\delta}) \lceil\log _2(\Delta^{-1})\rceil)$ where the effective dimension $k=(d_1+d_2)r$.
Finally, we combine the total samples of phase $\ell$ as $(\tau^E_\ell + \tau^G_\ell)$. The final sample complexity is given by summing over all phases from $\ell=1$ to $\left\lceil\log _2\left(4 \Delta^{-1}\right)\right\rceil$. 
%
%
%
The claim of the theorem follows by noting $\tO({k}/{\gamma_{\Y}^2})\leq \tO({k}/{\Delta^2})$.


\section{Multi-task Representation Learning}
\label{sec:multi-task}
In this section, we extend \gob\ to multi-task representation learning for the bilinear bandit setting. 
In the multi-task setting, we now have $M$ tasks, where each task $m\in[M]$ has a reward model stated in \eqref{eq:reward-model-multi}. 
%
%
The learning proceeds as follows: At each round $t=1, 2, \cdots$, for each task $m \in[M]$, the learner selects a left and right action $\bx_{m,t} \in \X$ and $\bz_{m,t}\in\Z$. After the player commits the batch of actions for each task $\left\{\bx_{m, t}, \bz_{m, t}: m \in[M]\right\}$, it receives the batch of rewards $\left\{r_{m, t}: m \in[M]\right\}$.
Finally recall that the goal is to identify the optimal left and right arms $\bx_{m,*}, \bz_{m,*}$ for each task $m$ with a minimum number of samples. We now state the following assumptions to enable representation learning across tasks.
%
%
\begin{assumption}\textbf{(Low-rank Tasks)}
    \label{assm:low-rank}
    We assume that the hidden parameter $\bTheta_{m,*}$ for all the $m\in[M]$ have a decomposition $\bTheta_{m,*} = \bB_1\bS_{m,*}\bB_2^\top$ and each $\bS_{m,*}$ has rank $r$.
\end{assumption}
This is similar to the assumptions in \citet{yang2020impact, yang2022nearly, du2023multi} ensuring the feature extractors are shared across tasks in the bilinear bandit setting.
\begin{assumption}\textbf{(Diverse Tasks)}
    \label{assm:diverse-task}
    We assume that $\sigma_{\min}(\tfrac{1}{M}\sum_{m=1}^M\bTheta_{m,*}) \geq \tfrac{c_0}{S_r}$, for some $c_0 >0$, $S_r$ is the $r$-th largest singular value of $\bTheta_{m,*}$ and $\sigma_{\min}(\bA)$ denotes the minimum eigenvalue of matrix $\bA$.  
\end{assumption}
This assumption is similar to the diverse tasks assumption of \citet{yang2020impact,yang2022nearly,tripuraneni2021provable,du2023multi} and ensures the possibility of
recovering the feature extractors $\bB_1$ and $\bB_2$ shared across tasks.

Our extension of \gob\ to the multi-task setting is now a phase-based, \textit{three-stage} arm elimination algorithm.
%
In \gob\ each phase $\ell=1,2,\ldots$ consists of three stages; the stage for estimation of feature extractors $\bB_1,\bB_2$, which runs for $\tau^E_{\ell}$ rounds, the stage for estimation of $\bS_{m,*}$ which runs for $\sum_m\widetilde{\tau}^E_{m,\ell}$ rounds, and a stage of pure exploration with rotated arms that runs for $\sum_m\tau^G_{m,\ell}$ rounds. We will define ${\tau^E_{m,\ell}}$ in \Cref{sec:stage-1-multi}, $\widetilde{\tau}^E_{m,\ell}$ in \Cref{sec:stage-2-multi}, while the rotated arms and $\tau^G_{m,\ell}$ are defined in \Cref{sec:stage-3-multi}. 
At the end of every phase, \gob\ eliminates sub-optimal arms to build the active set for the next phase and stops when only the optimal left and right arms are remaining.
Now we discuss the individual stages that occur at every phase $\ell=1,2,\ldots$ for multi-task \gob. 

\subsection{Estimating Feature Extractors $\bB_1$ and $\bB_2$ (Stage 1 of Phase $\ell$)}
\label{sec:stage-1-multi}

%
In the first stage of phase $\ell$, \gob\ leverages the batch of rewards $\left\{r_{m, t}: m \in[M]\right\}$ at every round $t$ from $M$ tasks to learn the feature extractors $\bB_1$ and $\bB_2$. 
%
%
To do this, \gob\ first vectorizes the $\bx\in\X,\bz\in\Z$ into a new vector $\ow = \left[\bx_{1: d_1}; \bz_{1: d_2}\right]\in\ocW_m$ and then solves the $E$-optimal design in step $3$ of \Cref{alg:bandit-pure-multi}.
%
Similar to the single-task setting (\Cref{sec:pure_exp_bi})  
\gob\ samples each $\ow\in\ocW_m$ for $\lceil \tau^E_{\ell}\bb^E_{\ell,\ow} \rceil$ times for each task $m$, where $\tau^E_{\ell} = \widetilde{O}(\sqrt{d_1 d_2 r}/ S_r)$ 
and $\bb^E_{\ell,\ow}$ is the solution to $E$-optimal design on $\ow$. 
Let the sampled arms for each task $m$ at round $s$ be denoted by $\bx_{m,s}$, $\bz_{m,s}$ which is obtained after reshaping $\ow_s$. 
Then it builds the estimator $\wZ_{\ell}$ as follows:
\begin{align}
\wZ_{\ell} &=\argmin_{\bTheta \in \R^{d_1 \times d_2}} L_{\ell}(\bTheta)+\gamma_\ell\|\bTheta\|_{\mathrm{nuc}},\nonumber\\  
L_{\ell}(\bTheta) &=\langle\bTheta, \bTheta\rangle-\frac{2}{M\tau^E_{\ell}} \sum_{m=1}^M \sum_{s=1}^{\tau^E_{\ell}}\langle\widetilde{\psi}_\nu(r_{m,s} \cdot Q(\bx_{m,s}\bz_{m,s}^\top)), \bTheta\rangle \label{eq:convex-prog-multi}
\end{align}
%
Then it performs SVD decomposition on $\wZ_{\ell}$, and let $\wB_1$, $\wB_2$ be the top-$k_1$ and top-$k_2$ left and right singular vectors of $\wZ_{\ell}$ respectively. These are the estimation of the feature extractors $\bB_1$ and $\bB_2$. 
%


\subsection{Estimating Hidden Parameter $\bS_{m,*}$ per Task (Stage 2 of phase $\ell$)}
\label{sec:stage-2-multi}

In the second stage of phase $\ell$, the goal is to  recover the hidden parameter  $\bS_{m,*}$ for each task $m$. \gob\ proceeds as follows: First, let $\tg_{m} = \bx^\top\wB_{1,\ell}$ and $\tv_{m} = \bz^\top\wB_{2,\ell}$ be the latent left and right arm respectively for each $m$. 
Then \gob\ defines the vector $\tw = [\tg_m ; \tv_m]\in\tW_m$ and then solves the $E$-optimal design in step $11$ of \Cref{alg:bandit-pure-multi}.
It then samples for each task $m$, the latent arm $\tw\in\tW_m$ for $\lceil \widetilde{\tau}^E_{m,\ell}\widetilde{\bb}^E_{m,\ell,\tw} \rceil$ times, where $\widetilde{\tau}^E_{m,\ell} \coloneqq \widetilde{O}(\sqrt{k_1 k_2 r}/ S_r)$ and 
$\widetilde{\bb}^E_{m,\ell,\tw}$ is the solution to $E$-optimal design on $\tw$. 
Then it builds estimator $\wS_{m, \ell}$ for each task $m$ in step $12$ as follows:
\begin{align}
    \wS_{m,\ell} &=\argmin _{\bTheta \in \R^{k_1 \times k_2}} L'_{\ell}(\bTheta)+\gamma_\ell\|\bTheta\|_{\mathrm{nuc}}, \nonumber\\
    L'_{\ell}(\bTheta) &=\langle\bTheta, \bTheta\rangle-\frac{2}{\widetilde{\tau}^E_{m,\ell}}\!\! \sum_{s=1}^{\widetilde{\tau}^E_{m,\ell}}\langle\widetilde{\psi}_\nu(r_{m,s} \cdot Q(\tg_{m,s}\tv_{m,s}^\top)), \bTheta\rangle \label{eq:convex-prog-multi-wS}
\end{align}
Once \gob\ recovers the $\wS_{m,\ell}$ for each task $m$ it has reduced the $d_1d_2$ bilinear bandit to a $k_1k_2$ dimension bilinear bandit where the left and right arms are $\tg_m\in\tG_m$, $\tv_m\in\tV_m$ respectively. 

\vspace*{-0.5em}
\subsection{Optimal Design for Rotated Arms per Task (Stage 3 of phase $\ell$)}
\vspace*{-0.5em}
\label{sec:stage-3-multi}

In the third stage of phase $\ell$, similar to \Cref{alg:bandit-pure}, the multi-task \gob\ defines the rotated arm set $\uG_m,\uV_m$ for each task $m$ for these $k_1 k_2$ dimensional bilinear bandits. 
Let the SVD of $\wS_{m,\ell} = \wU_{m,\ell}\wD_{m,\ell}\wV^\top_{m,\ell}$. 
Define $\wH_{m,\ell}=[\wU_{m,\ell} \wU^\perp_{m,\ell}]^{\top} \wS_{m,\ell}[\wV_{m,\ell} \wV_{m,\ell}^{\perp}]$. Then define the vectorized arm set so that the last $\left(k_1-r\right) \cdot\left(k_2-r\right)$ components are from the complementary subspaces as follows:
\begin{align}
\ucW_{m,\ell} &= \left\{\left[\vec\left(\tg_{m,1: r} \tv_{m,1: r}^{\top}\right) ; \vec\left(\tg_{m,r+1: k_1} \tv_{m,1: r}^{\top}\right) ; \vec\left(\tg_{m,1: r} \tv_{m,r+1: k_2}^{\top}\right) ; \right.\right.\nonumber\\
&\left.\left.\qquad\qquad\qquad\vec\left(\tg_{m,r+1: k_1} \tv_{m,r+1: k_2}^{\top}\right)\right]\right\}\nonumber\\
\wtheta_{m,\ell, 1: k} &=[\vec(\wH_{m,\ell,1: r, 1: r}) ; \vec(\wH_{m,\ell,r+1: k_1, 1: r}) ; \vec(\wH_{m,\ell,1: r, r+1: k_2})], \nonumber\\
 \btheta_{\ell, k+1: p} &=\vec(\wH_{m,\ell,r+1: k_1, r+1: k_2}).
\label{eq:rotated-arm-set-phase-multi}
\end{align}
This is shown in step $14$ of \Cref{alg:bandit-pure-multi}. 
Now we proceed similarly to \Cref{sec:optimal-design-single}. We construct a per-task optimal design for the rotated arm set $\uV_m, \uG_m$ and define the $\uw = [\tg_{m,1:d_1};\tv_{m,1:d_2}]$ and $\tw\in\tW_m$ where $\tg_{m}\in\uG_m$ and $\tv_m\in\uV_m$ respectively. 
Following \eqref{eq:bi-level-opt-single} we know that to minimize the sample complexity for the $m$-th bilinear bandit we need to sample according to $G$-optimal design
\begin{align}
    \wb^G_{m,\ell} = \argmin_{\bb_{m,\uw}}\max_{\uw,\uw'\in\ucW_{m,\ell}} \|\uw-\uw'\|^2_{(\sum_{\uw\in\ucW_{m}} \bb_{m,\uw}\uw\ \uw^\top + \bLambda_{m,\ell}/n )^{-1}} \label{eq:bi-level-opt-multi}
\end{align} 
Then \gob\ runs $G$-optimal design on the arm set $\ucW_\ell$ following the \eqref{eq:bi-level-opt-multi} and then samples each $\uw\in\ucW_{m,\ell}$ for $\lceil \tau^G_{m,\ell}\wb^G_{m,\ell,\uw} \rceil$ times where $\wb^G_{m,\ell,\uw}$ is the solution to the $G$-optimal design, and $\tau^G_\ell$ is defined in step $17$ of \Cref{alg:bandit-pure-multi}. So the total length of phase $\ell$, combining stages $1,2$ and $3$ is $(\tau^E_\ell + \sum_m\widetilde{\tau}^E_{m,\ell} + \sum_m\tau^G_{m,\ell})$ rounds. Observe that the stage 1 and 2 design is on the whole arm set $\ocW, \tW_m$ whereas the stage $3$ design is on the refined active set $\ucW_{m,\ell}$.
%
%
Let at the stage $3$ of $\ell$-th phase the actions sampled be denoted by the matrix $\uW_{m,\ell}\in\R^{\tau^G_{m,\ell}\times k_1k_2}$ and observed rewards $\mathbf{r}_m\in\R^{\tau^G_{m,\ell}\times k_1k_2}$. Define the positive diagonal matrix $\bLambda_{m,\ell}$ according to \eqref{eq:bLambda} but set $p=k_1k_2$ and $k=(k_1 + k_2)r$. 
%
Then similar to \Cref{sec:optimal-design-single} we can build for each task $m$ only from the observations from this phase 
%
%
\begin{align}
\wtheta_{m,\ell}=\arg \min_{\btheta} \tfrac{1}{2}\|\uW_{m,\ell} \btheta-\mathbf{r}_m\|_2^2+\tfrac{1}{2}\|\btheta\|_{\bLambda_{m,\ell}}^2
\label{eq:btheta-phase-multi}
\end{align}
Finally \gob\ eliminates the sub-optimal arms using the estimator $\wtheta_{m,\ell}$ to build the next phase active set $\ucW_{m,\ell}$ and stops when $|\ucW_{m,\ell}| = 1$. The full pseudo-code is given in \Cref{alg:bandit-pure-multi}.

\vspace*{-0.6em}
\begin{algorithm}
\caption{G-Optimal Design for Bilinear Bandits (\gob) for multi-task setting}
\label{alg:bandit-pure-multi}
\begin{algorithmic}[1]
\State Input: arm set $\X,\Z$, confidence $\delta$, rank $r$ of $\bTheta_*$, spectral bound $S_r$ of $\bTheta_*$, $S, S_{m,\ell}^{\perp} =  \tfrac{8k_1 k_2 r}{\widetilde{\tau}^E_{m,\ell} S_{r}^2} \log (\tfrac{k_1+k_2}{\delta_\ell}), \lambda, \lambda_{m,\ell}^{\perp} \!\!=\!\! \tfrac{\tau^G_{m,\ell-1}}{(8(k_1\!+\!k_2)r\log(1+{\tau^G_{m,\ell-1}}/{\lambda}))}$. Let $p=k_1k_2$, $k=(k_1+k_2)r$.
\State Let $\ucW_{m,1} \!\leftarrow\! \ucW_m, \ell \!\leftarrow\! 1$, $\tau^G_0 \!\!= \log (4\ell^2|\X| / \delta)$. Define $\bLambda_{m,\ell}$ as in \eqref{eq:bLambda}, 
$B^\ell_{m,*} \!\!\coloneqq\! (8\sqrt{\lambda} S \!\!+\!\!\sqrt{\lambda_{m,\ell}^{\perp}} S^{\perp}_{m,\ell})$ 
\State Define arm $\ow = \left[\bx_{1: d_1}; \bz_{1: d_2}\right]$ and $\ow\in\ocW_m$. Let $\tau^E_{\ell} = \frac{\sqrt{8d_1 d_2 r\log (4\ell^2|\W| / \delta_\ell)}}{S_{r}}$. Let $E$-optimal design be $\small\bb^E_{\ell} \!=\! \argmin_{\bb \in \triangle_{\ocW}}\big\|(\sum_{\ow\in\ocW} \bb_{\ow}\ow \ \ow^\top)^{-1}\big\|$.
\While{$\exists m\in [M], \left|\ucW_{m,\ell}\right|>1$}
\State $\epsilon_{\ell}=2^{-\ell}$, $\delta_\ell = \delta/\ell^2$.
%
%
%
\State \textbf{(Stage 1:) Explore the Low-Rank Subspace}
\State Pull arm $\ow \in \ocW$ exactly $\lceil\wb^E_{\ell, \ow} \tau^E_{\ell}\rceil$ times for each task $m$
and observe rewards $\{r_{m,t}\}_{t=1}^{\tau^E_{\ell}}$.
\State Compute $\wZ_\ell$ using \eqref{eq:convex-prog-multi}.
%
\State \textbf{(Stage 2:) Build $\wS_{m,\ell}$ for each task $m$}
\State Let $\wB_{1,\ell}$, $\wB_{2,\ell}$ be the top-$k_1$ left and top-$k_2$ right singular vectors of $\wZ_{\ell}$ respectively. Build \indent $\tg_{m} = \bx^\top\wB_{1,\ell}$ and $\tv_{m} = \bz^\top\wB_{2,\ell}$ for all $\bx\in\X$ and $\bz\in\Z$ for each $m$. 
\State Define a vectorized arm $\tw = \left[\tg_{m,1: k_1}; \tv_{m,1: k_2}\right]$ and $\tw\in\tW_m$ for each $m$. Let $\widetilde{\tau}^E_{m,\ell} \!=\! \frac{\sqrt{8k_1 k_2 r\log (4\ell^2|\W| / \delta_\ell)}}{S_{r}}$, 
%
and $\widetilde{\bb}^E_{m,\ell} = \argmin_{\bb_m \in \triangle_{\tW_m}}\big\|\big(\sum_{\tw\in\tW_m} \bb_{m,\tw}\tw \ \tw^\top\big)^{-1}\big\|$.
\State Pull arm $\tw \in \tW_m$ exactly $\left\lceil\widetilde{\bb}^E_{m,\ell, \tw} \widetilde{\tau}^E_{m,\ell}\right\rceil$ times
and observe rewards $r_{m,t}$, for $t= \indent 1, \ldots, \widetilde{\tau}^E_{m,\ell}$, for each task $m$. Then  compute $\wS_{m,\ell}$ using \eqref{eq:convex-prog-multi-wS} for each $m$.
\State \textbf{(Stage 3:) Reduction to low dimensional linear bandits for each task $m$}  
\State SVD of $\wS_{m,\ell} \!\!=\!\! \wU_{m,\ell}\wD_{m,\ell}\wV^\top_{m,\ell}$. Rotate arms in active set $\ucW_{m,\ell-1}$  to build $\ucW_{m,\ell}$ using \eqref{eq:rotated-arm-set-phase-multi}.
\State Let $\wb^G_{m,\ell} \!\!=\!\! \argmin_{\bb_{m,\uw}}\max_{\uw,\uw'\in\ucW_{m,\ell}} \|\uw-\uw'\|^2_{(\sum_{\uw_m\in\ucW_{m}} \bb_{m,\uw}\uw_m\ \uw_m^\top + \bLambda_{m,\ell}/n )^{-1}}$. 
\State Define $\rho^G(\Y(\ucW_{m,\ell})) \!\!=\!\! \min\limits_{\bb_{m,\uw}}\max\limits_{\uw,\uw'\in\ucW_{m,\ell}} \|\uw \!-\!\uw'\|^2_{(\sum_{\uw\in\ucW_m} \bb_{m,\uw}\uw\ \uw^\top \!+\! \tfrac{\bLambda_{m,\ell}}{n} )^{-1}}.$ 
\State Set 
$\tau^G_{m,\ell} \!\!=\!\! \frac{64 B^{\ell}_{m,*}\rho^G(\Y(\W_{m,\ell}))\log (4\ell^2|\W_m| / \delta_\ell)}{\epsilon_\ell^2}$. 
Then pull arm $\uw \in \ucW_m$ for each task $m$  \indent exactly $\lceil\wb_{m,\ell, \uw} \tau^G_{m,\ell}\rceil$ times and construct the least squares estimator $\wtheta_{m,\ell}$ using only the \indent  observations of this phase where $\wtheta_{m,\ell}$ is defined in \eqref{eq:btheta-phase-multi}.
\State Eliminate arms such that 
$\ucW_{m,\ell+1} \leftarrow \ucW_{m,\ell} \backslash\left\{\uw_m \in \ucW_{m,\ell}: \max_{\uw^{\prime}_m \in \ucW_{m,\ell}}\left\langle \uw^{\prime}_m-\uw_m, \wtheta_{m,\ell}\right\rangle>2 \epsilon_{m,\ell}\right\}$
\State $\ell \leftarrow \ell+1$
\EndWhile
\State Output the arm in $\ucW_{m,\ell}$ and reshape to get the $\widehat{\bx}_{m,*}$ and $\widehat{\bz}_{m,*}$ for each task $m$. 
\end{algorithmic}
\end{algorithm}

\subsection{Sample Complexity analysis of Multi-task \gob}
We now present the sample complexity of \gob\ for the multi-task setting.

\begin{customtheorem}{2}\textbf{(informal)}
\label{thm:multi-task}
With probability at least $1 - \delta$, \gob\ returns the best arms $\bx_{m,*}$, $\bz_{m,*}$ for each task $m$, and the total number of samples is bounded by $\widetilde{O}\left(\tfrac{M(k_1+k_2)r}{\Delta_{}^2} + \tfrac{M\sqrt{k_1 k_2 r}}{S_r} + \tfrac{\sqrt{d_1 d_2 r}}{S_r}\right)$.
\end{customtheorem}

\begin{discussion}
In \Cref{thm:multi-task} the first quantity is the sample complexity to identify the best arms $\bx_{m,*}$, $\bz_{m,*}$ and the second quantity is the number of samples to learn $\bS_{m,*}$ for each task $m$. This is required to rotate the arms to reach the effective dimension of $(k_1+k_2)r$. Finally, the third quantity is the number of samples needed to learn $\bTheta_{m,*}$ (which in turn is used to estimate the feature extractors $\bB_1$ and $\bB_2$ to learn the $\bS_{m,*}$).
Again we assume that $S_{r}=\Theta(1 / \sqrt{r})$ \citep{kang2022efficient}. So 
the sample complexity of multi-task \gob\ scales as $\tO({M(k_1+k_2)r}/{\Delta_{}^2})$. However, if one runs DouExpDes \citep{du2023multi} then the sample complexity will scale as $\tO(M(k_1k_2)/\Delta^2)$ which is worse than \gob when $r\ll k_1$ or $k_2$.
\end{discussion}

\textbf{Proof (Overview) of \Cref{thm:multi-task}:} \textbf{Step 1 (Subspace estimation in high dimension):} The first steps diverge from the proof technique of \Cref{thm:single-task}. We now build the average estimator $\wZ_{\ell}$ to estimate the quantity $\bZ_* = \frac{1}{M}\sum_{m=1}^M\bTheta_{*,m}$ using \eqref{eq:convex-prog-multi}. This requires us to modify the \Cref{theorem:kang-low-rank} in \Cref{app:prob-tools} and apply Stein's lemma (\Cref{lemma:kang-b2}) to get a bound of $\|\wZ_{\ell}- \bZ_*\|_F^2 \leq \tfrac{C_1 d_1 d_2 r \log \left(2\left(d_1+d_2\right)/\delta\right)}{\tau^E_{\ell}}$ for some $C_1>0$. This is shown in \Cref{theorem:kang-low-rank-multi} in \Cref{app:multi-bi}. 
Summing up over $\ell=1$ to $\left\lceil\log _2\left(4 \Delta^{-1}\right)\right\rceil$ we get that the total samples complexity of the first stage is bounded by $\widetilde{O}({\sqrt{d_1 d_2 r}}/{S_r})$.
%
%

\textbf{Step 2 (Estimation of left and right feature extractors):} Now using the estimator in \eqref{eq:convex-prog-multi} we get a good estimation of the feature extractors $\bB_1$ and $\bB_2$. Let $\wB_{1,\ell}$, $\wB_{2,\ell}$ be the top-$k_1$ left and top-$k_2$ right singular vectors of $\wZ_\ell$ respectively. Then using the Davis-Kahan $\sin \theta$ Theorem \citep{bhatia2013matrix} in \Cref{lemma:david-kahan-U}, \ref{lemma:david-kahan-V} (\Cref{app:multi-bi}) we have $\|(\wB^{\perp}_{{1,\ell}})^{\top} \bB_1\|,\|(\wB^{\perp}_{{2,\ell}})^{\top} \bB_2\| \leq \tO(\sqrt{(d_1 + d_2)r/M \tau^E_{\ell}})$.

\textbf{Step 3 (Estimation of $\wS_{m,\ell}$ in low dimension):} Now we estimate the quantity $\wS_{m,\ell}\in\R^{k_1\times k_2}$ for each task $m$. To do this we first build the latent arms  $\tg_{m} = \bx^\top\wU_\ell$ and $\tv_{m} = \bz^\top\wV_\ell$ for all $\bx\in\X$ and $\bz\in\Z$ for each $m$, and sample them following the $E$-optimal design in 
step $12$ of \Cref{alg:bandit-pure-multi}.
%
We also show in \Cref{lemma:min-sigma} (\Cref{app:multi-bi}) that $\sigma_{\min}(\sum_{\tw\in\tW}\bb_{\tw} \tw \tw^\top) > 0$ which enables us to sample following $E$-optimal design.
Then use the estimator in \eqref{eq:convex-prog-multi-wS}. Then in \Cref{theorem:kang-low-rank-multi-wS} we show that $\|\wS_{m,\ell}- \mu^*\bS_{m,*}\|_F^2 \leq {C_1 k_1 k_2 r \log \left(\frac{2\left(k_1+k_2\right)}{\delta_\ell}\right)}/{\tau^E_{m,\ell}}$ holds with probability greater than $(1-\delta)$.
Also, note that in the second phase by setting $\widetilde{\tau}^E_{m,\ell} = {\sqrt{8k_1 k_2 r\log (4\ell^2|\W| / \delta_\ell)}}/{S_{r}}$ and sampling each arm $\ow \in \ocW$ exactly $\lceil\wb^E_{\ell, \ow} \widetilde{\tau}^E_{m,\ell}\rceil$ times we are guaranteed that $\|\btheta_{k+1: p}^*\|_2=O(k_1k_2 r / \widetilde{\tau}^E_{m,\ell})$ in the $\ell$-th phase.
Summing up over $\ell=1$ to $\left\lceil\log _2\left(4 \Delta^{-1}\right)\right\rceil$ across each task $M$ we get that the total samples complexity of the second stage is bounded by $\widetilde{O}({M\sqrt{k_1 k_2 r}}/{S_r})$.

\textbf{Step 4 (Convert to $k_1k_2$ bilinear bandits):} Once \gob\ recovers $\wS_{m,\tau^E_\ell}$ it rotates the arm set following \eqref{eq:rotated-arm-set-phase-multi} to build $\ucW_m$ to get the $k_1k_2$ bilinear bandits. The rest of the steps follow the same way as in steps $2,3$ and $4$ of proof of \Cref{thm:single-task}.

\textbf{Step 5 (Total Samples):} We show the total samples in the third phase are bounded by $ O(\frac{k}{\gamma_{\Y}^2} \log (\frac{k\log_2( \Delta^{-1})|\ucW|}{\delta}) \lceil\log _2(\Delta^{-1})\rceil)$ where the effective dimension $k=(k_1+k_2)r$.
The total samples of phase $\ell$ is given by $\tau^E_{\ell} + \sum_{m}(\widetilde{\tau}^E_{m,\ell} + \tau^G_{m,\ell})$. Finally, we get the total sample complexity by summing over all phases from $\ell\!=\!1$ to $\lceil\log _2\left(4 \Delta^{-1}\right)\rceil$. 
The claim of the theorem follows by noting $\tO({k}/{\gamma_{\Y}^2})\leq \tO({k}/{\Delta^2})$.



\section{Experiments}
\label{sec:expt}
In this section, we conduct proof-of-concept experiments on both single and multi-task bilinear bandits. 
In the single-task experiment,  we compare against the state-of-the-art RAGE algorithm  \citep{fiez2019sequential}. We show in \Cref{app:fig-expt-1} (left) that \gob\ requires fewer samples than the RAGE with an increasing number of arms.
In the multi-task experiment,  we compare against the state-of-the-art DouExpDes algorithm  \citep{du2023multi}. We show in \Cref{app:fig-expt-1} (right) that \gob\ requires fewer samples than DouExpDes with an increasing number of tasks. As experiments are not a central contribution, we defer a fuller description of the experimental set-up to \Cref{app:expt}.
%
\begin{figure}[h]
  \begin{center}
    \begin{tabular}{cc}
    \includegraphics[width=0.33\textwidth]{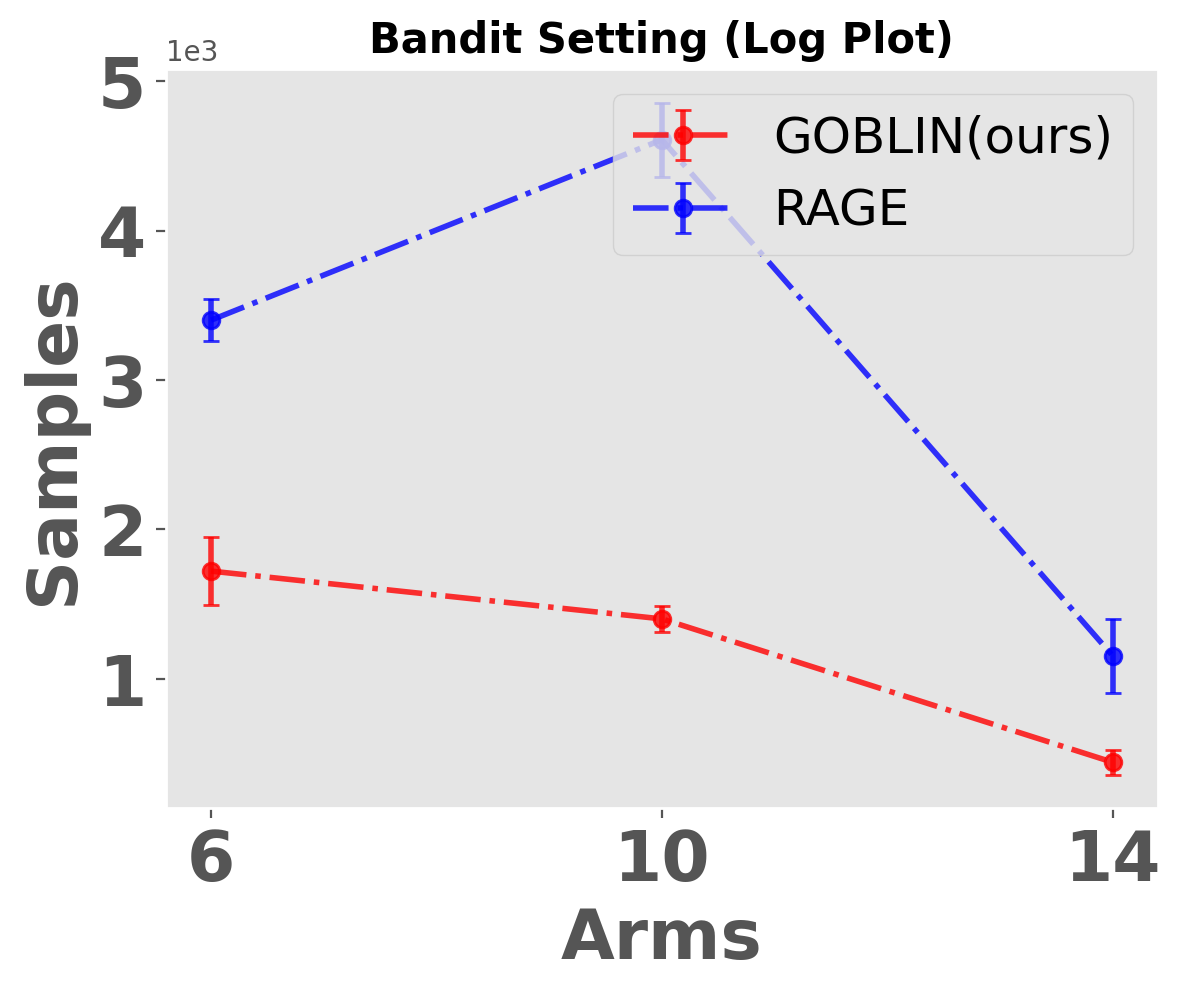} &
    \includegraphics[width=0.35\textwidth]{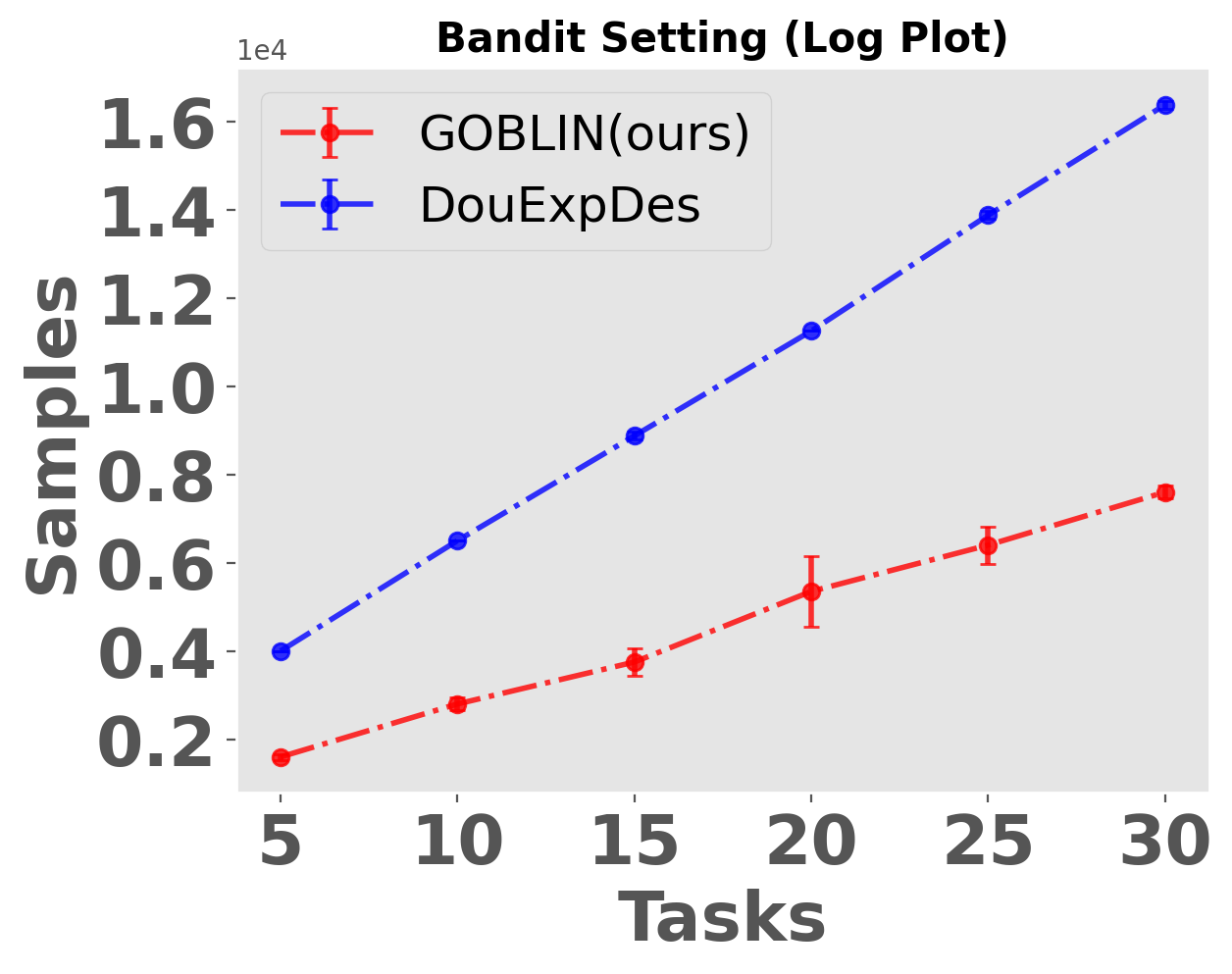} 
    \end{tabular}
    \caption{(Left) Single-task experiment: results show the number of samples required to identify the optimal action pair for differing numbers of actions. (Right) Multi-task experiment: results show the number of samples required to identify the optimal action pair for varying numbers of tasks.}
    \label{app:fig-expt-1}
    \vspace{-2.5em}
  \end{center}
\end{figure}
%

%


\section{Conclusions and Future Directions}
\label{sec:conc}
In this paper, we formulated the first pure exploration multi-task representation learning problem. We introduce an algorithm, \gob\, that achieves a sample complexity bound of $\tO((d_1+d_2)r/\Delta^2)$ which improves upon the $\tO((d_1d_2)/\Delta^2)$ sample complexity of RAGE \citep{fiez2019sequential} in a single-task setting. We then extend \gob\ for multi-task pure exploration bilinear bandit problems by learning latent features which enables  sample complexity that scales as $\tO(M(k_1+k_2)r/\Delta^2)$ which improves over the $\tO(M(k_1k_2)/\Delta^2)$ sample complexity of DouExpDes \citep{du2023multi}. 
Our analysis opens an exciting opportunity to analyze representation learning in the kernel and neural bandits \citep{zhu2021pure, mason2021nearly}. 
We can leverage the fact that this type of optimal design does not require the arm set to be an
ellipsoid \citep{du2023multi} which enables us to extend our analysis to non-linear representations.

\newpage
\bibliographystyle{apalike}
\bibliography{biblio}

\newpage
\appendix
\section{Appendix}
\subsection{Probability Tools and Previous Results}
\label{app:prob-tools}

In this section, we state useful lemmas we use in our proofs and previous results. 

\begin{lemma}
\label{lemma:kang-b2} (Generalized Stein's Lemma, \citep{stein2004use}) For a random variable $X$ with continuously differentiable density function $\bp: \R^d \rightarrow \R$, and any continuously differentiable function $f: \R^d \rightarrow \R$. Let $Q(\cdot)$ be a scoring function defined in \Cref{def:score-func}. If the expected values of both $\nabla f(X)$ and $f(X) \cdot Q(X)$ regarding the density $p$ exist, then they are identical, i.e.
\begin{align*}
\E[f(X) \cdot Q(X)]=\E[\nabla f(X)] .
\end{align*}
\end{lemma}

\begin{lemma}\textbf{\citep{minsker2018sub}}
\label{lemma:kang-b3} 
Define $\|\bA\|_{op}$ as the operator norm of $\bA$. 
Let $\bY_1, \ldots, \bY_n \in \R^{d_1 \times d_2}$ be a sequence of independent real random matrices, and assume that
\begin{align*}
\sigma_n^2 \geq \max \left(\left\|\sum_{j=1}^n \E\left(\bY_j \bY_j^{\top}\right)\right\|_{o p},\left\|\sum_{j=1}^n \E\left(\bY_j^{\top} \bY_j\right)\right\|_{o p}\right) .
\end{align*}
Then for any $t \in \R^{+}$and $\nu \in \R^{+}$, it holds that,
\begin{align*}
P\left(\left\|\sum_{j=1}^n \widetilde{\psi}_\nu\left(\bY_j\right)-\sum_{j=1}^n \E\left(\bY_j\right)\right\|_{o p} \geq t \sqrt{n}\right) \leq 2\left(d_1+d_2\right) \exp \left(\nu t \sqrt{n}+\frac{\nu^2 \sigma_n^2}{2}\right)
\end{align*}
\end{lemma}

\begin{lemma}
\label{theorem:kang-low-rank}
\textbf{(Restatement of Theorem 4.1 in \citet{kang2022efficient})} For any low-rank linear model with samples $\bX_1 \ldots, \bX_{n_1}$ drawn from $\X$ according to $\mathcal{D}$ then for the optimal solution to the nuclear norm regularization problem in \eqref{eq:convex-prog} with $\nu=\sqrt{2 \log \left(2\left(d_1+d_2\right) / \delta\right) /\left(\left(4 +S_0^2\right) M n_1 d_1 d_2\right)}$ and
\begin{align*}
\gamma_{n_1}=4 \sqrt{\frac{2\left(4 +S_0^2\right) C d_1 d_2 \log \left(2\left(d_1+d_2\right) / \delta\right)}{n_1}},
\end{align*}
with probability at least $1-\delta$ it holds that:
\begin{align*}
\left\|\wTheta- \mu^*\bTheta_*\right\|_F^2 \leq \frac{C_1 d_1 d_2 r \log \left(\frac{2\left(d_1+d_2\right)}{\delta}\right)}{n_1},
\end{align*}
for $C_1=36\left(4+S_0^2\right) C$, $\|\bX\|_F,\left\|\bTheta_*\right\|_F \leq S_0$, some nonzero constant $\mu^*$, and $\E\left[\left(S^{\bp}(\bX)\right)_{i j}^2\right] \leq C, \forall i, j$.
\end{lemma}

\subsection{$G$-optimal design on rotated arms}
\label{app:G-optimal-remark}

\begin{remark}
\label{remark:g-optimal}
\textbf{($G$-optimal design on rotated arms:)}
Using the concentration inequality in \Cref{prop:conc} we can show that for any arbitrary vector $\by\in\R^p$:
\begin{align*}
    |\by^{\top}(\wtheta_\ell-\btheta_*)| \leq\|\by\|_{\bV_\ell^{-1}} 2 \sqrt{14 \log (2 / \delta)}+\left\|\btheta_*\right\|_{\bLambda_\ell} \leq \|\by\|_{\bV_\ell^{-1}}  \sqrt{8 B^\ell_*\log (2 / \delta)}
\end{align*}
where the co-variance matrix $\bV_{\ell} \coloneqq \sum_{s=1}^{\tau^E_\ell}\uw_s \uw_s^\top +  \bLambda_\ell$ and $B^\ell_*$ is defined in \Cref{prop:conc}.
%
Now we want this to hold for all $\by \in \Y^*(\ucW_\ell)$, and so we need to union bound over $\W\supseteq\ucW_\ell$ replacing $\delta$ with $\delta /|\W|$.
Set the phase length $\tau^G_\ell \coloneqq \left\lceil \frac{64 B^\ell_*\rho^G(\Y(\W_\ell))\log (4\ell^2|\W| / \delta)}{\epsilon_\ell^2}\right\rceil$ where $\rho^G(\Y(\ucW_\ell))$ is defined in step $14$ of \Cref{alg:bandit-pure}.

Then for the  allocation $2\lfloor\bb_{\uw}^G \tau^G_\ell\rfloor$ for each $\bb_{\uw} \in \ucW_\ell$, we have for each $\uw \in \ucW_\ell \backslash \uw_*$ that with probability at least $1-\delta$,
\begin{align*}
    (\uw_*-\uw)^{\top}\wtheta_\ell &\geq (\uw_*-\uw)^{\top}\btheta_* - \|\uw_*-\uw\|_{(\sum_{\uw\in\W} \lceil2\tau^G_\ell\bb^*_{\uw}\rceil\uw\ \uw^\top + 2\tau^G_\ell\bLambda_\ell/\tau^G_\ell )^{-1}}  \sqrt{8 B^\ell_* \log (2 |\W| / \delta)}
\end{align*}
since for every $\uw_*-\uw\in\Y^*(\W)$ we have
\begin{align*}
    (\uw_*-\uw)^\top\big(2\sum_{\uw\in\W} \lceil \tau^G_\ell\bb^*_{\uw}\rceil\uw \ \uw^\top + 2\tfrac{\tau^G_\ell\bLambda_\ell}{\tau^G_\ell} \big)^{-1} (\uw_*-\uw)&\leq  \tfrac{1}{\tau^G_\ell}\|\uw_*-\uw\|^2_{(\sum_{\uw\in\ucW_\ell} \bb^*_{\uw}\uw\ \uw^\top + \tfrac{\bLambda_\ell}{\tau^G_\ell} )^{-1}} \\
    &\leq \tfrac{((\uw_*-\uw)^{\top}\btheta_*)^2}{\sqrt{8 B^\ell_* \log (2 |\W| / \delta)}}.
\end{align*}
The last inequality follows by plugging in the value of $\tau^G_\ell$ and $\rho^G(\Y(\ucW_\ell))$.
Hence to minimize the number of samples $\tau^G_\ell$ in phase $\ell$ we can re-arrange the above equation to show that
\begin{align*}
    \tau^G_\ell \geq \sqrt{8 B^\ell_* \log (2 |\W|/ \delta)} \max_{\uw\in\ucW_\ell\setminus\uw_*} \tfrac{\|\uw_*-\uw\|^2_{(\sum_{\uw\in\ucW_\ell} \bb^*_{\uw}\uw\ \uw^\top + \bLambda/n )^{-1}}}{(\uw_*-\uw)^{\top}\btheta_*}
\end{align*}
Hence, to minimize the sample complexity for the bilinear setting we need to sample according to
\begin{align}
    \bb^G_\ell = \argmin_{\bb}\max_{\uw} \tfrac{\|\uw_*-\uw\|^2_{(\sum_{\uw\in\ucW_\ell} \bb_{\uw}\uw\ \uw^\top + \bLambda/n )^{-1}}}{(\uw_*-\uw)^{\top}\btheta_*} \label{eq:bi-level-opt-1}
\end{align}
However, note that we do know the identity of $\bw_*$ or the gaps $(\uw_*-\uw)^{\top}\btheta_*$. So we replace the gaps with a lower bound of $\epsilon = 2^{-t}$ and compare against every pair of arms $\uw$ and $\uw'$ as follows:
\begin{align}
    \bb^G_\ell = \argmin_{\bb_{\uw}}\max_{\uw,\uw'\in\ucW_\ell} \|\uw-\uw'\|^2_{(\sum_{\uw\in\ucW} \bb_{\uw}\uw\ \uw^\top + \bLambda_\ell/n )^{-1}} \label{eq:bi-level-opt-single-1}
\end{align}
This is shown in step $12$ of \Cref{alg:bandit-pure}.
%
\end{remark}


\subsection{Application of Stein's Lemma}
\label{app:app-stein-lemma}

We also present the following two definitions from \citet{kang2022efficient} to facilitate analysis via Stein's method:

\begin{definition}\textbf{(Score Function)}
\label{def:score-func}
Let $\bp: \R \rightarrow \R$ be a univariate probability density function defined on $\R$. The score function $Q^{\bp}: \R \rightarrow \R$ regarding density $\bp(\cdot)$ is defined as:
\begin{align*}
Q^{\bp}(x)=-\nabla_x \log (\bp(x))=-\nabla_x \bp(x) / \bp(x), \quad x \in \R .
\end{align*}
\end{definition}
%
In particular, for a random matrix with its entrywise probability density $\mathbf{p}=\left(p_{i j}\right): \R^{d_1 \times d_2} \rightarrow$ $\R^{d_1 \times d_2}$, we define its score function $Q^{\mathrm{p}}=\left(Q_{i j}^{\mathbf{p}}\right): \R^{d_1 \times d_2} \rightarrow \R^{d_1 \times d_2}$ as $Q_{i j}^{\mathrm{p}}(x)=Q^{p_{i j}}(x)$ by applying the univariate score function to each entry of $\mathbf{p}$ independently.

\begin{assumption}
\label{assm:norm}
    The norm of true parameter $\bTheta^*$ and feature matrices in $\X$ is bounded: there exists $S \in \R^{+}$such that for all arms $\bX \in \X,\|\bX\|_F,\left\|\bTheta^*\right\|_F \leq S_0$
\end{assumption}

\begin{assumption}\textbf{(Finite second-moment score)}
\label{assm:distribution}
    There exists a sampling distribution $\mathcal{D}$ over $\X$ such that for the random matrix $\bX$ drawn from $\mathcal{D}$ with its associated density $\mathbf{p}: \R^{d_1 \times d_2} \rightarrow \R^{d_1 \times d_2}$, we have $\E\left[\left(Q^{\mathbf{p}}(\bX)\right)_{i j}^2\right] \leq C, \forall i, j$
\end{assumption}

\begin{definition}
\label{def:dilation-matrix}
Given a rectangular matrix $\bA \in \R^{d_1 \times d_2}$, the (Hermitian) dilation $\mathcal{H}: \R^{d_1 \times d_2} \rightarrow \R^{\left(d_1+d_2\right) \times\left(d_1+d_2\right)}$ is defined as:
\begin{align*}
\mathcal{H}(\bA)=\left(\begin{array}{cc}
0 & \bA \\
\bA^{\top} & 0
\end{array}\right)
\end{align*}
\end{definition}


\begin{definition}\textbf{(The function $\widetilde{\psi}_\nu$)}
\label{def:psi-tilde}
To explore the valid subspace of the parameter matrix $\bTheta_*$, we define a function $\psi: \R \rightarrow \R$ in \eqref{eq:psi-func}. Let $\mathcal{H}(\cdot)$ be as defined in \Cref{def:dilation-matrix}. Then define $\widetilde{\psi}_\nu: \R^{d_1 \times d_2} \rightarrow \R^{d_1 \times d_2}$ as $\widetilde{\psi}_\nu(\bA)=\psi(\nu \mathcal{H}(\bA))_{1: d_1,\left(d_1+1\right):\left(d_1+d_2\right)} / \nu$ for some parameter $\nu \in \R^{+}$
\begin{align}
\psi(x)= \begin{cases}\log \left(1+x+x^2 / 2\right), & x \geq 0 \label{eq:psi-func}\\ 
-\log \left(1-x+x^2 / 2\right), & x<0 \end{cases} 
\end{align}
\end{definition}

\subsection{Single-task Pure Exploration Proofs}
\label{app:single-bi}
\textbf{Good Event:} Define the good event $\F_\ell$ in phase $\ell$ that \gob\ has a good estimate of $\bTheta_*$ as follows:\
\begin{align}
\F_\ell = \left\|\wTheta_{\ell} - \mu^*\bTheta_*\right\|_F^2 \leq \frac{C_1 d_1 d_2 r \log \left(\frac{2\left(d_1+d_2\right)}{\delta_\ell}\right)}{\tau^E_\ell}, \label{eq:good-event-Theta}
\end{align}
where, $C_1=36\left(4+S_0^2\right) C$, $\|\bX\|_F,\left\|\bTheta_*\right\|_F \leq S_0$, some nonzero constant $\mu^*$, $\E\left[\left(S^{\bp}(\bX)\right)_{i j}^2\right] \leq C, \forall i, j$, 
and $\wTheta_{\ell}$ is the estimate from \eqref{eq:convex-prog}. Then define the good event 
\begin{align}
    \F \coloneqq \bigcap_{\ell=1}^\infty \F_\ell. \label{eq:good-event-Theta-single}
\end{align}

\begin{lemma}
\label{lemma:single-task-btheta} The event $\F$ holds with probability greater than $(1-\delta/2)$. 
\end{lemma}

\begin{proof}
From \Cref{theorem:kang-low-rank} we know the event $\F_\ell$ in \eqref{eq:good-event-Theta} holds with probability $(1-\delta_\ell)$.
%
%
Taking a union bound over all phases $\ell \geq 1$ and recalling $\delta_\ell:=\frac{\delta}{2 \ell^2}$, we obtain
\begin{align*}
\Pb(\F) & \geq 1-\sum_{\ell=1}^{\infty} \Pb\left(\F^c_\ell\right) \\
& \geq 1-\sum_{\ell=1}^{\infty} \frac{\delta_\ell}{2} \\
& =1-\sum_{\ell=1}^{\infty} \frac{\delta}{4 \ell^2} \\
& \geq 1-\frac{\delta}{2} .
\end{align*}
This concludes our proof.
\end{proof}

Now we move to the second stage for the rotated arm set $\uw\in\ucW$ and prove the following concentration event.

\begin{lemma}
\label{eq:conc-lemma-bilin-exp}
For any fixed $\uw \in \R^{p}$ and any $\delta>0$, we have that if $\beta\left(\btheta_*, \delta\right)=2 \sqrt{14 \log (2 / \delta)}+\left\|\btheta_*\right\|_{\bLambda}$, then at time $\tau_{\ell-1}+1$ (beginning of phase $\ell$):
\begin{align*}
\Pb\left(\left|\uw^{\top}\left(\wtheta_\ell-\btheta_*\right)\right| \leq\|\uw\|_{\bV_\ell^{-1}} \beta\left(\btheta_*, \delta\right)\right) \geq 1-\delta
\end{align*}
where, $\bV_{\ell} \coloneqq \sum_{s=\tau_{\ell-1}+1}^{\tau_\ell}\uw_s\uw_s^\top +  \bLambda_\ell$.
\end{lemma}


\begin{proof}
We follow the proof technique of Lemma 7 of \citet{valko2014spectral}.
Defining $\bxi_\ell=\sum_{s=\tau_{\ell-1}+1}^{\tau_{\ell}} \uw_s \eta_s$, we have:
\begin{align}
\left|\uw^{\top}\left(\wtheta_\ell-\btheta_*\right)\right| & \overset{(a}{=}\left|\uw^{\top}\left(-\bV_\ell^{-1} \bLambda \btheta_*+\bV_\ell^{-1} \bxi_\ell\right)\right| \nonumber\\
& \overset{(b)}{\leq}\left|\uw^{\top} \bV_\ell^{-1} \bLambda_\ell \btheta_*\right|+\left|\uw^{\top} \bV_\ell^{-1} \bxi_\ell\right| \label{eq:triangle-ineq}
\end{align}
where $(a)$ follows from Woodbury matrix identity and rearranging the terms, and $(b)$ follows from the triangle inequality. 

The first term in the right-hand side of \eqref{eq:triangle-ineq} is bounded as:
\begin{align*}
\left|\uw^{\top} \bV_\ell^{-1} \bLambda_\ell \btheta_*\right| & \leq\left\|\uw^\top \bV_\ell^{-1} \bLambda_\ell^{1 / 2}\right\|\left\|\bLambda_\ell^{1 / 2} \btheta_*\right\| \\
& \overset{(a)}{=}\left\|\btheta_*\right\|_{\bLambda_\ell} \sqrt{\uw^{\top} \bV_\ell^{-1} \bLambda_\ell \bV_\ell^{-1} \uw} \\
& \leq\left\|\btheta_*\right\|_{\bLambda_\ell} \sqrt{\uw^{\top} \bV_\ell^{-1} \uw}=\left\|\btheta_*\right\|_{\bLambda_\ell}\|\uw\|_{\bV_{\ell}^{-1}}
\end{align*}
where, $(a)$ follows as $\|\btheta_*\|_{\bLambda_\ell} = \sqrt{\btheta_*^\top\bLambda_\ell\btheta_*} = \|\bLambda_\ell^{1/2}\btheta_*\|$ and similarly for $\left\|\uw^\top \bV_\ell^{-1} \bLambda_\ell^{1 / 2}\right\|$. 
Now consider the second term in the r.h.s. of \eqref{eq:triangle-ineq}. We have:
\begin{align*}
\left|\uw^{\top} \bV_\ell^{-1} \bxi_\ell\right|=\left|\sum_{s=\tau_{\ell-1}+1}^{\tau_{\ell}}\left(\uw^{\top} \bV_\ell^{-1} \uw_s\right) \eta_s\right|.
\end{align*}
Now note that the arms $\left(\uw_s\right)$ selected by the algorithm during phase $\ell$ only depend on the proportion $\bb^G_*$ (the G-optimal design) and do not
depend on the rewards received during the phase $\ell-1$. Thus, given $\mathcal{F}_{j-2}$, the sequence $\left(\uw_s\right)_{\tau_{\ell-1}+1 \leq s<\tau_{\ell}}$ is deterministic. Consequently, one may use a variant of Azuma's inequality (\citet{shamir2011variant}) with a $1$-sub Gaussian assumption:
\begin{align*}
& \Pb\left(\left|\uw^{\top} \bV_\ell^{-1} \bxi_\ell\right|^2 \leq 28 \times 2 \log (2 / \delta) \times \uw^{\top} \bV_\ell^{-1}\left(\sum_{s=\tau_{\ell-1}+1}^{\tau_{\ell}} \uw_s \uw_s^{\top}\right) \bV_\ell^{-1} \uw \mid \mathcal{F}_{\ell-2}\right) \geq 1-\delta,
\end{align*}
from which we deduce:
\begin{align*}
& \Pb\left(\left|\uw^{\top} \bV_\ell^{-1} \bxi_\ell\right|^2 \leq 56  \uw^{\top} \bV_\ell^{-1} \uw \log (2 / \delta) \mid \mathcal{F}_{\ell-2}\right) \geq 1-\delta,
\end{align*}
since $\sum_{s=\tau_{\ell-1}+1}^{\tau_{\ell}} \uw_s \uw_s^{\top} \prec \bV_\ell$. Thus: 
\begin{align*}
& \qquad \Pb\left(\left|\uw^{\top} \bV_\ell^{-1} \bxi_\ell\right| \leq 2 \|\uw\|_{\bV_\ell^{-1}} \sqrt{14 \log (2 / \delta)}\right) \geq 1-\delta
\end{align*}
Combining everything we get that 
\begin{align*}
\Pb\left(\left|\uw^{\top}\left(\wtheta_\ell-\btheta_*\right)\right| \leq2 \sqrt{14 \log (2 / \delta)}+\left\|\btheta_*\right\|_{\bLambda_\ell}\right)\leq 1-\delta.
\end{align*}
\end{proof}

We need to change Lemma 6 of \citet{valko2014spectral}
in the following way so that the dependence on horizon $n$ is replaced by $\tau^G_{\ell-1}$. Note that $\tau^G_{\ell-1}$ is the phase length in the $\ell-1$-th phase and is determined before the start of phase $\tau^G_{\ell-1}$. 
Also, note that using the standard analysis of phase-based algorithms in \citet{fiez2019sequential, lattimore2020bandit} we do not re-use data between phases. 
First, we need the following support lemma from \citet{valko2014spectral}. 
\begin{lemma}\textbf{(Restatement of Lemma 4 from \citet{valko2014spectral})}
\label{lemma:valko-lemma-4}
Let $\bLambda_\ell=\operatorname{diag}\left(\lambda_1, \ldots, \lambda_\ell^\perp\right)$ be any diagonal matrix with strictly positive entries. Define $\mathbf{V}_\ell=\bLambda_\ell+\sum_{s=\tau^G_{\ell-1}+1}^{\tau^G_\ell} \uw_s \uw_s^{\top}$.
Then for any vectors $\left(\uw_s\right)_{\tau^G_{\ell-1}+1 \leq s \leq \tau^G_{\ell}}$, such that $\left\|\uw_s\right\|_2 \leq 1$ for all rounds $s$ such that  $\tau^G_{\ell-1}+1 \leq s \leq \tau^G_{\ell} $, we have that the determinant $\left|\mathbf{V}_\ell\right|$ is maximized when all $\uw_s$ are aligned with the axes.
\end{lemma}

\begin{lemma}
\label{lemma:valko-lemma-effective-dim-exp}
Let $k$ be the effective dimension. Then
\begin{align*}
\log \frac{\left|\mathbf{V}_\ell\right|}{|\bLambda_\ell|} \leq 8 k \log \left(1+\frac{\tau^G_{\ell-1}}{\lambda}\right)
\end{align*}
when $\lambda^{\perp}_\ell=\frac{\tau^G_{\ell-1}}{k \log \left(1+\tau^G_{\ell-1} / \lambda\right)}$. 
\end{lemma}

\begin{proof}
We want to bound the determinant $\left|\mathbf{V}_\ell\right|$ under the coordinate constraints $\left\|\uw_t\right\|_2 \leq 1$. 
%
%
Let:
\begin{align*}
M\left(\uw_1, \ldots, \uw_t\right)=\left|\bLambda_\ell + \sum_{s=\tau^G_{\ell-1}+1}^{\tau^G_\ell} \uw_s \uw_s^{\top}\right|
\end{align*}
From \Cref{lemma:valko-lemma-4} we deduce that the maximum of $M$ is reached when all $\uw_s$ are aligned with the axes. Let the number of samples of these axes-aligned $\uw_s$'s during the $\ell$-th phase be denoted as  $t^\ell_1,t^\ell_2,\ldots,t^\ell_p$ such that $\sum_{i=1}^p t^\ell_i = \tau^G_\ell$.  
Then we can show that
\begin{align*}
M &\overset{(a)}{=}\max _{\uw_1, \ldots, \uw_t ; \uw_s \in\left\{\mathrm{e}_1, \ldots, \mathbf{e}_p\right\}}\left|\bLambda_\ell +\sum_{s=\tau^G_{\ell-1}+1}^{\tau^G_\ell} \uw_s \uw_s^{\top}\right| \\
& \overset{(b)}{=}\max_{t^\ell_1, \ldots, t^\ell_p,\text { positive integers, } \sum_{i=1}^p t^\ell_i=\tau^G_\ell}\left|\operatorname{diag}\left(\lambda_i+t^\ell_i\right)\right| \\
& \overset{(c)}{\leq} \quad \max_{t^\ell_1, \ldots, t^\ell_p,\text { positive integers, } \sum_{i=1}^p t^\ell_i=\tau^G_\ell} \prod_{i=1}^p\left(\lambda_i+t^\ell_i\right)
\end{align*}
where, $(a)$ follows from \Cref{lemma:valko-lemma-4}, $(b)$ follows as the $\operatorname{diag}\left(\lambda_i+t^\ell_i\right)$ contains the number of times axis-aligned $\uw_s \in\left\{\mathrm{e}_1, \ldots, \mathbf{e}_p\right\}$ are observed, and $(c)$ follows as the determinant of a diagonal matrix is the product of the diagonal elements.
Now we can show that
\begin{align*}
\log \frac{\left|\mathbf{V}_\ell\right|}{|\bLambda_\ell|} & \leq \sum_{i=1}^k \log \left(1+\frac{t^\ell_i}{\lambda}\right)+\sum_{i=k+1}^p \log \left(1+\frac{t^\ell_i}{\lambda_i}\right) \\
& \overset{(a)}{\leq} k \log \left(1+\frac{t^{\ell}_i}{\lambda}\right)+\sum_{i=1}^p \frac{t^{\ell-1}_i}{\lambda_\ell^{\perp}} \\
& \overset{(b)}{\leq} k \log \left(1+\frac{t^{\ell}_i}{\lambda}\right)+\frac{\tau^G_{\ell-1}}{\lambda_{\ell}^{\perp}} \\
& \overset{(c)}{\leq} k \log \left(1+\frac{t^{\ell}_i}{\lambda}\right)+k \log \left(1+\dfrac{\tau^G_{\ell-1}}{ \lambda}\right) \\
& \overset{(d)}{\leq} 8 k \log \left(1+\frac{\tau^G_{\ell-1}}{\lambda}\right) 
\end{align*}
where, $(a)$ follows as $\log(1+t^\ell_i/\lambda_i) \leq t^{\ell-1}_i/\lambda_\ell^{\perp}$, $(b)$ follows as
$\sum_{i=1}^p t^{\ell-1}_i = \tau^G_{\ell-1}$, and $(c)$ follows for $\lambda_\ell^{\perp} = \frac{\tau^G_{\ell-1}}{k \log \left(1+\tau^G_{\ell-1} / \lambda\right)}$ and $(d)$ follows from \Cref{lemma:per-level}.
\end{proof}

\begin{lemma}
\label{lemma:per-level}
    Let $\rho^G(\Y(\ucW_{\ell})) = \min_{\bb_{\uw}}\max_{\uw,\uw'\in\ucW_\ell} \|\uw-\uw'\|^2_{(\sum_{\uw\in\ucW} \bb_{\uw}\uw\ \uw^\top + \bLambda_\ell/n )^{-1}}$. 
    Recall that $\tau^G_{\ell} = \frac{8 B^\ell_*\rho^G(\Y(\W_\ell))\log (4\ell^2|\W| / \delta)}{\epsilon_\ell^2}$. 
    Assume $\log(p)\leq k$. 
    Then we can show that
    \begin{align*}
        \log\left(1 + \dfrac{\tau^G_\ell}{\lambda}\right) \leq 8k\log\left(1 + \dfrac{\tau^G_{\ell-1}}{\lambda}\right).
    \end{align*}
\end{lemma}

\begin{proof}
    We start by first recalling the definition of $$\tau^G_{\ell}=2 \epsilon_{\ell}^{-2} 8k \log(1 + \frac{\tau^G_{\ell-1}}{\lambda}) \rho^G(\Y(\ucW_{\ell})) \log \left(4 \ell^2|\ucW| / \delta_\ell\right).$$ 
    Then we can show the following
    \begin{align*}
        \dfrac{\tau^G_{\ell}}{\tau^G_{\ell-1}}  = \dfrac{2 \epsilon_{\ell}^{-2} B^\ell_* \rho^G(\Y(\ucW_{\ell})) \log \left(4 \ell^2|\W| / \delta_\ell\right)}{2 \epsilon_{\ell-1}^{-2} B^{\ell-1}_* \rho^G(\Y(\ucW_{\ell-1})) \log \left(4 (\ell-1)^2|\W| / \delta_\ell\right)} & \leq \dfrac{4 \rho^G(\Y(\ucW_{\ell}))\left(64({\lambda S^2 +\lambda_\ell^{\perp} S^{(2),\ell}_{\perp}})\right) }{\rho^G(\Y(\ucW_{\ell-1}))\left(64({\lambda S^2 +\lambda_{\ell-1}^{\perp} S^{(2),\ell-1}_{\perp}})\right)}\\
        &\leq \dfrac{4 \rho^G(\Y(\ucW_{\ell}))\tau^G_{\ell-1}/\log(1 + \frac{\tau^G_{\ell-1}}{\lambda}) }{\rho^G(\Y(\ucW_{\ell-1}))\tau^G_{\ell-2}/\log(1 + \frac{\tau^G_{\ell-2}}{\lambda})}\\
        %
        &\overset{(a)}{\leq} \dfrac{4 \rho^G(\Y(\ucW_{\ell}))\tau^G_{\ell-1}\log(1 + \frac{\tau^G_{\ell-1}}{\lambda}) }{\rho^G(\Y(\ucW_{\ell-1}))}\\ 
        &\overset{(b)}{\leq} \dfrac{4p}{\gamma^2_{\Y}} \dfrac{\max_{\uw \in \ucW}\|\uw\|_2}{\max_{\by \in \Y(\ucW_\ell)}\|\by\|_2^2}\tau^G_{\ell-1}\log(1 + \frac{\tau^G_{\ell-1}}{\lambda}) = \dfrac{4p}{C\gamma^2_{\Y}}
    \end{align*}
    where, $(a)$ follows as $\log(1 + \frac{\tau^G_{\ell-2}}{\lambda}) \geq 1$ and $\log(1 + \frac{\tau^G_{\ell-1}}{\lambda}) \geq \log(1 + \frac{\tau^G_{\ell-2}}{\lambda})$.
    The $(b)$ follows using Lemma 1 from \citet{fiez2019sequential} such that
    $$
    \max_{\by \in \Y(\ucW_\ell)}\|\by\|_2^2 /\left(\max_{\uw \in \ucW}\|\uw\|_2\right) \leq \rho^G(\Y(\ucW_{\ell})) \leq p / \gamma_{\Y}^2 \overset{(a_1)}{\implies} 1 \leq \rho^G(\Y(\ucW_{\ell})) \leq p / \gamma_{\Y}^2  .
    $$
    where, $(a_1)$ follows as $\|\bx\|\leq 1$, $\|\bz\|\leq 1$.
    This implies that for a constant $C > 0$
    \begin{align*}
        \tau^G_\ell \leq \dfrac{4p}{C\gamma^2_{\Y}} (\tau^G_{\ell-1})^2\log(1 + \frac{\tau^G_{\ell-1}}{\lambda}) 
        &\implies \log\left(1 + \dfrac{\tau^G_\ell}{\lambda}\right) \leq \log\left(1 + \dfrac{4p}{C\gamma^2_{\Y}}(\tau^G_{\ell-1})^2\log(1 + \frac{\tau^G_{\ell-1}}{\lambda})\right)\\
        & \overset{(a)}{\implies} \log\left(1 + \dfrac{\tau^G_\ell}{\lambda}\right) \leq 4k\log\left(1 + \dfrac{\tau^G_{\ell-1}}{C\gamma^2_{\Y} \lambda}\right)\\
        & \overset{(b)}{\implies} \log\left(1 + \dfrac{\tau^G_\ell}{\lambda}\right) \leq 8k\log\left(1 + \dfrac{\tau^G_{\ell-1}}{\lambda}\right)
    \end{align*}
    where, in $(a)$ follows for $\log(p) \leq k$, $\log(a^2\log(a)) \leq 4\log(a)$.
    %
    The $(b)$ follows as $4k\log\left(1 + \dfrac{\tau^G_{\ell-1}}{C\gamma^2_{\Y} \lambda}\right) \leq 8k\log\left(1 + \dfrac{\tau^G_{\ell-1}}{\lambda}\right)$.
\end{proof}

\begin{lemma}
\label{lemma:equivalence}
The $G$-optimal design in \eqref{eq:bi-level-opt-single} is equivalent to solving the $D$-optimal design\
\begin{align*}
    \bb^D_\ell = \argmax_{\bb} \log \frac{\left|\sum_{\uw\in\ucW_\ell}\bb_{\uw}\uw\ \uw^\top + \bLambda_\ell\right|}{|\bLambda_\ell|}.
\end{align*}
Furthermore, the support of $|\bb^D_\ell|\leq \frac{8k\log(1+\frac{\tau^G_{\ell-1}}{\lambda})(8k\log(1+\frac{\tau^G_{\ell-1}}{\lambda}) + 1)}{2}$, where $k=(d_1 + d_2)r$.
\end{lemma}

\begin{proof}
To prove the equivalence between $\bb^G_*$ and $\bb^D_*$ we need to first show that the regularization matrix $\bLambda_\ell$ does not depend on $\uw$ or $\by = \uw-\uw'$, where $\uw\in\ucW_\ell$.
%
%
Define $\operatorname{conv}(\mathcal{X} \cup-\mathcal{X})$ as the convex hull of $\mathcal{X} \cup-\mathcal{X}$. Now 
recall we have from Lemma 1 of \citet{fiez2019sequential} that
\begin{align}
    1 \leq \rho^G(\Y(\ucW_{\ell})) \leq p / \gamma_{\Y}^2 \label{eq:rho-app}
\end{align}
where, $\gamma_{\Y}=\max \{c>0: c \Y \subset \operatorname{conv}(\ucW_{\ell} \cup-\ucW_{\ell})\}$ as the gauge norm of $\Y$ \citep{rockafellar2015convex}. We can consider the gauge norm $\gamma_{\Y}$ as a problem-dependent constant. 
%
%
Now recall that
\begin{align*}
    \lambda^\ell_{\perp} = \frac{\tau^G_{\ell-1}}{8k\log(1+\frac{\tau^G_{\ell-1}}{\lambda})} 
    &\leq 2\tau^G_{\ell-1} \leq
    \frac{64 B^{\ell-1}_*\rho^G(\Y(\W_{\ell-1}))\log (4(\ell-1)^2|\W| / \delta_\ell)}{\epsilon_{\ell-1}^2} \\
    &\overset{(a)}{\leq} \frac{64(B^{\ell-1}_*)^2p\log (4(\ell-1)^2|\W| / \delta_\ell)}{\gamma_{\Y}^2 \epsilon_{\ell-1}^2} \\
    &\overset{(b)}{\leq} \frac{\left(256({\lambda S^2 + 8 p^2 r})\right)\log (4(\ell-1)^2p|\W| / \delta_\ell)}{S_r\gamma_{\Y}^2 \epsilon_{\ell-1}^2} 
\end{align*}
where, $(a)$ follows from \eqref{eq:rho-app} and noting that $B^{\ell-1}_* \leq (B^{\ell-1}_*)^2$. The $(b)$ follows as $S_\ell^{\perp} \coloneqq  \frac{8p r}{\tau^E_{\ell} S^2_{r}} \log \left(\frac{d_1+d_2}{\delta_\ell}\right)$, $p=d_1d_2$ and substituting this value and $\lambda^\ell_{\perp}$ in $B^\ell_*$ we get
%
%
\begin{align*}
B^\ell_* \leq (B^\ell_*)^2 \leq\left(256({\lambda S^2 +\lambda^\ell_{\perp} S^{(2),\ell}_{\perp}})\right) &\leq \left(256\left({\lambda S^2 + \dfrac{8 p r}{S^2_r} \dfrac{\tau^G_{\ell-1}}{\tau^E_\ell}\cdot \log (4(\ell-1)^2|\W| / \delta_\ell)\log(d_1+d_2/\delta_\ell)}\right)\right) \\
&\leq  \left(256\left({\lambda S^2 + \dfrac{8 p r}{S^2_r} \rho^G_{\ell-1}(\ucW) \log(p|\W|/\delta_\ell)}\right)\right) \overset{(a)}{\leq} \left(512\left({\lambda S^2 + 8 p^2 r}\right)/S_r\right). 
\end{align*}
Here $S_r$ is the r-th larget eigenvalue of matrix $\bTheta_*$. 
Substituting this value of 
$\lambda_\ell^\perp$ we can show that $\bLambda_\ell$ does not depend on $\uw$ or $\by = \uw-\uw'$. The rest of the proof to show equivalence follows the same way as in Theorem 21.1 in \citet{lattimore2020bandit}.

To bound the support of $\bb^D_*$ we proceed as follows:
Define the set $\Y(\ucW_\ell)$ as the set of all arms containing $\by = \uw-\uw' \in\R^p$. 
Then we can use Lemma 7 of \citet{soare2014best} to show that the solution to 
$$
\max _{\by\in\Y(\ucW_\ell)}\left\|\by\right\|_{\left(\sum_{\uw\in\ucW_\ell} \bb_{\uw}^{} \uw\ \uw^{\top} +\bLambda_\ell \right)^{-1}}^2 = \max_{\uw,\uw'}\left\|\uw - \uw'\right\|_{\left(\sum_{\uw\in\ucW_\ell} \bb_{\uw}^{} \uw \uw^{\top} +\bLambda_\ell \right)^{-1}}^2
$$
has a support of atmost $(k_1+1)k_1/2$ where $k_1 = 8k\log(1+\tau^G_{\ell-1}/\lambda)$. The proof follows from the fact that for any pair $\left(\uw, \uw^{\prime}\right)$ we can show that
$$
\left\|\uw-\uw^{\prime}\right\|_{\left(\sum_{\uw\in\ucW_\ell} \bb_{\uw}^{} \uw \uw^{\top} +\bLambda \right)^{-1}} \leq 2 \max _{\uw^{\prime \prime} \in \ucW_\ell}\left\|\uw^{\prime \prime}\right\|_{\left(\sum_{\uw\in\ucW} \bb_{\uw}^{} \uw\ \uw^{\top} +\bLambda_\ell \right)^{-1}} .
$$
Then following the work of  \citet{jamieson2022interactive} Frank-Wolfe algorithm (in section 2.3.1) with the 
$$
g(\bb) = \log \frac{\left|\sum_{\uw\in\ucW_\ell}\bb_{\uw}\uw\ \uw^\top + \bLambda_\ell\right|}{|\bLambda_\ell|} = \log \left|\sum_{\uw\in\ucW_\ell}\bb_{\uw}\uw\ \uw^\top + \bLambda_\ell\right|- \log|\bLambda_\ell|,
$$ 
and setting for the $j$-th iteration of the Frank-Wolfe the 
$$
I_j=\argmax _{\by\in\Y(\ucW_\ell)}\left\|\by\right\|_{\left(\sum_{\uw\in\ucW_\ell} \bb^j_{\uw} \uw\ \uw^{\top} +\bLambda_\ell \right)^{-1}}^2,
$$
and stopping condition 
\begin{align}
    \max _{\by\in\Y(\ucW_\ell)}\left\|\by\right\|_{\left(\sum_{\uw\in\ucW_\ell} \bb^j_{\uw} \uw\ \uw^{\top} +\bLambda_\ell \right)^{-1}}^2 \leq 8 k\log(1+\frac{\tau^G_{\ell-1}}{\lambda}) \label{eq:stopping-condition}
\end{align}
%
%

This can be done because note that for any $\bb \in \triangle_{\ucW_\ell}$ we have by Kiefer-Wolfowitz Theorem \citep{kiefer1960equivalence} that $[\nabla g(\bb)]_{\by}=\left\|\by\right\|_{\left(\sum_{\bx \in \X} \bb_x \bx \bx^{\top} + \bLambda\right)^{-1}}^2 \geq 8k\log(1 + \frac{\tau^G_{\ell-1}}{\lambda})$.
This is because $\bLambda_\ell$ does not depend on $\uw$ or $\by$ by the same logic as discussed before. 
The rest of the proof follows by the same way as in section 2.3.1 in \citet{jamieson2022interactive}. 
This will result in a support size of $\bb$ at most $8k\log(1+\frac{\tau^G_{\ell-1}}{\lambda})(8k\log(1+\frac{\tau^G_{\ell-1}}{\lambda}) + 1)/2$ following Lemma 7 of \citet{soare2014best}.
Hence, it follows that solving the \cref{eq:bi-level-opt-single} 
will result in a support of $|\bb^D_\ell|\leq \frac{8k\log(1+\frac{\tau^G_{\ell-1}}{\lambda})(8k\log(1+\frac{\tau^G_{\ell-1}}{\lambda}) + 1)}{2}$.
\end{proof}

\begin{proposition}
\label{prop:conc}
If $\bb^G_\ell$ is the $G$-optimal design for $\ucW_\ell$ then if we pull arm $\uw \in \ucW_\ell$ exactly $\left\lceil\tau^G \bb^G_{\ell}\right\rceil$ times for some $\tau^G_\ell>0$ and compute the least squares estimator $\wtheta_\ell$. Then for each $\uw \in \ucW_\ell$ we have with probability at least $1-\delta$
\begin{align*}
    \Pb\bigg(\bigcup_{\uw \in \ucW_\ell}\bigg\{\left|\left\langle \uw, \wtheta_\ell-\btheta^*\right\rangle\right|&\leq \sqrt{\dfrac{64 B^\ell_* k\log(1+\frac{\tau^G_{\ell-1}}{\lambda}) \log \left(2|\W| / \delta_\ell\right)}{\tau^G_\ell}} \bigg\}\bigg) \geq 1-\delta_\ell.
\end{align*}
where, $\left\|\btheta^*\right\|_{\Lambda} \leq \sqrt{\lambda\left\|\btheta_{1: k}\right\|_2^2+\lambda_\ell^{\perp}\left\|\btheta_{k+1: p}\right\|_2^2} \leq \sqrt{\lambda} S+\sqrt{\lambda_\ell^{\perp}} S^\ell_{\perp}$, $B^\ell_*\!\!=\!\! \sqrt{\lambda} S + \sqrt{\lambda^{\perp}_\ell} S^{\perp}_\ell$.
\end{proposition}

\begin{proof}
From Woodbury Matrix Identity we know that for any arbitrary matrix $\bA$ and $\bB$, we have the following identity $(\bA+\bB)^{-1}=\bA^{-1}-\left(\bA+\bA \bB^{-1} \bA\right)^{-1}$. It follows then that 
\begin{align*}
    \uw^\top(\bA+\bB)^{-1}\uw = \uw^\top\left(\bA^{-1}-\left(\bA+\bA \bB^{-1} \bA\right)^{-1}\right)\uw \leq \uw^\top\bA^{-1}\uw = \|\uw\|_{\bA^{-1}}. 
\end{align*}
Hence we can show that,
\begin{align}
    &\|\uw\|_{\left(\sum_{\uw \in \ucW_\ell}\lceil\tau^G_\ell \bb^G_{\ell}\rceil \uw\ \uw^{\top} + \bLambda_\ell\right)^{-1}} 
    \leq \uw^\top\left(\sum_{\uw \in \ucW_\ell}\lceil\tau^G_\ell \bb^G_{\ell}\rceil \uw\ \uw^{\top}\right)^{-1}\uw. \label{eq:matrix-inv}
\end{align}
From \Cref{lemma:equivalence} we know the support of $\bb^G_\ell$ is less than 
$$
\frac{8k\log(1+\frac{\tau^G_{\ell-1}}{\lambda})(8k\log(1+\frac{\tau^G_{\ell-1}}{\lambda})+1)}{2} \leq (8k\log(1+\frac{\tau^G_{\ell-1}}{\lambda}))^2.
$$
Also note that $\left\|\btheta^*\right\|_{\bLambda_\ell} \leq \sqrt{\lambda\left\|\btheta_{1: k}\right\|_2^2+\lambda^\ell_{\perp}\left\|\btheta_{k+1: p}\right\|_2^2} \leq \sqrt{\lambda} S+\sqrt{\lambda^\ell_{\perp}} S^\ell_{\perp}$.
Then we can show that
\begin{align*}
\left\langle \uw, \wtheta_\ell -\btheta^*\right\rangle & \leq\|\uw\|_{\left(\sum_{\uw \in \ucW_\ell}\lceil\tau^G_\ell \bb^G_{\ell}\rceil \uw\ \uw^{\top} + \bLambda\right)^{-1}} \left(2\sqrt{14 \log (2 / \delta_\ell)} + \sqrt{\lambda} S+\sqrt{\lambda^\ell_{\perp}} S^\ell_{\perp} \right) \\
& \leq \frac{1}{\sqrt{\tau^G_\ell}}\|\uw\|_{\left(\sum_{\uw \in \ucW_\ell} \bb_{\ell^*} \uw\ \uw^{\top} + \bLambda\right)^{-1}} \left(2\sqrt{14 \log (2 / \delta_\ell)} + \sqrt{\lambda} S+\sqrt{\lambda^\ell_{\perp}} S^\ell_{\perp}\right)\\
& \overset{(a)}{\leq} \sqrt{\frac{56 \times 8k\log(1+\frac{\tau^G_{\ell-1}}{\lambda}) \log (2 / \delta_\ell)}{\tau^G_\ell}} + \sqrt{\dfrac{28k\log(1+\frac{\tau^G_{\ell-1}}{\lambda})\log(2/\delta_\ell)}{\tau^G_\ell}}(\sqrt{\lambda} S + \sqrt{\lambda^\ell_{\perp}} S^{\ell}_{\perp})\\
&= \sqrt{\frac{8k\log(1+\frac{\tau^G_{\ell-1}}{\lambda})\log (2 / \delta_\ell)}{\tau^G_\ell}}\left(\sqrt{56  } + \underbrace{\sqrt{\lambda} S + \sqrt{\lambda^\ell_{\perp}} S^{\ell}_{\perp}}_{B^\ell_*}\right)\\
&\leq \sqrt{\frac{64 B^\ell_* k\log(1+\frac{\tau^G_{\ell-1}}{\lambda})\log (2 / \delta_\ell)}{\tau^G_\ell}}
\end{align*}
where, $(a)$ follows as $\|\uw\|_{\left(\sum_{\uw \in \ucW_\ell} \bb^G_{\ell,\uw} \uw\ \uw^{\top} + \bLambda_\ell\right)^{-1}} \leq (8k\log(1+\tau^G_{\ell-1}/\lambda))^2$.
%
%
Thus we have taken at most $\tau^G_\ell+\frac{8k\log(1+\frac{\tau^G_{\ell-1}}{\lambda})(8k\log(1+\frac{\tau^G_{\ell-1}}{\lambda})+1)}{2}$ pulls. Thus, for any $\delta \in(0,1)$ we have 
\begin{align*}
\Pb\bigg(\bigcup_{\uw \in \ucW_\ell}\bigg\{\left|\left\langle \uw, \wtheta_\ell-\btheta^*\right\rangle\right|&\geq \sqrt{\dfrac{ 64k B^\ell_*\log(1+\frac{\tau^G_{\ell-1}}{\lambda}) \log \left(2|\W| / \delta_\ell\right)}{\tau^G_\ell}} \bigg\}\bigg) \leq \delta_\ell.
\end{align*}
The claim of the lemma follows.
\end{proof}

\begin{discussion}\textbf{(Phase Length)}
\label{dis:phase-length}
It follows from \Cref{prop:conc} that if the gaps are known, one can set the phase length as 
\begin{align*}
     \tau^G_\ell = \dfrac{64 B^\ell_* \rho(\Y(\ucW_\ell))  \log \left(2|\W| / \delta\right)}{(\uw^\top(\wtheta_\ell-\btheta^*))^2} 
\end{align*}
since $8k\log(1+\frac{\tau^G_{\ell-1}}{\lambda})\leq \rho(\Y(\ucW_\ell))\coloneqq  8k\log(1+\frac{\tau^G_{\ell}}{\lambda})$ to guarantee that the event $\bigcup_{\uw \in \ucW_\ell}\bigg\{\left|\left\langle \uw, \wtheta_\ell-\btheta^*\right\rangle\right|$ holds with probability greater than $1-\delta$.

However, since in practice, the gaps are not known, for an agnostic algorithm that does not know the gaps, one can set a proxy for the gap as $\epsilon_\ell$ (for some $\epsilon_\ell > 0$) and get the phase length as follows:
\begin{align*}
     \tau^G_\ell = \dfrac{64 B^\ell_* \rho(\Y(\ucW_\ell)) \log \left(2|\W| / \delta\right)}{\epsilon_\ell^2}.
\end{align*}
This gives us the desired phase length so that the event $\bigcup_{\uw \in \ucW_\ell}\bigg\{\left|\left\langle \uw, \wtheta_\ell-\btheta^*\right\rangle\right|$ holds with probability greater than $1-\delta$.
%
\end{discussion}

\begin{lemma}
\label{lemma:elim}
Assume that $\max_{\uw \in \ucW}\left\langle \uw_*-\uw, \btheta^*\right\rangle \leq 2$. With probability at least $1-\delta$, we have $\uw_* \in \ucW_{\ell}$ and $\max _{\uw \in \ucW_{\ell}}\left\langle \uw_*- \uw, \btheta^*\right\rangle \leq 4 \epsilon_{\ell}$ for all $\ell \in \mathbb{N}$.
\end{lemma}

\begin{proof}
For any $\V \subseteq \ucW_\ell$ be the active set and $\uw \in \V$ define
\begin{align}
\mathcal{E}_{\uw, \ell}(\V)=\left\{\left|\left\langle \uw-\uw_*, \wtheta_{\ell}(\V)-\btheta^*\right\rangle\right| \leq \epsilon_{\ell}\right\} \label{eq:event-arm-elim}
\end{align}
where it is implicit that $\wtheta_{\ell}:=\wtheta_{\ell}(\V)$ is the design constructed in the algorithm at stage $\ell$ with respect to $\ucW_{\ell}=\V$. Also note that $\delta_\ell = \frac{\delta}{4\ell^2}$. Given $\ucW_{\ell}$, with probability at least $1-2\cdot\frac{\delta}{4 \ell^2|\ucW|}$

\begin{align*}
\left|\left\langle \uw-\uw_*, \wtheta_{\ell}-\btheta^*\right\rangle\right| &\overset{(a)}{\leq}
\sqrt{64 B^\ell_* k\log(1+\frac{\tau^G_{\ell-1}}{\lambda}) \log \left(4 \ell^2|\ucW| / \delta\right)} \\
%
& \leq \frac{(k\log(1+\frac{\tau^G_{\ell-1}}{\lambda}))^2}{\sqrt{\tau^G_{\ell}}}  \sqrt{64 B^\ell_* \log \left(4 \ell^2|\ucW| / \delta\right)} \\
& \overset{(b)}{\leq} \sqrt{\frac{\left\|\uw-\uw_*\right\|_{\left(\sum_{\uw \in \V} \bb^G_{\ell, \uw}(\V) \uw\ \uw^{\top}+\bLambda_\ell\right)^{-1}}^2}{64 B^\ell_* \epsilon_{\ell}^{-2} \rho(\Y(\ucW_\ell)) \log \left(4 \ell^2|\ucW| / \delta\right)}} \sqrt{64 B^\ell_* \log \left(4 \ell^2|\ucW| / \delta\right)} \\
& \leq \sqrt{\frac{\left\|\uw-\uw_*\right\|_{\left(\sum_{\uw \in \V} \bb^G_{\ell, \uw}(\V) \uw\ \uw^{\top}+\bLambda_\ell\right)^{-1}}^2}{\epsilon_{\ell}^{-2} \rho(\Y(\ucW_\ell)) \log \left(4 \ell^2|\ucW| / \delta\right)}} \sqrt{ \log \left(4 \ell^2|\ucW| / \delta\right)} \\
& \overset{(c)}{=}\epsilon_{\ell}
\end{align*}
where, $(a)$ follows \Cref{prop:conc}. The $(b)$ follows as 
$$
(k\log(1+\frac{\tau^G_{\ell-1}}{\lambda}))^2 \leq (k\log(1+\frac{\tau^G_{\ell}}{\lambda}))^2 \coloneqq  \rho(\Y(\ucW_\ell)) = \left\|\uw-\uw_*\right\|_{\left(\sum_{\uw \in \V} \bb^G_{\ell, \uw}(\V) \uw\ \uw^{\top}+\bLambda_\ell\right)^{-1}}^2. 
$$
The $(c)$ follows as $\rho(\Y(\ucW_\ell)) = \left\|\uw-\uw_*\right\|_{\left(\sum_{\uw \in \V} \bb^G_{\ell, \uw}(\V) \uw\ \uw^{\top}+\bLambda_\ell\right)^{-1}}^2$.
%
%
%
By exactly the same sequence of steps as above, we have 
$$
\Pb\left(\bigcap_{\ell=1}^{\infty} \bigcap_{\uw \in \ucW_{\ell}}\left\{\left|\left\langle \uw-\uw_*, \wtheta_t-\btheta^*\right\rangle\right|>\epsilon_t\right\}\right)=\Pb\left(\bigcap_{\uw \in \ucW_\ell} \bigcap_{\ell=1}^{\infty} \mathcal{E}_{\uw, \ell}\left(\ucW_{\ell}\right)\right) \geq 1-\delta,
$$
so assume these events hold. Consequently, for any $\uw^{\prime} \in \ucW_{\ell}$
\begin{align*}
\left\langle \uw^{\prime}-\uw_*, \wtheta_{\ell}\right\rangle & =\left\langle \uw^{\prime}-\uw_*, \wtheta_{\ell}-\btheta^*\right\rangle+\left\langle \uw^{\prime}-\uw_*, \btheta^*\right\rangle \\
& \leq\left\langle \uw^{\prime}-\uw_*, \wtheta_{\ell}-\btheta^*\right\rangle \\
& \leq \epsilon_{\ell}
\end{align*}
so that $\uw_*$ would survive to round $\ell+1$. And for any $\uw \in \ucW_{\ell}$ such that $\left\langle \uw_*-\uw, \btheta^*\right\rangle>2 \epsilon_{\ell}$ we have
\begin{align*}
\max _{\uw^{\prime} \in \ucW_{\ell}}\left\langle \uw^{\prime}-\uw, \wtheta_{\ell}\right\rangle & \geq\left\langle \uw_*-\uw, \wtheta_{\ell}\right\rangle \\
& =\left\langle \uw_*-\uw, \wtheta_{\ell}-\btheta^*\right\rangle+\left\langle \uw_*-\uw, \btheta^*\right\rangle \\
& >-\epsilon_{\ell}+2 \epsilon_{\ell} \\
& =\epsilon_{\ell}
\end{align*}
which implies this $\uw$ would be eliminated. Note that this implies that $\max _{\uw \in \ucW_{\ell+1}}\left\langle \uw_*-\uw, \btheta^*\right\rangle \leq$ $2 \epsilon_{\ell}=4 \epsilon_{\ell+1}$. Hence, the claim of the lemma follows.
\end{proof}

\subsection{Final Sample Complexity Bound for Single Task Setting}

\begin{customtheorem}{1}\textbf{(Restatement)}
\label{app:theorem-1}
    With probability at least $1 - \delta$, \gob\ returns the best arms $\bx_*$, $\bz_*$, and the number of samples used is bounded by
\begin{align*}
    \widetilde{O}\left(\dfrac{(d_1+d_2)r}{\Delta^2} + \dfrac{\sqrt{d_1 d_2 r}}{S_r}\right)
\end{align*}
where, $\Delta = \min_{\bx\in\X\setminus\{\bx_*\},\bz\in\Z\setminus\{\bz_*\} }(\bx_*^\top\bTheta_*\bz_* - \bx^\top\bTheta_*\bz)$, and $S_r$ is the $r$-th largest singular value of $\bTheta_*$.
\end{customtheorem}

\begin{proof}
For the rest of the proof we have that the good events $\F_\ell \bigcap \mathcal{E}_{\uw, \ell}(\ucW_\ell)$ holds true for each phase $\ell$ with probability greater than $(1-\delta)$. The two events are defined in \eqref{eq:good-event-Theta} and \eqref{eq:event-arm-elim}.

\textbf{Second Stage:} Define $\underline{\A}_{\ell}=\left\{\uw \in \ucW_\ell:\left\langle \uw_*-\uw, \btheta^*\right\rangle \leq 4 \epsilon_{\ell}\right\}$. Note that by assumption $\ucW=\ucW_1=\underline{S}_1$. The above lemma implies that with probability at least $1-\delta$ we have $\bigcap_{\ell=1}^{\infty}\left\{\ucW_{\ell} \subseteq \underline{\A}_{\ell}\right\}$. This implies that
\begin{align*}
\rho^G\left(\ucW_{\ell}\right) & =\min _{\bb \in \Delta_{\ucW}} \max _{\uw, \uw^{\prime} \in \ucW_{\ell}}\left\|\uw-\uw^{\prime}\right\|_{\left(\sum_{\uw \in \ucW} \bb_{\uw} \uw\ \uw^{\top} +\bLambda \right)^{-1}}^2 \\
& \leq \min _{\bb \in \Delta \ucW} \max _{\uw, \uw^{\prime} \in \underline{\A}_{\ell}}\left\|\uw-\uw^{\prime}\right\|_{\left(\sum_{\uw \in \ucW} \bb_{\uw} \uw\ \uw^{\top} + \bLambda\right)^{-1}}^2 \\
& =\rho^G\left(\underline{\A}_{\ell}\right).
\end{align*}
Define $k^{\ell}_1 = 8k\log(1+\tau^G_{\ell-1}/\lambda)$.
For $\ell \geq\left\lceil\log _2\left(4 \Delta^{-1}\right)\right\rceil$ we have that $\underline{\A}_{\ell}=\left\{\uw_*\right\}$, thus, the sample complexity to identify $\uw_*$ is equal to
\begin{align*}
\sum_{\ell=1}^{\left\lceil\log _2\left(4 \Delta^{-1}\right)\right\rceil} &\sum_{\uw \in \ucW}\left\lceil\tau^G_{\ell} \wb^G_{\ell, \uw}\right\rceil  =\sum_{\ell=1}^{\left\lceil\log _2\left(4 \Delta^{-1}\right)\right\rceil}\left(\frac{(k^\ell_1+1) k^\ell_1}{2}+\tau^G_{\ell}\right) \\
& =\sum_{\ell=1}^{\left\lceil\log _2\left(4 \Delta^{-1}\right)\right\rceil}\left(\frac{(k^\ell_1+1) k^\ell_1}{2}+2 \epsilon_{\ell}^{-2} \rho^G(\ucW_\ell) B^\ell_*\log \left(4 k^\ell_1 \ell^2|\ucW| / \delta\right)\right) \\
& \overset{(a)}{\leq} 2\sum_{\ell=1}^{\left\lceil\log _2\left(4 \Delta^{-1}\right)\right\rceil}\left(\frac{(k+1) k}{2}\log^2(1+\tau^G_{\ell-1}) + 2 \epsilon_{\ell}^{-2} \rho^G(\ucW_\ell) B^\ell_*\log \left(4 k^\ell_1 \ell^2|\ucW| / \delta\right)\right) \\
& \overset{(b)}{\leq} 2\frac{(k+1) k}{2}\sum_{\ell=1}^{\left\lceil\log _2\left(4 \Delta^{-1}\right)\right\rceil}\left(\log^2(1+\tau^G_{\ell-1}) + 8 \epsilon_{\ell}^{-2} \rho^G(\ucW_\ell) B^\ell_*\log \left(4 k^\ell_1 \ell^2|\ucW| / \delta\right)\right) \\
& \overset{(c)}{\leq}(k+1) k\sum_{\ell=1}^{\left\lceil\log _2\left(4 \Delta^{-1}\right)\right\rceil}\left(1 + 16 \epsilon_{\ell}^{-2} \rho^G(\ucW_\ell) B^\ell_*\log^2 \left(4 k^\ell_1 \ell^2|\ucW| / \delta\right)\right) \\
& \overset{(d)}{\leq} (k+1) k\left\lceil\log _2\left(4 \Delta^{-1}\right)\right\rceil+\sum_{\ell=1}^{\left\lceil\log _2\left(4 \Delta^{-1}\right)\right\rceil} 32 \epsilon_{\ell}^{-2} f\left(\underline{\A}_{\ell}\right) B^\ell_*\log \left(4 k \ell^2|\ucW| / \delta\right) \\
&\overset{(e)}{\leq} (k+1) k\left\lceil\log _2\left(4 \Delta^{-1}\right)\right\rceil+\sum_{\ell=1}^{\left\lceil\log _2\left(4 \Delta^{-1}\right)\right\rceil} 32 \epsilon_{\ell}^{-2} f\left(\underline{\A}_{\ell}\right) (64\lambda S^2 + 64\tau^G_{\ell-1})\log \left(4 k \ell^2|\ucW| / \delta\right) \\
&\overset{}{=} (k+1) k\left\lceil\log _2\left(4 \Delta^{-1}\right)\right\rceil+\sum_{\ell=1}^{\left\lceil\log _2\left(4 \Delta^{-1}\right)\right\rceil} 32 \epsilon_{\ell}^{-2} f\left(\underline{\A}_{\ell}\right)(64\lambda S^2) \log \left(4 k \ell^2|\ucW| / \delta\right)  \\
&\quad + (k+1) k\left\lceil\log _2\left(4 \Delta^{-1}\right)\right\rceil+\sum_{\ell=1}^{\left\lceil\log _2\left(4 \Delta^{-1}\right)\right\rceil} 32 \epsilon_{\ell}^{-2} f\left(\underline{\A}_{\ell}\right)(64\tau^G_{\ell-1})\log \left(4 k \ell^2|\ucW| / \delta\right) \\
&\overset{(f)}{\leq} (k+1) k\left\lceil\log _2\left(4 \Delta^{-1}\right)\right\rceil+\sum_{\ell=1}^{\left\lceil\log _2\left(4 \Delta^{-1}\right)\right\rceil} 64 \epsilon_{\ell}^{-2} f\left(\underline{\A}_{\ell}\right) (64\lambda S^2)\log \left(4 k \ell^2|\ucW| / \delta\right) \\
& \leq (k+1) k\left\lceil\log _2\left(4 \Delta^{-1}\right)\right\rceil + 2048 \lambda S^2\log \left(\frac{4k \log _2^2\left(8 \Delta^{-1}\right)|\ucW|}{\delta}\right) \sum_{\ell=1}^{\left\lceil\log _2\left(4 \Delta^{-1}\right)\right\rceil} 2^{2 \ell} f\left(\underline{\A}_{\ell}\right) 
\end{align*}
where, $(a)$ follows as $\log^2(1+\tau^G_{\ell-1}/\lambda) \leq \log^2(1+\tau^G_{\ell-1})$, $(b)$ follows by noting that $\log(x \log(1 + x)) \leq 2\log(x)$ for any $x > 1$. The $(c)$ follows by subsuming the $\log^2(1+\tau^G_{\ell-1})$ into $2\tau^G_\ell$. The $(d)$ follows as $\log(1+\tau^G_{\ell-1}) < \tau^G_\ell$ which enables us to replace the $k^\ell_1$ inside the $\log$ with an additional factor of $2$. 
The $(e)$ follows by noting that 
\begin{align}
    B^\ell_* &\leq 64(\sqrt{\lambda} S + \sqrt{\lambda^\perp_\ell} S^{\perp}_\ell)\nonumber\\
    &\leq 64\lambda S^2 + \left(\dfrac{64\tau^G_{\ell-1}}{8(d_1+d_2)r\log(1+\frac{\tau^G_{\ell-1}}{\lambda})}\right)\cdot\left(\frac{8d_1 d_2 r}{\tau^E_{\ell} S^2_{r}} \log \left(\frac{d_1+d_2}{\delta_\ell}\right)\right)\nonumber\\
    &\overset{(a_1)}{\leq} 64\lambda S^2 + 64\tau^G_{\ell-1}. \label{eq:upper-bound_B_star}
\end{align}
where, $(a_1)$ follows by first substituting the value of $\tau^E_{\ell} \coloneqq \frac{\sqrt{8d_1 d_2 r\log (4\ell^2|\W| / \delta_\ell)}}{S_{r}}$ and noting that $\sqrt{(d_1d_2r)}\leq (d_1 + d_2)r$ and cancelling out the other terms. Finally the $(f)$ follows by subsuming the $\tau^G_{\ell-1}$ with a factor of $2$ into the quantity of $\tau^G_{\ell}$. 
Then it follows that

\begin{align*}
& \rho^{G}_* =\inf _{\bb \in \triangle_{\ucW}} \max _{\uw \in \ucW} \frac{\left\|\uw-\uw_*\right\|_{\left(\sum_{\uw \in \ucW} \bb_{\uw} \uw\ \uw^{\top} + \bLambda\right)^{-1}}^2}{\left(\left\langle \uw-\uw_*, \btheta^*\right\rangle\right)^2} \\
& =\inf _{\bb \in \triangle_{\ucW}} \max _{\ell \leq\left\lceil\log _2\left(4 \Delta^{-1}\right)\right\rceil} \max _{\uw \in \underline{\A}_{\ell}} \frac{\left\|\uw-\uw_*\right\|_{\left(\sum_{\uw \in \ucW} \bb_{\uw} \uw\ \uw^{\top}+ \bLambda\right)^{-1}}^2}{\left(\left\langle \uw-\uw_*, \btheta^*\right\rangle\right)^2} \\
& \geq \frac{1}{\left\lceil\log _2\left(4 \Delta^{-1}\right)\right\rceil} \inf _{\bb \in \triangle_{\ucW}} \sum_{\ell=1}^{\left\lceil\log _2\left(4 \Delta^{-1}\right)\right\rceil} \max _{\uw \in \underline{\A}_{\ell}} \frac{\left\|\uw-\uw_*\right\|_{\left(\sum_{\uw \in \ucW} \bb_{\uw} \uw\ \uw^{\top}+ \bLambda\right)^{-1}}^2}{\left(\left\langle \uw-\uw_*, \btheta^*\right\rangle\right)^2} \\
& \geq \frac{1}{16\left\lceil\log _2\left(4 \Delta^{-1}\right)\right\rceil} \sum_{\ell=1}^{\left\lceil\log _2\left(4 \Delta^{-1}\right)\right\rceil} 2^{2 \ell} \inf _{\bb \in \triangle_{\ucW}} \max _{\uw \in \underline{\A}_{\ell}}\left\|\uw-\uw_*\right\|_{\left(\sum_{\uw \in \ucW} \bb_{\uw} \uw\ \uw^{\top}+ \bLambda\right)^{-1}}^2 \\
& \geq \frac{1}{64\left\lceil\log _2\left(4 \Delta^{-1}\right)\right\rceil} \sum_{\ell=1}^{\left\lceil\log _2\left(4 \Delta^{-1}\right)\right\rceil} 2^{2 \ell} \inf _{\bb \in \triangle_{\ucW}} \max _{\uw, \uw^{\prime} \in \underline{\A}_{\ell}}\left\|\uw-\uw^{\prime}\right\|_{\left(\sum_{\uw \in \ucW} \bb_{\uw} \uw\ \uw^{\top}+ \bLambda\right)^{-1}}^2 \\
& \geq \frac{1}{64\left\lceil\log _2\left(4 \Delta^{-1}\right)\right\rceil} \sum_{\ell=1}^{\left\lceil\log _2\left(4 \Delta^{-1}\right)\right\rceil} 2^{2 \ell} f\left(\underline{\A}_{\ell}\right).
\end{align*}
This implies that 
\begin{align*}
    \sum_{\ell=1}^{\left\lceil\log _2\left(4 \Delta^{-1}\right)\right\rceil} 2^{2 \ell} f\left(\underline{\A}_{\ell}\right) \leq \rho^{G}_* 64\left\lceil\log _2\left(4 \Delta^{-1}\right)\right\rceil
\end{align*}
Plugging this back we get
\begin{align*}
    \sum_{\ell=1}^{\left\lceil\log _2\left(4 \Delta^{-1}\right)\right\rceil} \sum_{\uw \in \ucW}\left\lceil\tau^G_{\ell} \widehat{\bb}_{\ell, \uw}\right\rceil 
    &\leq (k+1) k\left\lceil\log _2\left(4 \Delta^{-1}\right)\right\rceil 
    + 2048 \lambda S^2\log \left(\frac{8k \log _2^2\left(8 \Delta^{-1}\right)|\ucW|}{\delta}\right) \rho^{G}_*64\left\lceil\log _2\left(4 \Delta^{-1}\right)\right\rceil\\
    &\leq (k+1) k\left\lceil\log _2\left(4 \Delta^{-1}\right)\right\rceil 
    + C_2 \lambda S^2\log \left(\frac{8 k \log _2^2\left(8 \Delta^{-1}\right)|\ucW|}{\delta}\right) \rho^{G}_*\left\lceil\log _2\left(4 \Delta^{-1}\right)\right\rceil
\end{align*}
for some constant $C_2 > 0$.
Now to understand the bound we need the following:
Let $\operatorname{conv}(\ucW \cup-\ucW)$ denote the convex hull of $\ucW \cup-\ucW$, and for any set $\Y \subset \mathbb{R}^p$ define the gauge of $\Y$
\begin{align}
\gamma_{\Y}=\max \{c>0: c \Y \subset \operatorname{conv}(\ucW \cup-\ucW)\} \label{eq:gauge-norm}
\end{align}
In the case where $\Y$ is a singleton $\Y=\{y\}, \gamma(y):=\gamma_{\Y}$ is the gauge norm of $\by$ with respect to $\operatorname{conv}(\ucW \cup-\ucW)$.
%
%
We can provide a natural upper bound for $\rho(\Y)$ in terms of the gauge.
Observe that
\begin{align*}
    \rho^G_* &= \inf _{\bb \in \triangle_{\ucW}} \max_{\by\in\Y}\left\|\by\right\|_{\left(\sum_{\uw \in \ucW} \bb_{\uw} \uw\ \uw^{\top}+ \bLambda\right)^{-1}}^2\\
    &= \frac{1}{\gamma_{\Y}^2}\inf _{\bb \in \triangle_{\ucW}} \max_{\by\in\Y}\left\|\by\gamma_{\Y}\right\|_{\left(\sum_{\uw \in \ucW} \bb_{\uw} \uw\ \uw^{\top}+ \bLambda\right)^{-1}}^2\\
    &\leq \frac{1}{\gamma_{\Y}^2}\inf _{\bb \in \triangle_{\ucW}} \max_{\uw\in\operatorname{conv}(\ucW\bigcup-\ucW)}\left\|\uw\right\|_{\left(\sum_{\uw \in \ucW} \bb_{\uw} \uw\ \uw^{\top}+ \bLambda\right)^{-1}}^2\\
    &\overset{(a)}{=} \frac{1}{\gamma_{\Y}^2}\inf _{\bb \in \triangle_{\ucW}} \max_{\uw\in\ucW}\left\|\uw\right\|_{\left(\sum_{\uw \in \ucW} \bb_{\uw} \uw\ \uw^{\top}+ \bLambda\right)^{-1}}^2 \overset{(b)}{\leq} k^{3/2}
    \\
    &\leq \frac{k}{\gamma_{\Y}^2} O\left(B_*\log \left(\frac{k\log_2\left( \Delta^{-1}\right)|\ucW|}{\delta}\right) \left\lceil\log _2\left(\Delta^{-1}\right)\right\rceil\right)
\end{align*}
The $(a)$ follows from the fact that the maximum value of a convex function on a convex set must occur at a vertex. The $(b)$ follows from Kiefer-wolfowitz theorem for $\uw\in\R^p$ such that $\inf _{\bb \in \triangle_{\ucW}} \max_{\uw\in\ucW}\left\|\uw\right\|_{\left(\sum_{\uw \in \ucW} \bb_{\uw} \uw\ \uw^{\top}+ \bLambda\right)^{-1}}^2 \leq k\log(1+\frac{\tau^G_{\ell-1}}{\lambda})$. 
The simplified sample complexity for the second stage is given by
\begin{align*}
    N_2\leq O\left(\dfrac{k}{\Delta^2}\log \left(\frac{k\log_2\left( \Delta^{-1}\right)|\ucW|}{\delta}\right)\right) = \widetilde{O}\left(\dfrac{(d_1+d_2)r}{\Delta^2}\right)
\end{align*}
where $\Delta = \min_{\uw\in\ucW} (\uw_* - \uw)^\top\btheta_*  \overset{(a_1)}{=} \min_{\bx\in\X\setminus\{\bx_*\},\bz\in\Z\setminus\{\bz_*\} }(\bx_*^\top\bTheta_*\bz_* - \bx^\top\bTheta_*\bz)$. The $(a_1)$ follows by reshaping the arms in $\ucW$ to recover the arms in $\X$ and $\Z$. 


\textbf{1st Stage:} First recall that the $E$-optimal design in step $3$ of \Cref{alg:bandit-pure} satisfies the \Cref{assm:distribution} as the sample distribution $\D$ has finite second order moments.
%
%
%
For the first stage first observe that by plugging in the definition of $\tau^E_\ell = \frac{\sqrt{8d_1 d_2 r\log (4\ell^2|\W| / \delta)}}{S_{r}}$ we get
\begin{align*}
\left\|\btheta_{k+1: p}^*\right\|_2^2 & =\sum_{i>r \wedge j>r} H_{i j}^2=\left\|(\widehat{\mathbf{U}}_\ell^{\perp})^{\top}\left(\mathbf{U}^* \mathbf{S}^* \bV^{* \top}\right) \widehat{\bV}_\ell^{\perp}\right\|_F^2 \\
& \leq\left\|(\widehat{\mathbf{U}}_\ell^{\perp})^{\top} \mathbf{U}^*\right\|_F^2\left\|\mathbf{S}^*\right\|_2^2\left\|(\widehat{\bV}_\ell^{\perp})^{\top} \bV^*\right\|_F^2  \leq O\left(\frac{d_1 d_2 r}{\tau^E_\ell S^2_r } \log \left(\frac{d_1+d_2}{\delta}\right)\right)\\
&= O\left(\dfrac{\sqrt{d_1 d_2 r}}{S_{r}} \log \left(\frac{d_1+d_2}{\delta}\right)\right)
\end{align*}
which implies $\left\|\btheta_{k+1: p}^*\right\|_2=\tO\left(\sqrt{d_1d_2 r}/S_r\right)$.
Now we bound the sample complexity from the first stage. 
From the first stage we can show that we have for the arm set $\ocW$
\begin{align*}
    N_1 &= \sum_{\ell=1}^{\left\lceil\log _2\left(4 \Delta^{-1}\right)\right\rceil} \sum_{\ow \in \ocW}\left\lceil\tau^E_{\ell} \wb^E_{\ell, \ow}\right\rceil\\
    & =\sum_{\ell=1}^{\left\lceil\log _2\left(4 \Delta^{-1}\right)\right\rceil}\left(\frac{(p+1) p}{2}+\tau^E_{\ell}\right) \\
& =\sum_{\ell=1}^{\left\lceil\log _2\left(4 \Delta^{-1}\right)\right\rceil}\left(\frac{(p+1) p}{2}+ \dfrac{\sqrt{8d_1 d_2 r\log \left(4  \ell^2|\W| / \delta\right)}}{S_r}\right) \\
& \leq (p+1) p\left\lceil\log _2\left(4 \Delta^{-1}\right)\right\rceil + 32 \dfrac{\sqrt{d_1 d_2 r}}{S_r}\log \left(\frac{4 \log _2^2\left(8 \Delta^{-1}\right)|\W|}{\delta}\right) \left\lceil\log _2\left(4 \Delta^{-1}\right)\right\rceil \\
    &\overset{(a)}{=} O\left(\dfrac{\sqrt{d_1 d_2 r}}{S_r}\log \left(\frac{4 \log _2^2\left(8 \Delta^{-1}\right)|\W|}{\delta}\right)\right) = \widetilde{O}\left(\dfrac{\sqrt{d_1 d_2 r}}{S_r}\right)
\end{align*}
where, $(a)$ follows as $p = d_1d_2$.
Combining $N_1$ and $N_2$ gives the claim of the theorem.
\end{proof}

\subsection{Multi-Task Pure Exploration Proofs}
\label{app:multi-bi}
\begin{remark}\textbf{(Comparison with \citet{du2023multi})}
\label{remark:comparison}
    In this remark, we discuss a key comparison of DouExpDes \citep{du2023multi} with \gob. Note that DouExpDes does not implement the second stage of finding the $\bS_{m,*}\in\R^{k_1\times k_2}$ for each of the $m$ bilinear bandits. Hence, DouExpDes does not rotate the arms so that the last $\left(k_1-r\right) \cdot\left(k_2-r\right)$ components are from the complementary subspaces of the left and right eigenvectors of $\bS_{m,*}$. This results in DouExpDes suffering a sample complexity of $\tO(k_1k_2/\Delta^2)$ even though it learns the common feature extractors shared across the tasks. In contrast \gob\ uses the second stage to learn $\bS_{m,*}\in\R^{k_1\times k_2}$ and reduces the latent bilinear bandits problem of $k_1k_2$ dimension to $(k_1+k2)r$ dimension by rotating the arms so that the last $\left(k_1-r\right) \cdot\left(k_2-r\right)$ components are from the complementary subspaces of the left and right eigenvectors of $\bS_{m,*}$. Hence, \gob\ suffers a sample complexity of $\tO((k_1 + k_2)r/\Delta^2)$.
\end{remark} 

\begin{remark}\textbf{(Arm set)} 
\label{remark:arm-set}
The observable left and right arm sets $\X$ and $\Z$ are common across the $M$ tasks. This leads to each task estimating the same E-optimal design in line $3$ of \Cref{alg:bandit-pure-multi} of stage $1$. Note that \citet{du2023multi} also uses a similar idea of the same arm set $\X$ shared across tasks in the linear bandit setting.
Observe that if each task has access to its own separate arm sets $\X_m$ and $\Z_m$, then each of the $m$-tasks has to estimate a separate E-optimal design for the stage $1$. This will lead to the sample complexity of the first stage scaling as $\tO(M\sqrt{d_1d_2}/S_r)$ instead of  $\tO(\sqrt{d_1d_2 r}/S_r)$.
\end{remark}

\textbf{Good Event:} We first recall the total stage $1$ length as
\begin{align*}
    \tau^E_{\ell} \coloneqq 
    \frac{\sqrt{8d_1 d_2 r\log (4\ell^2|\W| / \delta_\ell)}}{ S_{r}}.
\end{align*}
Then define 
the good event $\F_\ell$ in phase $\ell$ that \gob\ has a good estimate of $\bZ_* = \frac{1}{M}\sum_{m=1}^M\bTheta_m$ as follows: For any phase $\ell >0$
\begin{align}
\F_\ell \coloneqq\left\{ \left\|\wZ_{\ell} - \mu^*\bZ_*\right\|_F^2 \leq \frac{C_1 d_1 d_2 r \log \left(\frac{2\left(d_1+d_2\right)}{\delta_\ell}\right)}{\tau^E_\ell}\right\}, \label{eq:good-event-Theta-multi}
\end{align}
where, $C_1=36\left(4+S_0^2\right) C$, $\|\bX\|_F,\left\|\bTheta_*\right\|_F \leq S_0$, some nonzero constant $\mu^*$, $\E\left[\left(S^{\bp}(\bX)\right)_{i j}^2\right] \leq C, \forall i, j$, 
and $\wTheta_{\ell}$ is the estimate from \eqref{eq:convex-prog-multi}.
%
Then define the event 
\begin{align}
    \F \coloneqq \bigcap_{\ell=1}^\infty \F_\ell \label{eq:good-event-Theta-multi-F}
\end{align}


Then we start by modifying \Cref{theorem:kang-low-rank} for the multi-task setting. We first prove this support lemma for the loss function defined in  \eqref{eq:convex-prog-multi}.

\begin{lemma}
\label{lemma:B4}
Let $L: \mathbb{R}^{d_1 \times d_2} \rightarrow \mathbb{R}$ is the loss function defined in  \eqref{eq:convex-prog-multi}. Then by setting
\begin{align*}
t &=\sqrt{2 d_1 d_2 C\left(4 +S_0^2\right) \log \left(\frac{2\left(d_1+d_2\right)}{\delta_\ell}\right)}, \\
\nu &=\frac{t}{\left(4 +S_0\right) C d_1 d_2 \sqrt{M \tau^E_\ell}}=\sqrt{\frac{2 \log \left(\frac{2\left(d_1+d_2\right)}{\delta_\ell}\right)}{M\tau^E_\ell d_1 d_2 C\left(4 +S_0^2\right)}},
\end{align*}
we have with probability at least $1-\delta_\ell$, it holds that
\begin{align*}
\Pb\left(\left\|\nabla L\left(\mu^* \bZ_*\right)\right\|_{o p} \geq \frac{2 t}{\sqrt{M \tau^E_\ell}}\right) \leq \delta_\ell,
\end{align*}
where $\mu^*=\dfrac{2}{M}\E\left[\left\langle \bX_m, \bZ_*\right\rangle\right] > 0$, and $\bX_{m} = \bx_{m}\bz^\top_{m}$.
\end{lemma}

\begin{proof}
Let $\bZ_* = \frac{1}{M}\sum_{m=1}^M\bTheta_{m,*}$. Let $\bX_{m,i} = \bx_{m,i}\bz^\top_{m,i}$ for $i\in[\tau^E_{\ell}]$. 
Based on the definition of our loss function $L(\cdot)$ in \eqref{eq:convex-prog-multi}, we have that
\begin{align*}
\nabla_{x_m} L\left(\bZ_*\right) & =\mu^*\bZ_*-\frac{2}{M\tau^E_{\ell}} \sum_{m=1}^M\sum_{i=1}^{\tau^E_{\ell}} \widetilde{\psi}_\nu(r_{m,i} \cdot Q(x_m)) \\
& =\frac{2}{M} \E\left[\left\langle \bX_{m,1}, \bZ_*\right\rangle\right] \bZ_*-\frac{2}{M\tau^E_{\ell}} \sum_{m=1}^M\sum_{i=1}^{\tau^E_{\ell}} \widetilde{\psi}_\nu\left(r_{m,i} \cdot Q\left(\bX_{m,i}\right)\right) \\
& \overset{(a)}{=} \frac{2}{M} \E\left[\left\langle \bX_{m,1}, \bZ_*\right\rangle Q\left(\bX_{m,1}\right)\right]-\frac{2}{M\tau^E_{\ell}} \sum_{m=1}^M\sum_{i=1}^{\tau^E_{\ell}} \widetilde{\psi}_\nu\left(r_{m,i} \cdot Q\left(\bX_{m,i}\right)\right) \\
& \overset{(b)}{=} \frac{2}{M}\left[\E\left(r_{m,1} \cdot Q\left(\bX_{m,1}\right)\right)-\frac{1}{M\tau^E_{\ell}} \sum_{m=1}^M\sum_{i=1}^{\tau^E_{\ell}} \widetilde{\psi}_\nu\left(r_{m,i} \cdot Q\left(\bX_{m,i}\right)\right)\right]
\end{align*}
where we have (a) due to the generalized Stein's Lemma stated in \Cref{lemma:kang-b2}, and (b) comes from the fact that the random noise $\eta_1=y_1-\left\langle \bX_1, \bZ_*\right\rangle$ is zero-mean and independent from $\bX_1$. Therefore, in order to implement the \Cref{lemma:kang-b3}, we can see that it suffices to get $\sigma^2$ defined as:
$$
\sigma^2=\max \left(\left\|\frac{2}{M}\sum_{m=1}^M\sum_{j=1}^{\tau^E_\ell} \E\left[r_{m,j}^2 Q\left(\bX_{m,j}\right) Q\left(\bX_{m,j}\right)^{\top}\right]\right\|_{\mathrm{op}},\left\|\frac{2}{M}\sum_{m=1}^M\sum_{j=1}^{\tau^E_\ell} \E\left[r_{m,j}^2 Q\left(\bX_{m,j}\right)^{\top} Q\left(\bX_{m,j}\right)\right]\right\|_{\mathrm{op}}\right) .
$$
\begin{align*}
&\left\|\dfrac{2}{M}\sum_{m=1}^M\sum_{j=1}^{\tau^E_{\ell}} \E\left[r_{m,j}^2 Q\left(\bX_{m,j}\right) Q\left(\bX_{m,j}\right)^{\top}\right]\right\|_{\text {op }}  \leq \tau^E_{\ell} \times\left\|\E\left[r_{m,1}^2 Q\left(\bX_{m,1}\right) Q\left(\bX_{m,1}\right)^{\top}\right]\right\|_{\text {op }} \\ 
& \overset{(a)}{=}\tau^E_{\ell} \times\left\|\E\left[\left(\eta_{m,1}+\left\langle \bX_{m,1}, \bZ_*\right\rangle\right)^2 Q\left(\bX_{m,1}\right) Q\left(\bX_{m,1}\right)^{\top}\right]\right\|_{\text {op }} \\ 
& \left.\overset{(b)}{=}\tau^E_{\ell} \times \| \E\left[\eta_{m,1}^2 Q\left(\bX_{m,1}\right) Q\left(\bX_{m,1}\right)^{\top}\right]+\E\left[\left\langle \bX_{m,1}, \bZ_*\right\rangle\right)^2 Q\left(\bX_{m,1}\right) Q\left(\bX_{m,1}\right)^{\top}\right] \|_{\text {op }} \\  
&\left.\overset{(c)}{=}\tau^E_{\ell} \times \| \E\left(\eta_{m,1}^2\right) \E\left[Q\left(\bX_{m,1}\right) Q\left(\bX_{m,1}\right)^{\top}\right]+\E\left[\left\langle \bX_{m,1}, \bZ_*\right\rangle\right)^2 Q\left(\bX_{m,1}\right) Q\left(\bX_{m,1}\right)^{\top}\right] \|_{\text {op }} \\ 
& \overset{(d)}{\leq} \tau^E_{\ell} \times\left\|4 \E\left[Q\left(\bX_{m,1}\right) Q\left(\bX_{m,1}\right)^{\top}\right]+S_0^2 \E\left[Q\left(\bX_{m,1}\right) Q\left(\bX_{m,1}\right)^{\top}\right]\right\|_{\text {op }} \\ 
& =\left(4 +S_0^2\right) \tau^E_{\ell} \times\left\|\E\left[Q\left(\bX_{m,1}\right) Q\left(\bX_{m,1}\right)^{\top}\right]\right\|_{\text {op }}
\end{align*}
where the $(a)$ follows by plugging in the definition for reward, $(b)$ follows by the linearity of expectation, $(c)$ follows as noises are independent, and the
inequality $(d)$ comes from the fact that $\left|\left\langle \bX_{m,1}, \bZ_*\right\rangle\right| \leq S_0$, and $Q\left(\bX_{m,1}\right) Q\left(\bX_{m,1}\right)^{\top}$ is always positive semidefinite. Next, since we know that $\E\left[Q\left(\bX_{m,1}\right) Q\left(\bX_{m,1}\right)^{\top}\right]$ is always symmetric and positive semidefinite, and hence we have
\begin{align*}
\left\|\E\left[Q\left(\bX_{m,1}\right) Q\left(\bX_{m,1}\right)^{\top}\right]\right\|_{\mathrm{op}} & \overset{(a)}{\leq}\left\|\E\left[Q\left(\bX_{m,1}\right) Q\left(\bX_{m,1}\right)^{\top}\right]\right\|_{\text {nuc }}=\operatorname{trace}\left(\E\left[Q\left(\bX_{m,1}\right) Q\left(\bX_{m,1}\right)^{\top}\right]\right) \\
& =\E\left[\operatorname{trace}\left(Q\left(\bX_{m,1}\right) Q\left(\bX_{m,1}\right)^{\top}\right)\right] = \E\left(\dfrac{2}{M}\sum_{m=1}^M\sum_{i=1}^{d_1} \sum_{j=1}^{d_2} Q_{i j}\left(\bX_{m,1}\right)^2\right) \\
& \leq d_1 d_2 C.
\end{align*}
where, in $(a)$ $\|\cdot\|_{\text {nuc}}$ denotes the nuclear norm.
Therefore, we have that under $1$-subGaussian assumption
$$
\left\|\dfrac{2}{M}\sum_{m=1}^M\sum_{j=1}^{\tau^E_{\ell}} \E\left[r_{m,j}^2 Q\left(\bX_{m,j}\right) Q\left(\bX_{m,j}\right)^{\top}\right]\right\|_{\text {op }} \leq\left(4 +S_0^2\right) d_1 d_2 \tau^E_{\ell} C.
$$
And similarly, we can prove that
$$
\left\|\dfrac{2}{M}\sum_{m=1}^M\sum_{j=1}^{\tau^E_{\ell}} \E\left[r_{m,j}^2 Q\left(\bX_{m,j}\right)^{\top} Q\left(\bX_{m,j}\right)\right]\right\|_{\mathrm{op}} \leq\left(4 +S_0^2\right) d_1 d_2 \tau^E_{\ell} C.
$$
Therefore, we can take $\sigma^2=\left(4 +S_0^2\right) d_1 d_2 \tau^E_{\ell} C$ consequently. By using \Cref{lemma:kang-b3}, we have
$$
\Pb\left(\left\|\nabla L\left(\mu^* \bZ_*\right)\right\|_{\mathrm{op}} \geq \frac{2 t}{\sqrt{M\tau^E_{\ell}}}\right) \leq 2\left(d_1+d_2\right) \exp \left(-\nu t \sqrt{M\tau^E_{\ell}}+\frac{\nu^2\left(4 +S_0^2\right) C d_1 d_2 \tau^E_{\ell}}{2}\right)
$$
By plugging the values of $t$ and $\nu$ in \Cref{lemma:B4}, we finish the proof.
\end{proof}

\begin{lemma}
\label{theorem:kang-low-rank-multi}
For any low-rank linear model with samples $\bX_1 \ldots, \bX_{\tau^E_\ell}$ drawn from $\X$ according to $\mathcal{D}$ then for the optimal solution to the nuclear norm regularization problem in \eqref{eq:convex-prog} with $\nu=\sqrt{2 \log \left(2\left(d_1+d_2\right) / \delta_\ell\right) /\left(\left(4 +S_0^2\right) M \tau^E_\ell d_1 d_2\right)}$ and
\begin{align*}
\gamma_{\ell}=4 \sqrt{\frac{2\left(4 +S_0^2\right) C d_1 d_2 \log \left(2\left(d_1+d_2\right) / \delta_\ell\right)}{M\tau^E_{\ell}}},
\end{align*}
with probability at least $1-\delta_\ell$ it holds that:
\begin{align*}
\left\|\wZ_{\ell}- \mu^*\bZ_*\right\|_F^2 \leq \frac{C_1 d_1 d_2 r \log \left(\frac{2\left(d_1+d_2\right)}{\delta_\ell}\right)}{M\tau^E_{\ell}},
\end{align*}
for $C_1=36\left(4+S_0^2\right) C$, $\|\bX\|_F,\left\|\bZ_*\right\|_F \leq S_0$, some nonzero constant $\mu^*$, and $\E\left[\left(S^{\bp}(\bX)\right)_{i j}^2\right] \leq C, \forall i, j$. Summing over all phases $\ell \geq 1$ it follows that $\Pb(\F) \geq 1 - \delta/2$.
\end{lemma}

\begin{proof}
Since the estimator $\wZ_{\ell}$ minimizes the regularized loss function defined in \eqref{eq:convex-prog-multi}, we have
$$
L(\wZ_{\ell})+\gamma_{\ell}\|\wZ_{\ell}\|_{\mathrm{nuc}} \leq L\left(\mu^* \bZ_*\right)+\gamma_{\ell}\left\|\mu^* \bZ_*\right\|_{\mathrm{nuc}}
$$
And due to the fact that $L(\cdot)$ is a quadratic function, we have the following expression based on multivariate Taylor's expansion:
$$
L(\wZ_{\ell})-L\left(\mu^* \bZ_*\right)=\left\langle\nabla L\left(\mu^* \bZ_*\right), \bTheta\right\rangle+2\|\bTheta\|_F^2, \quad \text { where } \bTheta=\wZ_{\ell}-\mu^* \bZ_*
$$
By rearranging the above two results, we can deduce that
\begin{align}
2\|\bTheta\|_F^2 & \leq-\left\langle\nabla L\left(\mu^* \bZ_*\right), \bTheta\right\rangle+\gamma_{\ell}\left\|\mu^* \bZ_*\right\|_{\mathrm{nuc}}-\gamma_{\ell}\|\wZ_{\ell}\|_{\mathrm{nuc}} \nonumber\\
& \overset{(a)}{\leq}\left\|\nabla L\left(\mu^* \bZ_*\right)\right\|_{\mathrm{op}}\|\bTheta\|_{\mathrm{nuc}}+\gamma_{\ell}\left\|\mu^* \bZ_*\right\|_{\mathrm{nuc}}-\gamma_{\ell}\|\wZ_{\ell}\|_{\mathrm{nuc}},\label{eq:15-kang}
\end{align}
where $(a)$ comes from the duality between matrix operator norm and nuclear norm. Next, we represent the saturated SVD of $\bZ_*$ in the main paper as $\bZ_*=\bU \bD \bV^{\top}$ where $\bU \in \mathbb{R}^{d_1 \times r}$ and $\bV \in \mathbb{R}^{d_2 \times r}$, and here we would work on its full version, i.e.
$$
\bZ_*=\left(\bU, \bU_{\perp}\right)\left(\begin{array}{cc}
\bD & 0 \\
0 & 0
\end{array}\right)\left(\bV, \bV_{\perp}\right)^{\top}=\left(\bU, \bU_{\perp}\right) \bD^*\left(\bV, \bV_{\perp}\right)^{\top}
$$
where we have $\bU_{\perp} \in \mathbb{R}^{d_1 \times\left(d_1-r\right)}, \bD^* \in \mathbb{R}^{d_1 \times d_2}$ and $\bV_{\perp} \in \mathbb{R}^{d_2 \times\left(d_2-r\right)}$. Furthermore, we define
$$
\bLambda=\left(\bU, \bU_{\perp}\right)^{\top} \bTheta\left(\bV, \bV_{\perp}\right)=\left(\begin{array}{cc}
\bU^{\top} \bTheta \bV & \bU^{\top} \bTheta \bV_{\perp} \\
\bU_{\perp}^{\top} \bTheta \bV & \bU_{\perp}^{\top} \bTheta V_{\perp}
\end{array}\right)=\bLambda_1+\bLambda_2
$$
where we write
$$
\bLambda_1=\left(\begin{array}{cc}
0 & 0 \\
0 & \bU_{\perp}^{\top} \bTheta \bV_{\perp}
\end{array}\right), \quad \bLambda_2=\left(\begin{array}{cc}
\bU^{\top} \bTheta \bV & \bU^{\top} \bTheta \bV_{\perp} \\
\bU_{\perp}^{\top} \bTheta \bV & 0
\end{array}\right).
$$
Afterward, it holds that
\begin{align}
\|\wZ_{\ell}\|_{\mathrm{nuc}} & =\left\|\mu^* \bZ_*+\bTheta\right\|_{\mathrm{nuc}} \overset{(a)}{=}\left\|\left(\bU, \bU_{\perp}\right)\left(\mu^* \bD^*+\bLambda\right)\left(\bV, \bV_{\perp}\right)^{\top}\right\|_{\mathrm{nuc}} \nonumber\\
& \overset{(b)}{=}\left\|\mu^* \bD^*+\bLambda\right\|_{\mathrm{nuc}}+\left\|\mu^* \bD^*+\bLambda_1+\bLambda_2\right\|_{\mathrm{nuc}} \nonumber\\
& \geq\left\|\mu^* \bD^*+\bLambda_1\right\|_{\mathrm{nuc}}-\left\|\bLambda_2\right\|_{\mathrm{nuc}} \nonumber\\
& =\left\|\mu^* \bD\right\|_{\mathrm{nuc}}+\left\|\bLambda_1\right\|_{\mathrm{nuc}}-\left\|\bLambda_2\right\|_{\mathrm{nuc}} \nonumber\\
& =\left\|\mu^* \bZ_*\right\|_{\mathrm{nuc}}+\left\|\bLambda_1\right\|_{\mathrm{nuc}}-\left\|\bLambda_2\right\|_{\mathrm{nuc}}, \label{eq:16-kang}
\end{align}
where, $(a)$ follows from the definition of $\bZ_*$, and $(b)$ follows from the definition of $\bLambda$.  
This implies that
$$
\left\|\mu^* \bZ_*\right\|_{\mathrm{nuc}}-\|\wZ_{\ell}\|_{\mathrm{nuc}} \leq\left\|\bLambda_2\right\|_{\mathrm{nuc}}-\left\|\bLambda_1\right\|_{\mathrm{nuc}}.
$$
Combining \eqref{eq:15-kang} and \eqref{eq:16-kang}, we have that
$$
2\|\bTheta\|_F^2 \leq\left(\left\|\nabla L\left(\mu^* \bZ_*\right)\right\|_{\mathrm{op}}+\gamma_{\ell}\right)\left\|\bLambda_2\right\|_{\mathrm{nuc}}+\left(\left\|\nabla L\left(\mu^* \bZ_*\right)\right\|_{\mathrm{op}}-\gamma_{\ell}\right)\left\|\bLambda_1\right\|_{\mathrm{nuc}}.
$$
Then, we refer to the setting in our \Cref{lemma:B4}, and we choose $\gamma_\ell=4 t / \sqrt{M\tau^E_{\ell}}$ where the value of $t$ is determined in \Cref{lemma:B4}, i.e.
$$
\gamma_{\ell}=4 \sqrt{\frac{2\left(4 +S_0^2\right) C d_1 d_2 \log \left(2\left(d_1+d_2\right) / \delta_\ell\right)}{M\tau^E_{\ell}}},
$$
we know that $\lambda_{T-1} \geq 2\left\|\nabla L\left(\mu^* \bZ_*\right)\right\|_{o p}$ with probability at least $1-\delta_\ell$ for any $\delta_\ell \in(0,1)$. Therefore, with a probability at least $1-\delta_\ell$, we have
$$
2\|\bTheta\|_F^2 \leq \frac{3}{2} \gamma_{\ell}\left\|\bLambda_2\right\|_{\mathrm{nuc}}-\frac{1}{2} \gamma_{\ell}\left\|\bLambda_1\right\|_{\mathrm{nuc}} \leq \frac{3}{2} \gamma_{\ell}\left\|\bLambda_2\right\|_{\mathrm{nuc}}.
$$
Since we can easily verify that the rank of $\bLambda_2$ is at most $2 r$, and by using Cauchy-Schwarz Inequality we have that
$$
2\|\bTheta\|_F^2 \leq \frac{3}{2} \gamma_{\ell} \sqrt{2 r}\left\|\bLambda_2\right\|_F \leq \frac{3}{2} \gamma_{\ell} \sqrt{2 r}\|\bLambda\|_F=\frac{3}{2} \gamma_{\ell} \sqrt{2 r}\|\bTheta\|_F
$$
which implies that
$$
\|\bTheta\|_F \leq \frac{3}{4} \sqrt{2 r} \gamma_{\ell}=6 \sqrt{\frac{\left(4+S_0^2\right) C d_1 d_2 r \log \left(\frac{2\left(d_1+d_2\right)}{\delta_\ell}\right)}{M\tau^E_{\ell}}}.
$$
This implies that $\Pb(\F_\ell) \geq 1-\delta_\ell$.
Taking a union bound over all phases $\ell \geq 1$ and recalling $\delta_\ell:=\frac{\delta}{2 \ell^2}$, we obtain
\begin{align*}
\Pb(\F) & \geq 1-\sum_{\ell=1}^{\infty} \Pb\left(\F^c_\ell\right) \\
& \geq 1-\sum_{\ell=1}^{\infty} \frac{\delta_\ell}{2} \\
& =1-\sum_{\ell=1}^{\infty} \frac{\delta}{4 \ell^2} \\
& \geq 1-\frac{\delta}{2} .
\end{align*}
This concludes our proof.
\end{proof}

Define $\bX_{\text {batch }}^{+}:=\left(\bX_{\text {batch }}^{\top} \bX_{\text {batch }}\right)^{-1} \bX_{\text {batch }}^{\top}$ where $\bX_{\text {batch }}^{+}$ is constructed through the $E$-optimal design. Using Lemma C.1 from \citet{du2023multi} it holds that
\begin{align*}
\left\|\bX_{\text {batch }}^{+}\right\| \leq \sqrt{\frac{(1+\beta) \rho_\ell^E}{\overline{p}}} .
\end{align*}
where $\overline{p} = 180 d_1d_2/\beta^2$ is the batch size to control the rounding procedure  and $\small\rho^E_{\ell} \!=\! \min_{\bb \in \triangle_{\ocW}}\big\|(\sum_{\ow\in\ocW} \bb_{\ow}\ow \ \ow^\top)^{-1}\big\|$. It follows then that $\left\|\bX_{\text {batch }}^{+}\right\|^2 \leq 4\rho^E_\ell$.

\begin{lemma}
\label{lemma:expectation-multi}
(Expectation of $\wZ_\ell$ ). It holds that $\E\left[\wZ_\ell\right]= \bZ = \frac{1}{M} \sum_{m=1}^M \bTheta_m$.
\end{lemma}

\begin{proof}
%
%
    %
    %
    Note that we can re-write
    \begin{align*}
    \wZ_{\ell} \!=\!\argmin_{\bTheta \in \R^{d_1 \times d_2}} L_{\ell}(\bTheta)+\gamma_{\ell}\|\bTheta\|_{\mathrm{nuc}}, 
    L_{\ell}(\bTheta) \!=\!\langle\bTheta, \bTheta\rangle-\tfrac{2}{M \tau^E_\ell} \sum_{m=1}^M\sum_{s=1}^{\tau^E_\ell}\langle\widetilde{\psi}_\nu(r_{m,s} \cdot Q(\bx_{m,s}\bz_{m,s}^\top)), \bTheta\rangle 
    \end{align*}
    such that
    \begin{align*}
        \wZ_{\ell} = \frac{2}{M \tau^E_\ell} \sum_{m=1}^M \sum_{s=1}^{\tau^E_\ell} \wTheta_{m,s} -\boldsymbol{X}_{\mathrm{batch}}^{+}\left(\boldsymbol{X}_{\mathrm{batch}}^{+}\right)^{\top}
    \end{align*}
    where $\wTheta_{m,s} = \langle\bTheta, \bTheta\rangle-\tfrac{2}{Ms} \sum_{m=1}^M\langle\widetilde{\psi}_\nu(r_{m,s} \cdot Q(\bx_{m,s}\bz_{m,s}^\top)), \bTheta\rangle$. 
    Now using Lemma C.2 from \citet{du2023multi} we can prove the result of the lemma. 
\end{proof}

\begin{lemma}
\label{lemma:david-kahan-U}
(Concentration of $\wB_{1,\ell}$ ). Suppose that event $\F_\ell$ holds. Then, for any phase $\ell>0$,
\begin{align*}
\left\|(\wB^{\perp}_{1,\ell})^{\top} \bB_1\right\| \leq \frac{c' \rho_\ell^E \sqrt{(d_1 + d_2)r} }{S_r\sqrt{M \tau^E_\ell}} \log \left(\frac{16 (d_1 + d_2)r M \tau^E_\ell}{\delta_\ell}\right),
\end{align*}
for some constant $c'>0$ and $\rho^E_{\ell} \!=\! \min_{\bb \in \triangle_{\ocW}}\big\|(\sum_{\ow\in\ocW} \bb_{\ow}\ow \ \ow^\top)^{-1}\big\|$.
\end{lemma}

\begin{proof} Using the Davis-Kahan $\sin \theta$ Theorem \citep{bhatia2013matrix} and letting $\tau^E_\ell$ be large enough to satisfy $\left\|\wZ_{\ell}- \mu^*\bZ_*\right\|_F^2 \leq \frac{C_1 d_1 d_2 r \log \left(\frac{2\left(d_1+d_2\right)}{\delta_\ell}\right)}{M\tau^E_{\ell}}$, we have
\begin{align*}
\left\|(\wB^{\perp}_{1,\ell})^{\top} \bB_1\right\| & \leq \frac{\left\|\wZ_{\ell}-\E\left[\wZ_{\ell}\right]\right\|}{\sigma_r\left(\E\left[\wZ_{\ell}\right]\right)-\sigma_{r+1}\left(\E\left[\wZ_{\ell}\right]\right)-\left\|\wZ_{\ell}-\E\left[\wZ_{\ell}\right]\right\|} \\
& \overset{(a)}{\leq} \frac{c_0}{S_r}\left\|\wZ_{\ell}-\E\left[\wZ_{\ell}\right]\right\| \\
& \overset{(b)}{\leq} \frac{c c_0\left\|\bX_{\text {batch }}^{+}\right\|^2 \sqrt{(d_1 + d_2)r} }{S_r\sqrt{M \tau^E_\ell}} \log \left(\frac{16 (d_1 + d_2)r M \tau^E_\ell}{\delta_\ell}\right) \\
& \overset{(c)}{\leq} \frac{c'\rho^E_\ell \sqrt{(d_1 + d_2)r} }{S_r\sqrt{M \tau^E_\ell}} \log \left(\frac{16 (d_1 + d_2)r M \tau^E_\ell}{\delta_\ell}\right) .
\end{align*}
where, $(a)$ follows from \Cref{assm:diverse-task}, the $(b)$ follows from event $\F_\ell$ and $(b)$ follows as $\left\|\bX_{\text {batch }}^{+}\right\|^2 \leq 4\rho^E_\ell$. 
The claim of the lemma follows.
\end{proof}

\begin{lemma}
\label{lemma:david-kahan-V}
(Concentration of $\wB_{2,\ell}$ ). Suppose that event $\F_\ell$ holds. Then, for any phase $\ell>0$,
\begin{align*}
\left\|(\wB^{\perp}_{2,\ell})^{\top} \bB_2\right\| \leq \frac{c \rho^E_\ell \sqrt{(d_1 + d_2)r} }{S_r\sqrt{M \tau^E_\ell}} \log \left(\frac{16 (d_1 + d_2)r M \tau^E_\ell}{\delta_\ell}\right),
\end{align*}
for some constant $c'>0$ and $\rho^E_{\ell} \!=\! \min_{\bb \in \triangle_{\ocW}}\big\|(\sum_{\ow\in\ocW} \bb_{\ow}\ow \ \ow^\top)^{-1}\big\|$.
\end{lemma}

\begin{proof} The proof follows the same way as \Cref{lemma:david-kahan-U} and using the Davis-Kahan $\sin \theta$ Theorem \citep{bhatia2013matrix} 
\begin{align*}
\left\|(\wB^{\perp}_{2,\ell})^{\top} \bB_2\right\| & 
\leq \frac{c\rho^E_\ell \sqrt{(d_1 + d_2)r} }{S_r\sqrt{M \tau^E_{\ell}}} \log \left(\frac{16 (d_1 + d_2)r M \tau^E_\ell}{\delta_\ell}\right) .
\end{align*}
The claim of the lemma follows.
\end{proof}

\textbf{Good Event per Task:} We now define the good event $\F'_\ell$ in phase $\ell$ that \gob\ has a good estimate of $\bS_{m,*}$ as follows: For any phase $\ell >0$
\begin{align}
\F'_\ell \coloneqq \left\{\left\|\wS_{m,\ell} - \mu^*\bS_{m,*}\right\|_F^2 \leq \frac{C_1 k_1 k_2 r \log \left(\frac{2\left(k_1+k_2\right)}{\delta_\ell}\right)}{\tau^E_{m,\ell}}\right\}, \label{eq:good-event-Theta-multi-wS}
\end{align}
where, $C,\mu^*>0$ are constants and $\wS_{m,\ell}$ is the estimate from \eqref{eq:convex-prog-multi-wS}. Then define the event 
\begin{align}
    \F' \coloneqq \bigcap_{\ell=1}^\infty \F'_\ell. \label{eq:good-event-Theta-multi-wS-S}
\end{align}

We now prove the following lemmas to show the good event $\F'_\ell$ holds with probability $(1-\delta_\ell)$.
Before proving the concentration of $\wS_{m,\ell}$ we first need to show that $\sigma_{\min}(\sum_{\tw\in\tW}\bb_{\tw} \tw \tw^\top) > 0$. If this holds true then we can sample following $E$-optimal design.  

\begin{lemma}
\label{lemma:min-sigma} 
For any phase $\ell>0$ and task $m \in[M]$, let $\left\|\wB_{1,\ell}^{\top} \bB_1^{\perp}\right\| \leq c_1$ and $\left\|\wB_{2,\ell}^{\top} \bB_2^{\perp}\right\| \leq c_2$, for some $c_1, c_2 > 0$. Then we have
\begin{align*}
    \sigma_{\min}(\sum_{\tw\in\tW}\bb_{\tw} \tw \tw^\top) > 0
\end{align*}
\end{lemma}

\begin{proof}
We can show that 
\begin{align*}
    \sum_{\tw\in\tW}\bb_{\tw} \tw \tw^\top \overset{(a)}{=} \sum_{\bx\in\X_m, \bz\in\Z_m} \bb_{x,z} \underbrace{\wU_\ell^\top\bx_m}_{\tg_m} \underbrace{\bz_m\wV_\ell^\top}_{\tv_m^\top}
\end{align*}
where, in $(a)$ the $\bb_{x,z}$ is the sampling proportion for the arms $\bx$ and $\bz$ (they are allocated the same proportion, as they are pulled the same number of times). Also note that from \Cref{theorem:kang-low-rank-multi} we know that $\left\|\wB_{1,\ell}^{\top} \bB_1^{\perp}\right\| \leq c_1$ and $\left\|\wB_{2,\ell}^{\top} \bB_2^{\perp}\right\| \leq c_2$
%
for some $c_1, c_2 > 0$ holds with high probability. This helps us to apply \Cref{lemma:du-C5} to get the claim of the lemma.
\end{proof}

\begin{lemma}\textbf{(Restatement of Lemma C.5 from \citet{du2023multi})}
\label{lemma:du-C5}
For any phase $\ell>0$ and task $m \in[M]$, if $\left\|\wU_\ell^{\top} \bU^{\perp}\right\| \leq c$ for some $c > 0$, then we have
\begin{align*}
\sigma_{\min }\left(\sum_{i=1}^n \bb_m^*\left(\bx_i\right) \wU_\ell^{\top} \bx_i \bx_i^{\top} \wU_\ell\right)>0
\end{align*}
where $\bb_m^*$ is a sampling proportion on $\bx$.
\end{lemma}

\begin{lemma}
\label{lemma:conc-S}
Let $L': \mathbb{R}^{k_1 \times k_2} \rightarrow \mathbb{R}$ is the loss function defined in  \eqref{eq:convex-prog-multi-wS}. Then by setting
\begin{align*}
t &=\sqrt{2 k_1 k_2 C\left(4 +S_0^2\right) \log \left(\frac{2\left(k_1+k_2\right)}{\delta_\ell}\right)}, \\
\nu &=\frac{t}{\left(4 +S_0\right) C k_1 k_2 \sqrt{ \tau^E_\ell}}=\sqrt{\frac{2 \log \left(\frac{2\left(k_1+k_2\right)}{\delta_\ell}\right)}{\tau^E_\ell k_1 k_2 C\left(4 +S_0^2\right)}},
\end{align*}
we have with probability at least $1-\delta_\ell$, it holds that
\begin{align*}
\Pb\left(\left\|\nabla L'\left(\mu^* \bS_{m,*}\right)\right\|_{o p} \geq \frac{2 t}{\sqrt{\tau^E_\ell}}\right) \leq \delta_\ell,
\end{align*}
where $\mu^*=\E\left[\left\langle \bX_m, \bS_{m,*}\right\rangle\right] > 0$, and $\bX_{m} = \tg_{m}\tv^\top_{m}$.
\end{lemma}

\begin{proof}
Let $\bX_{m,i} = \tg_{m,i}\tv^\top_{m,i}$
Based on the definition of our loss function $L'(\cdot)$ in \eqref{eq:convex-prog-multi-wS}, we have that
\begin{align*}
\nabla_{x_m} L'\left(\bS_{m,*}\right) & =\mu^*\bS_{m,*}-\frac{2}{\tau^E_{\ell}} \sum_{i=1}^{\tau^E_{\ell}} \widetilde{\psi}_\nu(r_{m,i} \cdot Q(x_m)) \\
& \overset{(a)}{=} \left[\E\left(r_{m,1} \cdot Q\left(\bX_{m,1}\right)\right)-\frac{1}{\tau^E_{\ell}} \sum_{i=1}^{\tau^E_{\ell}} \widetilde{\psi}_\nu\left(r_{m,i} \cdot Q\left(\bX_{m,i}\right)\right)\right]
\end{align*}
where $(a)$ follows using the same steps as in \Cref{lemma:B4}. Similarly, using the same steps for a single task as in \Cref{lemma:B4} we have
$$
\Pb\left(\left\|\nabla L'\left(\mu^* \bS_{m,*}\right)\right\|_{\mathrm{op}} \geq \frac{2 t}{\sqrt{\tau^E_{\ell}}}\right) \leq 2\left(k_1+k_2\right) \exp \left(-\nu t \sqrt{\tau^E_{\ell}}+\frac{\nu^2\left(4 +S_0^2\right) C k_1 k_2 \tau^E_{\ell}}{2}\right)
$$
By plugging the values of $t$ and $\nu$ in \Cref{lemma:B4}, we finish the proof.
\end{proof}

\begin{lemma}
\label{theorem:kang-low-rank-multi-wS}\textbf{(Concentration of $\wS_{m,\ell}$)}
For any low-rank linear model with samples $\bX_1 \ldots, \bX_{\tau^E_\ell}$ drawn from $\X$ according to $\mathcal{D}$ then for the optimal solution to the nuclear norm regularization problem in \eqref{eq:convex-prog} with $\nu=\sqrt{2 \log \left(2\left(k_1+k_2\right) / \delta_\ell\right) /\left(\left(4 +S_0^2\right)  \tau^E_\ell k_1 k_2\right)}$ and
\begin{align*}
\gamma_{m,\ell}=4 \sqrt{\frac{2\left(4 +S_0^2\right) C k_1 k_2 \log \left(2\left(k_1+k_2\right) / \delta_\ell\right)}{\tau^E_{m,\ell}}},
\end{align*}
with probability at least $1-\delta_\ell$ it holds that:
\begin{align*}
\left\|\wS_{m,\ell}- \mu^*\bS_{m,*}\right\|_F^2 \leq \frac{C_1 k_1 k_2 r \log \left(\frac{2\left(k_1+k_2\right)}{\delta_\ell}\right)}{\tau^E_{m,\ell}},
\end{align*}
for $C_1=36\left(4+S_0^2\right) C$, $\|\bX\|_F,\left\|\bS_{m,*}\right\|_F \leq S_0$, some nonzero constant $\mu^*$, and $\E\left[\left(S^{\bp}(\bX)\right)_{i j}^2\right] \leq C, \forall i, j$. Summing over all phases $\ell \geq 1$ it follows that $\Pb(\F') \geq 1 - \delta/2$.
\end{lemma}

\begin{proof}
Since the estimator $\wS_{m,\ell}$ minimizes the regularized loss function defined in Eqn. (6), we have
$$
L(\wS_{m,\ell})+\gamma_{\ell}\|\wS_{m,\ell}\|_{\mathrm{nuc}} \leq L\left(\mu^* \bS_{m,*}\right)+\gamma_{\ell}\left\|\mu^* \bS_{m,*}\right\|_{\mathrm{nuc}}.
$$
And due to the fact that $L'(\cdot)$ is a quadratic function, we have the following expression based on multivariate Taylor's expansion:
$$
L'(\wS_{m,\ell})-L'\left(\mu^* \bS_{m,*}\right)=\left\langle\nabla L'\left(\mu^* \bS_{m,*}\right), \bTheta\right\rangle+2\|\bTheta\|_F^2, \quad \text { where } \bTheta=\wS_{m,\ell}-\mu^* \bS_{m,*}.
$$
By rearranging the above two results, we can deduce that
\begin{align}
2\|\bTheta\|_F^2 & \leq-\left\langle\nabla L'\left(\mu^* \bS_{m,*}\right), \bTheta\right\rangle+\gamma_{\ell}\left\|\mu^* \bS_{m,*}\right\|_{\mathrm{nuc}}-\gamma_{\ell}\|\wS_{m,*}\|_{\mathrm{nuc}} \nonumber\\
& \overset{(i)}{\leq}\left\|\nabla L'\left(\mu^* \bS_{m,*}\right)\right\|_{\mathrm{op}}\|\bTheta\|_{\mathrm{nuc}}+\gamma_{\ell}\left\|\mu^* \bS_{m,*}\right\|_{\mathrm{nuc}}-\gamma_{\ell}\|\wS_{m,\ell}\|_{\mathrm{nuc}},\label{eq:15-kang-wS}
\end{align}
where (i) comes from the duality between matrix operator norm and nuclear norm. Next, we represent the saturated SVD of $\bS_{m,*}$ as $\bS_{m,*}=\bU \bD \bV^{\top}$ where $\bU \in \mathbb{R}^{k_1 \times r}$ and $\bV \in \mathbb{R}^{k_2 \times r}$, and here we would work on its full version, i.e.
$$
\bS_{m,*}=\left(\bU, \bU_{\perp}\right)\left(\begin{array}{cc}
\bD & 0 \\
0 & 0
\end{array}\right)\left(\bV, \bV_{\perp}\right)^{\top}=\left(\bU, \bU_{\perp}\right) \bD^*\left(\bV, \bV_{\perp}\right)^{\top}
$$
where we have $\bU_{\perp} \in \mathbb{R}^{k_1 \times\left(k_1-r\right)}, \bD^* \in \mathbb{R}^{k_1 \times k_2}$ and $\bV_{\perp} \in \mathbb{R}^{k_2 \times\left(k_2-r\right)}$. Furthermore, we define
$$
\bLambda=\left(\bU, \bU_{\perp}\right)^{\top} \bTheta\left(\bV, \bV_{\perp}\right)=\left(\begin{array}{cc}
\bU^{\top} \bTheta \bV & \bU^{\top} \bTheta \bV_{\perp} \\
\bU_{\perp}^{\top} \bTheta \bV & \bU_{\perp}^{\top} \bTheta \bV_{\perp}
\end{array}\right)=\bLambda_1+\bLambda_2
$$
where we write
$$
\bLambda_1=\left(\begin{array}{cc}
0 & 0 \\
0 & \bU_{\perp}^{\top} \bTheta \bV_{\perp}
\end{array}\right), \quad \bLambda_2=\left(\begin{array}{cc}
\bU^{\top} \bTheta \bV & \bU^{\top} \bTheta \bV_{\perp} \\
\bU_{\perp}^{\top} \bTheta \bV & 0
\end{array}\right).
$$
Afterward, it holds that
\begin{align}
\|\wS_\ell\|_{\mathrm{nuc}} & =\left\|\mu^* \bS_{m,*} +\bTheta\right\|_{\mathrm{nuc}} \overset{(a)}{=}\left\|\left(\bU, \bU_{\perp}\right)\left(\mu^* \bD^*+\bLambda\right)\left(\bV, \bV_{\perp}\right)^{\top}\right\|_{\mathrm{nuc}} \nonumber\\
& \overset{(b)}{=}\left\|\mu^* \bD^*+\bLambda\right\|_{\mathrm{nuc}}+\left\|\mu^* \bD^*+\bLambda_1+\bLambda_2\right\|_{\mathrm{nuc}} \nonumber\\
& \geq\left\|\mu^* \bD^*+\bLambda_1\right\|_{\mathrm{nuc}}-\left\|\bLambda_2\right\|_{\mathrm{nuc}} \nonumber\\
& =\left\|\mu^* \bD\right\|_{\mathrm{nuc}}+\left\|\bLambda_1\right\|_{\mathrm{nuc}}-\left\|\bLambda_2\right\|_{\mathrm{nuc}} \nonumber\\
& =\left\|\mu^* \bS_{m,*}\right\|_{\mathrm{nuc}}+\left\|\bLambda_1\right\|_{\mathrm{nuc}}-\left\|\bLambda_2\right\|_{\mathrm{nuc}}, \label{eq:16-kang-wS}
\end{align}
where, $(a)$ follows from the definition of $\bS_{m,*}$, and $(b)$ follows the definition of $\bLambda$.
This implies that
$$
\left\|\mu^* \bS_{m,*}\right\|_{\mathrm{nuc}}-\|\wS_{m,\ell}\|_{\mathrm{nuc}} \leq\left\|\bLambda_2\right\|_{\mathrm{nuc}}-\left\|\bLambda_1\right\|_{\mathrm{nuc}}.
$$
Combining \eqref{eq:15-kang-wS} and \eqref{eq:16-kang-wS}, we have that
$$
2\|\bTheta\|_F^2 \leq\left(\left\|\nabla L\left(\mu^* \bB\right)\right\|_{\mathrm{op}}+\gamma_{\ell}\right)\left\|\bLambda_2\right\|_{\mathrm{nuc}}+\left(\left\|\nabla L\left(\mu^* \bB\right)\right\|_{\mathrm{op}}-\gamma_{\ell}\right)\left\|\bLambda_1\right\|_{\mathrm{nuc}}.
$$
Then, we refer to the setting in our \Cref{lemma:B4}, and we choose $\gamma_\ell=4 t / \sqrt{M\tau^E_{\ell}}$ where the value of $t$ is determined in \Cref{lemma:B4}, i.e.
$$
\gamma_{\ell}=4 \sqrt{\frac{2\left(4 +S_0^2\right) C k_1 k_2 \log \left(2\left(d_1+d_2\right) / \delta_\ell\right)}{M\tau^E_{\ell}}},
$$
we know that $\lambda_{T-1} \geq 2\left\|\nabla L\left(\mu^* \bS_{m,*}\right)\right\|_{o p}$ with probability at least $1-\delta_\ell$ for any $\delta_\ell \in(0,1)$. Therefore, with a probability at least $1-\delta_\ell$, we have
$$
2\|\bTheta\|_F^2 \leq \frac{3}{2} \gamma_{\ell}\left\|\bLambda_2\right\|_{\mathrm{nuc}}-\frac{1}{2} \gamma_{\ell}\left\|\bLambda_1\right\|_{\mathrm{nuc}} \leq \frac{3}{2} \gamma_{\ell}\left\|\bLambda_2\right\|_{\mathrm{nuc}}.
$$
Since we can easily verify that the rank of $\bLambda_2$ is at most $2 r$, and by using Cauchy-Schwarz Inequality we have that
$$
2\|\bTheta\|_F^2 \leq \frac{3}{2} \gamma_{\ell} \sqrt{2 r}\left\|\bLambda_2\right\|_F \leq \frac{3}{2} \gamma_{\ell} \sqrt{2 r}\|\bLambda\|_F=\frac{3}{2} \gamma_{\ell} \sqrt{2 r}\|\bTheta\|_F
$$
which implies that
$$
\|\bTheta\|_F \leq \frac{3}{4} \sqrt{2 r} \gamma_{\ell}=6 \sqrt{\frac{\left(4+S_0^2\right) C k_1 k_2 r \log \left(\frac{2\left(k_1+k_2\right)}{\delta_\ell}\right)}{\tau^E_{\ell}}}
$$
This implies that $\Pb(\F'_\ell) \geq 1-\delta_\ell$.
Taking a union bound over all phases $\ell \geq 1$ and recalling $\delta_\ell:=\frac{\delta}{2 \ell^2}$, we obtain
\begin{align*}
\Pb(\F') & \geq 1-\sum_{\ell=1}^{\infty} \Pb\left((\F')^c_\ell\right) \\
& \geq 1-\sum_{\ell=1}^{\infty} \frac{\delta_\ell}{2} \\
& =1-\sum_{\ell=1}^{\infty} \frac{\delta}{4 \ell^2} \\
& \geq 1-\frac{\delta}{2} .
\end{align*}
This concludes our proof.
\end{proof}

\subsection{Final Sample Complexity Bound}

We first define the arm elimination event similar to \Cref{thm:single-task}. 
For any $\V \subseteq \ucW$ be the active set and $\uw \in \V$ define
\begin{align}
\mathcal{E}_{\uw, \ell}(\V)=\left\{\left|\left\langle \uw-\uw^{\star}, \wtheta_{\ell}(\V)-\btheta^*\right\rangle\right| \leq \epsilon_{\ell}\right\} \label{eq:event-arm-elim-multi}
\end{align}
where it is implicit that $\wtheta_{\ell}:=\wtheta_{\ell}(\V)$ is the design constructed in the algorithm at stage $\ell$ with respect to $\ucW_{\ell}=\V$.

\begin{customtheorem}{2}\textbf{(Restatement)}
\label{app:theorem-2}
    With probability at least $1 - \delta$, multi-task \gob\ returns the best arms $\bx_*$, $\bz_*$, and the number of samples used is bounded by
\begin{align*}
    \widetilde{O}\left(\dfrac{M(k_1+k_2)r}{\Delta^2} + \dfrac{M\sqrt{k_1 k_2 r}}{S_r} + \dfrac{\sqrt{d_1 d_2 r}}{S_r}\right).
\end{align*}
\end{customtheorem}

\begin{proof} For the rest of the proof we have that the good events $\F_\ell \bigcap \F'_\ell \bigcap \mathcal{E}_{\uw, \ell}(\ucW_\ell)$ holds true for each phase $\ell$ with probability greater than $(1-\delta)$. The three events are defined in \eqref{eq:good-event-Theta-multi}, \eqref{eq:good-event-Theta-multi-wS} and \eqref{eq:event-arm-elim-multi}.

\textbf{Third Stage:} Define $\underline{\A}_{m,\ell}=\left\{\uw \in \ucW_\ell:\left\langle \uw^{\star}-\uw, \btheta^*\right\rangle \leq 4 \epsilon_{m,\ell}\right\}$. Note that by assumption $\ucW=\ucW_1=\underline{\A}_1$. The above lemma implies that with probability at least $1-\delta$ we have $\bigcap_{\ell=1}^{\infty}\left\{\ucW_{m,\ell} \subseteq S_{m,\ell}\right\}$. This implies that
\begin{align*}
\rho^G\left(\ucW_{m,\ell}\right) & =\min _{\bb \in \Delta_{\ucW_m}} \max _{\uw, \uw^{\prime} \in \ucW_{m,\ell}}\left\|\uw-\uw^{\prime}\right\|_{\left(\sum_{\uw \in \ucW} \bb_{\uw} \uw\ \uw^{\top} +\bLambda \right)^{-1}}^2 \\
& \leq \min _{\bb \in \Delta \ucW_m} \max _{\uw, \uw^{\prime} \in S_{m,\ell}}\left\|\uw-\uw^{\prime}\right\|_{\left(\sum_{\uw \in \ucW} \bb_{\uw} \uw\ \uw^{\top} + \bLambda\right)^{-1}}^2 \\
& =\rho^G\left(\underline{\A}_{m,\ell}\right).
\end{align*}
Let the effective dimension be $k = (k_1 + k_2)r$. 
Define $k^{\ell}_1 = 8k\log(1+\tau^G_{m,\ell-1}/\lambda)$.
For $\ell \geq\left\lceil\log _2\left(4 \Delta^{-1}_m\right)\right\rceil$ we have that $S_{m,\ell}=\left\{\uw^{\star}\right\}$, thus, the sample complexity to identify $\uw^{\star}_m$ is equal to
\begin{align*}
&\sum_{\ell=1}^{\left\lceil\log _2\left(4 \Delta^{-1}_m\right)\right\rceil} \sum_{\uw \in \ucW_m}\left\lceil\tau^G_{m,\ell} \wb^G_{m, \ell, \uw}\right\rceil =\sum_{\ell=1}^{\left\lceil\log _2\left(4 \Delta^{-1}_m\right)\right\rceil}\left(\frac{(k^\ell_1+1) k^\ell_1}{2}+\tau^G_{m,\ell}\right) \\
& =\sum_{\ell=1}^{\left\lceil\log _2\left(4 \Delta^{-1}_m\right)\right\rceil}\left(\frac{(k^\ell_1+1) k^\ell_1}{2}+2 \epsilon_{m,\ell}^{-2} \rho^G(\ucW_{m,\ell}) B^\ell_{m,*} \log \left(4 k^\ell_1 \ell^2|\ucW_m| / \delta\right)\right) \\
& \overset{(a)}{\leq} 2\sum_{\ell=1}^{\left\lceil\log _2\left(4 \Delta^{-1}_m\right)\right\rceil}\left(\frac{(k+1) k}{2}\log^2(1+\tau^G_{m,\ell-1}) + 2 \epsilon_{m,\ell}^{-2} \rho^G(\ucW_{m,\ell}) B^\ell_{m,*} \log \left(4 k^\ell_1 \ell^2|\ucW_m| / \delta\right)\right) \\
& \overset{(b)}{\leq} 2\frac{(k+1) k}{2}\sum_{\ell=1}^{\left\lceil\log _2\left(4 \Delta^{-1}_m\right)\right\rceil}\left(\log^2(1+\tau^G_{m,\ell-1}) + 8 \epsilon_{m,\ell}^{-2} \rho^G(\ucW_{m,\ell}) B^\ell_{m,*} \log \left(4 k^\ell_1 \ell^2|\ucW_m| / \delta\right)\right) \\
& \overset{(c)}{\leq}(k+1) k\sum_{\ell=1}^{\left\lceil\log _2\left(4 \Delta^{-1}_m\right)\right\rceil}\left(1 + 16 \epsilon_{m,\ell}^{-2} \rho^G(\ucW_{m,\ell}) B^\ell_{m,*} \log^2 \left(4 k^\ell_1 \ell^2|\ucW_m| / \delta\right)\right) \\
& \overset{(d)}{\leq} (k+1) k\left\lceil\log _2\left(4 \Delta^{-1}_m\right)\right\rceil+\sum_{\ell=1}^{\left\lceil\log _2\left(4 \Delta^{-1}_m\right)\right\rceil} 32 \epsilon_{m,\ell}^{-2} f\left(\underline{\A}_{m,\ell}\right) B^\ell_{m,*} \log \left(4 k \ell^2|\ucW_m| / \delta\right) \\
&\overset{(e)}{\leq} (k+1) k\left\lceil\log _2\left(4 \Delta^{-1}\right)\right\rceil+\sum_{\ell=1}^{\left\lceil\log _2\left(4 \Delta^{-1}\right)\right\rceil} 32 \epsilon_{\ell}^{-2} f\left(\underline{\A}_{m,\ell}\right) (64\lambda S^2 + 64\tau^G_{m,\ell-1})\log \left(4 k \ell^2|\ucW| / \delta\right) 
\end{align*}
\begin{align*}
&\overset{}{=} (k+1) k\left\lceil\log _2\left(4 \Delta^{-1}\right)\right\rceil+\sum_{\ell=1}^{\left\lceil\log _2\left(4 \Delta^{-1}\right)\right\rceil} 32 \epsilon_{\ell}^{-2} f\left(\underline{\A}_{m,\ell}\right) (64\lambda S^2)\log \left(4 k \ell^2|\ucW| / \delta\right) \\
&\quad + (k+1) k\left\lceil\log _2\left(4 \Delta^{-1}\right)\right\rceil+\sum_{\ell=1}^{\left\lceil\log _2\left(4 \Delta^{-1}\right)\right\rceil} 32 \epsilon_{\ell}^{-2} f\left(\underline{\A}_{m,\ell}\right) (64\tau^G_{m,\ell-1})\log \left(4 k \ell^2|\ucW| / \delta\right) \\
& \overset{(f)}{\leq} (k+1) k\left\lceil\log _2\left(4 \Delta^{-1}_m\right)\right\rceil + 2048 \lambda S^2\log \left(\frac{4k \log _2^2\left(8 \Delta^{-1}_m\right)|\ucW_m|}{\delta}\right) \sum_{\ell=1}^{\left\lceil\log _2\left(4 \Delta^{-1}_m\right)\right\rceil} 2^{2 \ell} f\left(\underline{\A}_{m,\ell}\right) .
\end{align*}
where, $(a)$ follows as $\log^2(1+\tau^G_{m,\ell-1}/\lambda) \leq \log^2(1+\tau^G_{m,\ell-1})$, $(b)$ follows by noting that $\log(x \log(1 + x)) \leq 2\log(x)$ for any $x > 1$. The $(c)$ follows by subsuming the $\log^2(1+\tau^G_{m,\ell-1})$ into $2\tau^G_\ell$. The $(d)$ follows as $\log(1+\tau^G_{m,\ell-1}) < \tau^G_\ell$ which enables us to replace the $k^\ell_1$ inside the $\log$ with an addtional factor of $2$. 
The $(e)$ follows similarly to \eqref{eq:upper-bound_B_star} by noting that 
\begin{align*}
    B^\ell_{m,*} &\leq 64(\sqrt{\lambda} S + \sqrt{\lambda^\perp_{m,\ell}} S^{\perp}_{m,\ell})\nonumber\\
    &\leq 64\lambda S^2 + \left(\dfrac{64\tau^G_{m,\ell-1}}{8(d_1+d_2)r\log(1+\frac{\tau^G_{m,\ell-1}}{\lambda})}\right)\cdot\left(\frac{8d_1 d_2 r}{\tau^E_{\ell} S^2_{r}} \log \left(\frac{d_1+d_2}{\delta_\ell}\right)\right)\nonumber\\
    &\overset{(a_1)}{\leq} 64\lambda S^2 + 64\tau^G_{m,\ell-1}.
\end{align*}
Finally the $(f)$ follows by subsuming the $\tau^G_{\ell-1}$ with a factor of $2$ into the quantity of $\tau^G_{\ell}$. 
%
%
%
Then it follows that
\begin{align*}
& \rho^{G}_{m,*} =\inf _{\bb \in \triangle_{\ucW_m}} \max _{\uw \in \ucW_m} \frac{\left\|\uw-\uw^{\star}\right\|_{\left(\sum_{\uw \in \ucW_m} \bb_{\uw} \uw\ \uw^{\top} + \bLambda\right)^{-1}}^2}{\left(\left\langle \uw-\uw^{\star}, \btheta^*\right\rangle\right)^2} \\
& =\inf _{\bb \in \triangle_{\ucW_m}} \max _{\ell \leq\left\lceil\log _2\left(4 \Delta^{-1}_m\right)\right\rceil} \max _{\uw \in S_{m,\ell}} \frac{\left\|\uw-\uw^{\star}\right\|_{\left(\sum_{\uw \in \ucW_m} \bb_{\uw} \uw\ \uw^{\top}+ \bLambda\right)^{-1}}^2}{\left(\left\langle \uw-\uw^{\star}, \btheta^*\right\rangle\right)^2} \\
& \geq \frac{1}{\left\lceil\log _2\left(4 \Delta^{-1}_m\right)\right\rceil} \inf _{\bb \in \triangle_{\ucW_m}} \sum_{\ell=1}^{\left\lceil\log _2\left(4 \Delta^{-1}_m\right)\right\rceil} \max _{\uw \in \underline{\A}_{m,\ell}} \frac{\left\|\uw-\uw^{\star}\right\|_{\left(\sum_{\uw \in \ucW_m} \bb_{\uw} \uw\ \uw^{\top}+ \bLambda\right)^{-1}}^2}{\left(\left\langle \uw-\uw^{\star}, \btheta^*\right\rangle\right)^2} \\
& \geq \frac{1}{16\left\lceil\log _2\left(4 \Delta^{-1}_m\right)\right\rceil} \sum_{\ell=1}^{\left\lceil\log _2\left(4 \Delta^{-1}_m\right)\right\rceil} 2^{2 \ell} \inf _{\bb \in \triangle_{\ucW_m}} \max _{\uw \in \underline{\A}_{m,\ell}}\left\|\uw-\uw^{\star}\right\|_{\left(\sum_{\uw \in \ucW_m} \bb_{\uw} \uw\ \uw^{\top}+ \bLambda\right)^{-1}}^2 
\end{align*}
\begin{align*}
& \geq \frac{1}{64\left\lceil\log _2\left(4 \Delta^{-1}_m\right)\right\rceil} \sum_{\ell=1}^{\left\lceil\log _2\left(4 \Delta^{-1}_m\right)\right\rceil} 2^{2 \ell} \inf _{\bb \in \triangle_{\ucW_m}} \max _{\uw, \uw^{\prime} \in \underline{\A}_{m,\ell}}\left\|\uw-\uw^{\prime}\right\|_{\left(\sum_{\uw \in \ucW_m} \bb_{\uw} \uw\ \uw^{\top}+ \bLambda\right)^{-1}}^2 \\
& \geq \frac{1}{64\left\lceil\log _2\left(4 \Delta^{-1}_m\right)\right\rceil} \sum_{\ell=1}^{\left\lceil\log _2\left(4 \Delta^{-1}_m\right)\right\rceil} 2^{2 \ell} f\left(\underline{\A}_{m,\ell}\right).
\end{align*}
This implies that 
\begin{align*}
    \sum_{\ell=1}^{\left\lceil\log _2\left(4 \Delta^{-1}_m\right)\right\rceil} 2^{2 \ell} f\left(\underline{\A}_{m,\ell}\right) \leq \rho^{G}_{m,*} 64\left\lceil\log _2\left(4 \Delta^{-1}_m\right)\right\rceil.
\end{align*}
Plugging this back we get
\begin{align*}
    \sum_{\ell=1}^{\left\lceil\log _2\left(4 \Delta^{-1}_m\right)\right\rceil} \sum_{\uw \in \ucW_m}\left\lceil\tau^G_{m,\ell} \widehat{\bb}_{\ell, \uw}\right\rceil 
    &\leq (k+1) k\left\lceil\log _2\left(4 \Delta^{-1}_m\right)\right\rceil \\
    &\qquad + 2048 \lambda S^2\log \left(\frac{8k \log _2^2\left(8 \Delta^{-1}_m\right)|\ucW_m|}{\delta}\right) 64\rho^{G}_{m,*}\left\lceil\log _2\left(4 \Delta^{-1}_m\right)\right\rceil\\
    &\leq (k+1) k\left\lceil\log _2\left(4 \Delta^{-1}_m\right)\right\rceil \\
    &\qquad + C_2 \lambda S^2\log \left(\frac{8 k \log _2^2\left(8 \Delta^{-1}_m\right)|\ucW_m|}{\delta}\right) \rho^{G}_{m,*}\left\lceil\log _2\left(4 \Delta^{-1}_m\right)\right\rceil
\end{align*}
where, $C_2 > 0$ is a constant.
Summing over each task $m$, the simplified sample complexity for the third stage is given by
\begin{align*}
    N_3\leq O\left(\dfrac{Mk}{\Delta^2}\log \left(\frac{k\log_2\left( \Delta^{-1}\right)|\ucW_m|}{\delta}\right)\right) = \widetilde{O}\left(\dfrac{M(k_1+k_2)r}{\Delta^2}\right)
\end{align*}
where $\Delta = \min_{\uw\in\ucW} (\uw_* - \uw)^\top\btheta_*  \overset{(a1)}{=} \min_{\bx\in\X\setminus\{\bx_*\},\bz\in\Z\setminus\{\bz_*\} }(\bx_*^\top\bTheta_*\bz_* - \bx^\top\bTheta_*\bz)$. The $(a1)$ follows by reshaping the arms in $\ucW$ to recover the arms in $\X$ and $\Z$. 


\textbf{2nd Stage:} Again recall that the $E$-optimal design in stage 2 of \Cref{alg:bandit-pure-multi} satisfies the \Cref{assm:distribution} as the sample distribution $\D$ has finite second order moments.



For the second stage first observe that by plugging in the definition of $\widetilde{\tau}^\ell_E$ we get
\begin{align*}
\left\|\btheta_{m,k+1: p}^*\right\|_2^2 & =\sum_{i>r \wedge j>r} H_{i j}^2=\left\|(\widehat{\bU}_\ell^{\perp})^{\top}\left(\bU^* \mathbf{S}^* \bV^{* \top}\right) \widehat{\bV}_\ell^{\perp}\right\|_F^2 \\
& \leq\left\|(\widehat{\bU}_\ell^{\perp})^{\top} \bU^*\right\|_F^2\left\|\mathbf{S}^*\right\|_2^2\left\|(\widehat{\bV}_\ell^{\perp})^{\top} \bV^*\right\|_F^2  \leq O\left(\frac{k_1 k_2 r}{\tau^E_\ell } \log \left(\frac{k_1+k_2}{\delta}\right)\right)\\
&= O\left(\dfrac{\sqrt{d_1 d_2 r}}{S_{r}} \log \left(\frac{d_1+d_2}{\delta_\ell}\right)\right),
\end{align*}
which implies $\left\|\btheta_{k+1: p}^*\right\|_2=\tO\left(\sqrt{k_1k_2 r}/S_r\right)$. We also set $\frac{8k_1 k_2 r}{\tau^E_{m,\ell} S^2_{r}} \log \left(\frac{k_1+k_2}{\delta_\ell}\right) :=S_{m,\ell}^{\perp}$.
Now we bound the sample complexity from the second stage. 
From the second stage we can show that we have for the arm set $\tW_m$

\begin{align*}
N_2 &= \sum_{\ell=1}^{\left\lceil\log _2\left(4 \Delta^{-1}\right)\right\rceil} \sum_{\tw \in \tW_m}\left\lceil\widetilde{\tau}^E_{m,\ell} \wb^E_{m,\ell, \tw}\right\rceil\\
& =\sum_{\ell=1}^{\left\lceil\log _2\left(4 \Delta^{-1}\right)\right\rceil}\left(\frac{(p+1) p}{2}+\widetilde{\tau}^E_{m,\ell}\right) \\
& =\sum_{\ell=1}^{\left\lceil\log _2\left(4 \Delta^{-1}\right)\right\rceil}\left(\frac{(p+1) p}{2}+ \dfrac{\sqrt{8k_1 k_2 r\log \left(4  \ell^2|\W| / \delta\right)}}{S_r}\right) \\
& \leq (p+1) p\left\lceil\log _2\left(4 \Delta^{-1}\right)\right\rceil + 32 \dfrac{\sqrt{k_1 k_2 r}}{S_r}\log \left(\frac{4 \log _2^2\left(8 \Delta^{-1}\right)|\W|}{\delta}\right) \left\lceil\log _2\left(4 \Delta^{-1}\right)\right\rceil \\
&= O\left(\dfrac{\sqrt{k_1 k_2 r}}{S_r}\log \left(\frac{4 \log _2^2\left(8 \Delta^{-1}\right)|\W|}{\delta}\right)\right) \overset{(a)}{=} \widetilde{O}\left(\dfrac{\sqrt{k_1 k_2 r}}{S_r}\right)
\end{align*}

\textbf{1st Stage:} Finally we also use the $E$-optimal design in first stage of \Cref{alg:bandit-pure-multi}. Note that this design satisfies the \Cref{assm:distribution} as the sample distribution $\D$ has finite second-order moments. Now we bound the sample complexity from the first stage. 
From the first stage we can show that we have for the arm set $\ocW$
\begin{align*}
N_1 &= \sum_{\ell=1}^{\left\lceil\log _2\left(4 \Delta^{-1}\right)\right\rceil} \sum_{m=1}^M\sum_{\ow \in \ocW_m}\left\lceil\tau^E_{m,\ell} \wb^E_{m,\ell, \ow}\right\rceil\\
& =\sum_{\ell=1}^{\left\lceil\log _2\left(4 \Delta^{-1}\right)\right\rceil}\left(\frac{M(p+1) p}{2}+\sum_{m}\tau^E_{m,\ell}\right) \\
& =\sum_{\ell=1}^{\left\lceil\log _2\left(4 \Delta^{-1}\right)\right\rceil}\left(\frac{M(p+1) p}{2}+ \dfrac{\sqrt{8d_1 d_2 r\log \left(4  \ell^2|\W| / \delta\right)}}{S_r}\right) \\
& \leq M(p+1) p\left\lceil\log _2\left(4 \Delta^{-1}\right)\right\rceil + 32 \dfrac{\sqrt{d_1 d_2 r}}{S_r}\log \left(\frac{4 \log _2^2\left(8 \Delta^{-1}\right)|\W|}{\delta}\right) \left\lceil\log _2\left(4 \Delta^{-1}\right)\right\rceil \\
&\overset{(a)}{=} O\left(\dfrac{\sqrt{d_1 d_2 r}}{S_r}\log \left(\frac{4 \log _2^2\left(8 \Delta^{-1}\right)|\W|}{\delta}\right)\right) \overset{(a)}{=} \widetilde{O}\left(\dfrac{\sqrt{d_1 d_2 r}}{S_r}\right)
\end{align*}
where, $(a)$ follows as $p = d_1d_2$.
Combining $N_1, N_2$ and $N_3$ gives the claim of the theorem.
\end{proof}

\subsection{Additional Experimental Details}
\label{app:expt}

\textbf{Single Task Unit Ball:} This experiment consists of a set of $\{6,10,14\}$ left and right arms that are arranged in a unit ball in $\R^6$, and $\|\bx\|=1$, $\|\bz\|=1$ for all $\bx\in\X$ and $\bz\in\Z$. Hence, we have $d_1\in\R^6$ and $d_2\in\R^6$. We choose a random $\bTheta_*\in\R^{d_1\times d_2}$ which has rank $r=2$. We set $\delta = 0.1$. We compare against RAGE \citep{fiez2019sequential} that treats this $d_1d_2$ bilinear bandit as a linear bandit setting and suffers a sample complexity that scales as $\tO(d_1d_2/\Delta^2)$. We do a continuous relaxation of the algorithm when implementing it to make this more tractable.

\textbf{Multi-task Unit Ball:} This experiment consists of a set of $\{5,10,15,20,25,30\}$ tasks. For each task, we choose left and right arms that are arranged in a unit ball in $\R^8$, and $\|\bx\|=1$, $\|\bz\|=1$ for all $\bx\in\X$ and $\bz\in\Z$. Hence, we have $d_1\in\R^8$ and $d_2\in\R^8$. We choose $k_1 = k_2 = 4$, and feature extractors $\bB_1\in\R^{d_1\times k_1}$, $\bB_2\in\R^{d_2\times k_2}$ shared across tasks. We choose a random matrix $\bS_{m,*}\in\R^{k_1\times k_2}$ for each task $m$ such that $\bS_{m,*}$ has rank $r=2$. We set $\delta = 0.1$. We compare against DouExpDes \citep{du2023multi} that treats this setting as $M$ $k_1k_2$ bilinear bandits (after learning the feature extractors) and suffers a sample complexity that scales as $\tO(M k_1k_2/\Delta^2)$ (see \Cref{remark:comparison}). Again we do a continuous relaxation of the algorithm when implementing it to make this more tractable.

\newpage
\section{Table of Notations}
\label{table-notations}

\begin{table}[!tbh]
    \centering
    \begin{tabular}{|p{10em}|p{33em}|}
        \hline\textbf{Notations} & \textbf{Definition} \\\hline
        $\X$  & Left arm set \\\hline
        $\Z$  & Right arm set \\\hline
        $M$  & Number of tasks \\\hline
        $\ell$  & Phase number \\\hline
        $\bTheta_{m,*}$  & Hidden parameter matrix for \\\hline
        $\bb^E_\ell$  & E-optimal design at the $\ell$-th phase\\\hline
        $\bb^G_{m,\ell}$  & G-optimal design at the $\ell$-th phase for the $m$-th task\\\hline
        $S_{m,\ell}^{\perp} $  & $\frac{8d_1 d_2 r}{\tau^E_{\ell} S^2_{r}} \log \left(\frac{d_1+d_2}{\delta_\ell}\right)$\\\hline
        $\lambda_{m,\ell}^{\perp}$ & $\tau^G_{m,\ell-1}/8(d_1+d_2)r\log(1+\frac{\tau^G_{m,\ell-1}}{\lambda})$ \\\hline
        $B^\ell_{m,*}$  & $(8\sqrt{\lambda} S + \sqrt{\lambda_{m,\ell}^{\perp}} S^{\perp}_{m,\ell})$\\\hline
        $\bB_1$  & Left feature extractor\\\hline
        $\bB_2$  & Right feature extractor\\\hline
        $S_r$  & $r$-th largest singular value of $\bTheta_*$\\\hline
        $\Delta(\bx,\bz)$ & $\bx_*^\top\bTheta_*\bz_* - \bx^\top\bTheta_*\bz$\\\hline 
        $\Delta $ & $\min_{\bx \neq \bx_*,\bz\neq \bz_*} \Delta(\bx,\bz)$\\\hline
        $\Y(\W)$ & $\left\{\bw-\bw^{\prime}: \forall \bw, \bw^{\prime} \in \W, \bw \neq \bw^{\prime}\right\}$ \\\hline
        $\Y^*(\W) $ & $\left\{\bw_*-\bw: \forall \bw \in \W \backslash \bw_*\right\}$ \\\hline
        $\delta$ & confidence level \\\hline
    \end{tabular}
    \vspace{1em}
    \caption{Table of Notations}
    \label{tab:my_label}
\end{table}




\end{document}